\documentclass{article}
\pdfoutput=1

\PassOptionsToPackage{numbers, compress}{natbib}

\usepackage[preprint]{neurips_2022}

\usepackage[utf8]{inputenc} 
\usepackage[T1]{fontenc}    
\usepackage{url}            
\usepackage{booktabs}       
\usepackage{amsfonts}       
\usepackage{nicefrac}       
\usepackage{microtype}      
\usepackage{xcolor}         

\usepackage{amsmath,amsthm,amssymb,sansmath,graphicx,float,color,authblk,stmaryrd,calc}
\usepackage[utf8]{inputenc}
\usepackage[inline]{enumitem}
\usepackage[active]{srcltx}
\usepackage{upgreek}
\usepackage{mathrsfs}
\usepackage[utf8]{inputenc}
\usepackage{makeidx}
\usepackage{oldgerm}
\usepackage{mathrsfs}
\usepackage[active]{srcltx}
\usepackage{verbatim}
\usepackage{aliascnt}
\usepackage{array}
\usepackage[textwidth=4cm,textsize=footnotesize]{todonotes}
\usepackage{xargs}
\usepackage{cellspace}
\usepackage[linesnumbered,ruled]{algorithm2e}
\usepackage{svg}
\usepackage{bbm}
\usepackage{dsfont}
\usepackage{datatool}

\usepackage{aliascnt}

\usepackage{xr-hyper}
\usepackage{hyperref}
\usepackage{cleveref}

\usepackage{autonum}
\makeatletter
\newtheorem{theorem}{Theorem}
\crefname{theorem}{theorem}{Theorems}
\Crefname{Theorem}{Theorem}{Theorems}

\newaliascnt{lemma}{theorem}
\newtheorem{lemma}[lemma]{Lemma}
\aliascntresetthe{lemma}
\crefname{lemma}{lemma}{lemmas}
\Crefname{Lemma}{Lemma}{Lemmas}

\newaliascnt{corollary}{theorem}

\aliascntresetthe{corollary}
\crefname{corollary}{corollary}{corollaries}
\Crefname{Corollary}{Corollary}{Corollaries}

\newaliascnt{proposition}{theorem}

\aliascntresetthe{proposition}
\crefname{proposition}{proposition}{propositions}
\Crefname{Proposition}{Proposition}{Propositions}

\newaliascnt{definition}{theorem}

\aliascntresetthe{definition}
\crefname{definition}{definition}{definitions}
\Crefname{Definition}{Definition}{Definitions}

\newaliascnt{remark}{theorem}

\aliascntresetthe{remark}
\crefname{remark}{remark}{remarks}
\Crefname{Remark}{Remark}{Remarks}

\crefname{example}{example}{examples}
\Crefname{Example}{Example}{Examples}

\crefname{algorithm}{algorithm}{algorithms}
\Crefname{Algorithm}{Algorithm}{Algorithms}

\crefname{figure}{figure}{figures}
\Crefname{Figure}{Figure}{Figures}

\newtheorem{assumption}{\textbf{A}\hspace{-3pt}}
\Crefname{assumption}{\textbf{A}\hspace{-3pt}}{\textbf{A}\hspace{-3pt}}
\crefname{assumption}{\textbf{A}}{\textbf{A}}

\Crefname{assumptionL}{\textbf{L}\hspace{-3pt}}{\textbf{L}\hspace{-3pt}}
\crefname{assumptionL}{\textbf{L}}{\textbf{L}}

\Crefname{assumptionL2}{\textbf{L2HMC}\hspace{-3pt}}{\textbf{L2HMC}\hspace{-3pt}}
\crefname{assumptionL2}{\textbf{L2HMC}}{\textbf{L2HMC}}

\Crefname{assumptionG}{\textbf{G}\hspace{-3pt}}{\textbf{G}\hspace{-3pt}}
\crefname{assumptionG}{\textbf{G}}{\textbf{G}}

\Crefname{assumptionN}{\textbf{NICE}\hspace{-3pt}}{\textbf{NICE}\hspace{-3pt}}
\crefname{assumptionN}{\textbf{NICE}}{\textbf{NICE}}

\Crefname{assumptionH}{{\textbf{H}\hspace{-1pt}}}{{\textbf{H}\hspace{-1pt}}}
\crefname{assumptionH}{{\textbf{H}}}{{\textbf{H}}}

\usepackage{caption}
\usepackage{subcaption}
\usepackage{multirow}

\def\elbo{\mathcal{L}}
\def\elboiw{\elbo_{M}}

\def\msa{\mathsf{A}}

\def\burningloss{\upsilon}
\def\mserunconst{\upzeta^{\scriptsize{\mbox{{\it mse}}}}}
\def\biasrunconst{\upzeta^{\scriptsize{\mbox{{\it bias}}}}}

\def\mcf{\mathcal{F}}

\newcommand{\dataset}[1]{\mathcal{D}_{\scriptsize{\mbox{{\it #1}}}}}

\def\rset{\mathbb{R}}

\def\nset{\mathbb{N}}
\def\nsets{\mathbb{N}^*}

\def\rmd{\mathrm{d}}

\def\rme{\mathrm{e}}

\newcommandx{\functionspace}[2][1=+]{\mathbb{F}_{#1}(#2)}

\newcommand{\1}{\mathds{1}}
\newcommand{\SUISIR}{BR-SNIS}

\newcommand{\LeftEqNo}{\let\veqno\@@leqno}

\def\osc{\operatorname{osc}}

\newcommand{\Exp}[1]{\PE\left[ #1 \right]}
\newcommand{\PE}{\mathbb{E}}
\newcommandx{\CPE}[3][1=]{{\mathbb E}_{#1}\left[\left. #2 \, \right| #3 \right]} 
\newcommandx{\CPP}[3][1=]{{\mathbb P}_{#1}\left[\left. #2 \, \right| #3 \right]} 

\newcommandx{\bconst}[1][1=]{\ensuremath{\mathsf{b}_{#1}}}
\newcommandx{\cconst}[1][1=]{\ensuremath{c_{#1}}}
\newcommandx{\rate}[1][1=]{\ensuremath{\tilde{\kappa}_{#1}}}
\newcommandx{\dconst}[1][1=]{\ensuremath{d_{#1}}}

\newcommandx{\pushforward}[2][1=T,2=\Lambda]{ {#1} \# {#2}}
\newcommandx{\VnormEq}[2][1=V]{\left\| #2 \right\|_{#1}}
\newcommandx{\norm}[2][1=]{\ifthenelse{\equal{#1}{}}{\Vert #2 \Vert}{\Vert #2 \Vert^{#1}}}

\newcommand{\PP}{\mathbb{P}}
\newcommand{\tvnorm}[1]{\| #1 \|_{\mathrm{TV}}}
\newcommandx{\tvdist}[3][1=]{\ensuremath{d^{#1}_{\scriptsize{\mbox{{\it TV}}}}}(#2,#3)}
\newcommandx{\Vnorm}[2][1=V]{\| #2 \|_{#1}}

\def\ie{\textit{i.e.}}
\def\eqsp{}

\newcommand{\ooint}[1]{\left(#1\right)}
\newcommand{\ccint}[1]{\left[#1\right]}

\def\State{\chunku{X}{1}{N}}
\def\state{\chunku{x}{1}{N}}

\newcommandx{\estsnis}[2][1=M]{\widehat{\target}_{#1}(#2)}
\newcommandx{\msesnis}[1][1=M]{\mathsf{MSE}^{\scriptsize{\mbox{{\it is}}}}_{#1}}
\newcommandx{\TmixN}[1][1=N]{\tau_{\scriptsize{\mbox{{\it mix}}}, #1}}
\newcommand{\indi}[1]{\1_{#1}}
\newcommand{\indiacc}[1]{\1_{\{ #1 \}}}

\def\as{\ensuremath{\text{a.s.}}}

\newcommandx\sequenceg[3][2=,3=]
{\ifthenelse{\equal{#3}{}}{\ensuremath{( #1_{#2})}}{\ensuremath{( #1_{#2})_{ #2 \geq #3}}}}

\newcommandx\sequence[3][2=,3=]
{\ifthenelse{\equal{#3}{}}{\ensuremath{(#1_{#2})}}{\ensuremath{(#1_{#2})_{#2 \in #3}}}}
\newcommandx\dsequence[4][3=,4=]{\ensuremath{ ((#1_{#3}, #2_{#3}))_{#3 \in #4}}}

\newcommandx{\sequencen}[2][2=n\in\N]{\ensuremath{\{ #1_n, \eqsp #2 \}}}

\newcommand{\wrt}{w.r.t.}

\def\eg{\emph{e.g.}}

\def\target{\pi}

\newcommandx{\targetkern}[1][1=N]{\Pi_{#1}}
\newcommandx{\utargetkern}[1][1=N]{\Gamma_{#1}}
\newcommandx{\rollingestim}[4][1=N,2=K_0,3=K,4=f]{\Pi_{(#1,#2),#3}(#4)}

\def\normconst{\operatorname{Z}}

\def\proposal{\lambda}

\def\bound{\omega}
\def\epssmallisir{\epsilon_N}

\def\driftconstisir{\kappa_N}

\newcommand{\ki}{{k}}

\def\MKQ{{\rm Q}}

\newcommandx{\MKisirjoint}[1][1=N]{\mathsf{\mathbf{P}}_{#1}}
\def\xijoint{{\boldsymbol{\xi}}}
\newcommandx{\MKisir}[1][1=N]{\mathsf{P}_{#1}}
\newcommandx{\MKisirpopN}[1][1=N]{\bar{\mathsf{P}}_{#1}}
\def\Xset{\mathbbm{X}}
\def\Xsigma{\mathcal{X}}

\def\Zset{\mathbbm{Z}}
\def\Zsigma{\mathcal{Z}}

\newcommand{\chunk}[3]{#1_{#2:#3}}

\newcommand{\chunku}[3]{#1^{#2:#3}}

\newcommand{\chunkum}[4]{#1^{#2:#3 \setminus \{#4\}}}

\def\biasconst{\varsigma^{\scriptsize{\mbox{{\it bias}}}}}
\def\mseconst{\varsigma^{\scriptsize{\mbox{{\it mse}}}}}
\def\covconst{\varsigma^{\scriptsize{\mbox{{\it cov}}}}}
\def\weightfunc{w}

\newcommandx{\XtryK}[1][1=N]{\mathsf{K}_{#1}}

\newcommand{\isir}{i-SIR}

\def\eTarget{\boldsymbol{\varphi}_N}
\def\eTargetmarginx{\boldsymbol{\pi}_N}

\newcommand{\supnorm}[1]{\|#1\|_{\infty}}

\def\Idd{\mathbf{I}}

\def\partopopN{\boldsymbol{\Lambda}_N}

\newcommand{\PhiN}[1]{\Phi_N (#1)}
\newcommand{\vectmean}{\boldsymbol{\mu}}
\newcommand{\covmat}{\boldsymbol{\Sigma}}
\newcommand{\kmax}{k_{\scriptsize{\mbox{{\it max}}}}}

\def\hpdconstant{\varsigma^{\scriptsize{\mbox{{\it hpd}}}}}
\def\MKQ{\operatorname{Q}}
\def\State{Z}
\def\taumix{t_{\operatorname{mix}}}
\def\invariantQ{\pi}
\newcommandx{\normop}[2][2=]{\Vert{#1}\Vert_{{#2}}}
\newcommand{\absD}[1]{\left\vert #1\right\vert}
\newcommand{\abs}[1]{\vert #1\vert}
\def\rhs{right-hand side}

\def\SNIS{\ensuremath{\mbox{SNIS}}}
\def\ISIR{\ensuremath{\mbox{i-SIR}}}
\def\UnbiasedPIMH{\ensuremath{\mbox{Unbiased-PIMH}}}
\def\bX{\mathbf{x}}
\def\loglikelihood{\mathcal{L}} 

\title{BR-SNIS: Bias Reduced Self-Normalized Importance Sampling}
\author{
\textbf{Gabriel Cardoso} \\
Centre de Math\'ematiques appliqu\'ees, \\
Ecole polytechnique, \\
\texttt{gabriel.victorino-cardoso@polytechnique.edu}\\

\textbf{Sergey Samsonov} \\
HSE University \\
 
\textbf{Achille Thin} \\
, \\
AgroParisTech, \\

\textbf{Eric Moulines} \\
Centre de Math\'ematiques appliqu\'ees, \\
Ecole polytechnique, \\

\textbf{Jimmy Olsson} \\
Department of Mathematics, \\
KTH Royal Institute of Technology.\\
}

\begin{document}
\maketitle
\begin{abstract}
Importance Sampling (IS) is a method for approximating expectations under a target distribution using independent samples from a proposal distribution and the associated importance weights. In many applications, the target distribution is known only up to a normalization constant, in which case self-normalized IS (\SNIS) can be used. While the use of self-normalization can have a positive effect on the dispersion of the estimator, it introduces bias. In this work, we propose a new method, {\SUISIR}, whose complexity is essentially the same as that of {\SNIS} and which significantly reduces bias without increasing the variance. This method is a wrapper in the sense that it uses the same proposal samples and importance weights as {\SNIS}, but makes clever use of iterated sampling--importance resampling (\isir) to form a bias-reduced version of the estimator. We furnish the proposed algorithm with rigorous theoretical results, including new bias, variance and high-probability bounds, and these are illustrated by numerical examples.
\end{abstract}
\vspace{-5pt}
\section{Introduction}
\label{sec:introduction}
\vspace{-5pt}
\paragraph{Background and previous work: }
\emph{Importance sampling} \cite{kahn1953methods, agapiou2017importance} (IS) is a classical Monte Carlo technique for estimating expectations under some given probability distribution (the \emph{target}) on the basis of a sample of draws from a different distribution (the \emph{proposal}). In the modern era of artificial intelligence and statistical machine learning, characterized by large computational resources and Bayesian inference, IS technologies are enjoying a revival; see, \emph{e.g.}, \cite{niknejad2019external,kuzborskij2021confident} and \cite{elvira2021advances} for a recent survey. The method is not only relevant to situations where sampling from the target is intractable; it can also be used to achieve variance reduction \cite{lamberti2018double}. When the proposal is dominating the target---in the sense that the support of the latter is contained in the support of the former---unbiased estimation can be achieved by assigning each draw an \emph{importance weight} given by the likelihood ratio between the target and the proposal. In the very common case where the target is known only up to a normalizing constant, consistent estimation can still be achieved by simply normalizing each importance weight by the total weight of the sample; however, since such \emph{self-normalized importance sampling} (\SNIS) involves ratios of random variables, the procedure can only be implemented at the cost of bias, which can be significant in some applications.

More precisely, let $(\Xset, \Xsigma)$ be some state space and $\target(\rmd x) \propto \weightfunc(x) \proposal(\rmd x)$ a given target probability distribution, where $\weightfunc$ and $\proposal$ are a positive weight function and a proposal probability distribution on $(\Xset, \Xsigma)$, respectively, such that the normalizing constant $\proposal(\weightfunc) =
\int \weightfunc(x) \proposal(\rmd x)$
(this will be our generic notation for Lebesgue integrals)
of $\target$ is finite. The {\SNIS} estimator is given by
\begin{equation}
{\textstyle
\label{eq:self-normalized}
\Pi_M f(\chunku{X}{1}{M})
= \sum_{i=1}^M \omega_M^i f(X^i)\eqsp, \qquad \omega_M^i= \weightfunc(X^i)/ \sum_{\ell = 1}^M \weightfunc(X^\ell)
}\end{equation}
where $\chunku{X}{1}{M} = (X^1, \ldots, X^M)$ are independent draws from $\proposal$, and can be used to approximate $\target(f) = \int f(x) \target(\rmd x)$ for any test function $f$ such that $\target(|f|) < \infty$.
The estimator \eqref{eq:self-normalized} can be calculated without knowledge of the normalizing constant $\proposal(\weightfunc)$, which is intractable in general.

The \SNIS\ estimator is known to be biased; provided that $\proposal(\weightfunc^2) < \infty$, the bias and mean-squared error (MSE) of the \SNIS\ estimator \eqref{eq:self-normalized} over bounded test functions $f$ satisfying $\supnorm{f} \leq 1$ are given
respectively (see \cite[Theorem~2.1]{agapiou2017importance}) by
\begin{equation}
\label{eq:bias-MSE-self-normalized}
| \PE[ \Pi_M f(\chunku{X}{1}{M}) 
] - \target(f) |\leq (12/M) \kappa[\target, \proposal],  \quad
\PE[ \{ \Pi_M f(\chunku{X}{1}{M}) 
- \target(f) \}^2] \leq (4/M) \kappa[\target, \proposal]\eqsp, 
\end{equation}
where  $\kappa[\target,\proposal] = \proposal(\weightfunc^2)/\proposal^2(\weightfunc)$. 
Although IS is primarily intended to approximate integrals in the form $\target(f)$, it can also be used to generate unweighted samples being approximately distributed according to $\target$.  In this paper, we consider \emph{iterated sampling importance resampling} (\isir), proposed in \cite{tjelmeland2004using}; see \citep{andrieu2010particle,lee2010utility,lee2011auxiliary,andrieu2018uniform}. The \isir\ can be seen as an iterative application of the \emph{sampling importance resampling} (SISR) algorithm proposed by \cite{rubin1987comment}; the $k$-th iteration is defined as follows. Given a state $Y_k \in \Xset$, (i) set $X^1_{k+1}=Y_k$ and draw $\chunku{X_{k+1}}{2}{N}$ independently from the proposal distribution $\proposal$; (ii) compute, for $i \in \{1, \dots, N\}$, the normalized importance weights $\omega^i_{N,k+1} = \weightfunc(X^i_{k+1})/\sum_{\ell=1}^N \weightfunc(X^\ell_{k
+1})$; (iii) select $Y_{k+1}$ from the set $\chunku{X_{k+1}}{1}{N}$ by choosing $X_{k+1}^i$ with probability $\omega^i_{N,k+1}$. In the following, $Y_{k + 1}$ and $\chunku{X_{k+1}}{1}{N}$ will be referred to as the \emph{state} and the \emph{candidate pool}, respectively.
 Following \citep{tjelmeland2004using} (see~\Cref{sec:main}), \isir\ may be viewed (up to an irrelevant permutation of the samples) as a two-stage Gibbs sampler targeting an extended probability distribution $\eTarget$ on an enlarged state space including the state as well as the candidate pool. As this extended distribution allows $\pi$ as a marginal with respect to the state, one can expect the marginal distribution of the generated states $(Y_k)_{k \in \nset}$, forming themselves a Markov chain, to approach the target $\target$ of interest as $k$ tends to infinity. 
 \vspace{-5pt}
\paragraph{This paper:} In \isir, the only function of the candidate pool is to guide the states selected at stage (iii) towards the target. Thus, since all rejected candidates are discarded, the approach results generally in a large waste of computational work. Thus, in the present paper we propose to recycle \emph{all} the generated samples by incorporating all the proposed candidates $\chunku{X_k}{1}{N}$ into the estimator rather than only the selected candidate $Y_k$. We proceed in three steps. First, we show that under the stationary distribution $\eTarget$ of the process $(Y_k, \chunku{X_k}{1}{N})_{k \in \nset}$ generated by \isir, the expectation of $\targetkern f (\chunku{X_k}{1}{N})$ (given by \eqref{eq:self-normalized}) equals $\target(f)$ for every valid test function $f$ (see \Cref{thm:unbiasedness}). Second, we establish that since {\isir} is nothing but a systematic-scan Gibbs sampler, the two processes $(\chunku{X_k}{1}{N})_{k \in \nset}$ and $(Y_k)_{k \in \nset}$ are  \emph{interleaving} (see \Cref{theo:main-properties-deterministic-scan});
thus, if $(Y_k)_{k \in \nset}$ is uniformly geometrically ergodic, so is $(\chunku{X_k}{1}{N})_{k \in \nset}$ with the same mixing rate $\kappa_N$.
Third, as the main result of the present paper, we establish a novel $\mathcal{O}(\kappa_N^k/N)$ bound on the bias of the estimator $\targetkern f (\chunku{X_k}{1}{N})$ (see \Cref{theo:bias-i-SIR-recycling}), where the exponentially diminishing factor $\kappa_N^k$ indicates a drastic bias reduction \emph{vis-\`a-vis} the standard IS estimator \eqref{eq:self-normalized} based on i.i.d. samples.
As a consequence, approximating $\target(f)$ by the average of $(\targetkern f (\chunku{X_k}{1}{N}))_{\ell = k_0 + 1}^k$, where the ``burn-in'' period $k_0$ should be chosen proportionally to the mixing time of the process, yields an estimator whose bias can be furnished with a bound which is, roughly, proportional to $\kappa_N^{k_0}$ and inversely proportional to the total number $M = k N$ of samples generated in the algorithm (see \Cref{theo:bias-mse-rolling}).
To complete the theoretical analysis of these estimators, we also equip the same with variance bounds. The procedure of recycling, as described above, all the samples generated in the {\isir} and to incorporate, at negligible computational cost, the same into the final estimator, will from now on be referred as \SUISIR. Finally, we test numerically the proposed estimators and illustrate how a significant bias reduction relatively to the standard {\isir} can be obtained at basically no cost.

To sum up, our contribution is twofold, since we
\begin{itemize}[leftmargin=*,nosep]
    \item[--] propose a new algorithm, {\SUISIR}, which makes better use of the available computational resources by recycling the candidate pool generated at each iteration of i-SIR.
    \item[--] furnish the proposed algorithm with rigorous theoretical results, including novel bias, variance, and high-probability bounds which support our claim that sample recycling may lead to drastic bias reduction without impairing the variance.
\end{itemize}

\vspace{-5pt}
\section{Main results}
\label{sec:main:results}
\subsection{Statements}
\vspace{-9pt}
\label{sec:main}
The {\isir} algorithm can be interpreted as a systematic-scan two-stage Gibbs sampler, alternately sampling from the full conditions of an extended target $\eTarget$ on the product space of states and candidate pools. Once the extended target $\eTarget$ is properly defined, these full conditionals can be retrieved from a dual representation of $\eTarget$ presented in \Cref{thm:Gibbs:duality}. In order to define $\eTarget$, we introduce the Markov kernel (see \Cref{sec:isir-algorithm} for comments)
\begin{equation}
\label{eq:definition-psiN}
{\textstyle
\partopopN(y, \rmd \chunku{x}{1}{N}) =N^{-1} \sum \nolimits_{i=1}^N   \delta_y(\rmd x^i) \prod\nolimits_{j \ne i} \proposal(\rmd x^j)}
\end{equation}
on $\Xset \times \Xsigma^{\varotimes N}$, which describes probabilistically the sampling operation (i) in {\isir}. 
Using the kernel $\partopopN$ we may now define properly the extended target $\eTarget$ as the probability law
\begin{equation}
\label{eq:extended-target}
{\textstyle
\eTarget(\rmd (y, \chunku{x}{1}{N})) = \target(\rmd y) \partopopN(y, \rmd \chunku{x}{1}{N}) = N^{-1} \sum_{i=1}^N \target(\rmd y)
\delta_y(\rmd x^i) \prod_{j \ne i} \proposal(\rmd x^j)}
\end{equation}
on $(\Xset^{N + 1}, \Xsigma^{\varotimes (N + 1)})$. Note that since for every $A \in \Xsigma$, $\eTarget(\indi{A \times \Xset^N}) = \target(A)$, the target $\target$ coincides with the marginal of $\eTarget$ with respect to the state. Moreover, it is easily seen that $\partopopN$ provides the conditional distribution, under $\eTarget$, of the candidate pool given the state.
Defining the kernels
\begin{equation} \label{eq:def:utargetkern}
{\textstyle
\utargetkern(\chunku{x}{1}{N}, \rmd y) = N^{-1} \sum_{i=1}^N \weightfunc(x^i) \delta_{x^i}(\rmd y), \quad \targetkern(\chunku{x}{1}{N}, \rmd y) = {\utargetkern(\chunku{x}{1}{N}, \rmd y)}/{\utargetkern \1_{\Xset} (\chunku{x}{1}{N})}
}
\end{equation}
on $\Xset \times \Xsigma^{\varotimes N}$,
the marginal distribution $\eTargetmarginx$ of $\eTarget$ with respect to $\chunku{x}{1}{N}$ is given by
\begin{equation} \label{eq:marginal-joint}
{\textstyle \eTargetmarginx(\rmd \chunku{x}{1}{N}) = {\proposal(\weightfunc)}^{-1} \utargetkern \indi{\Xset}(\chunku{x}{1}{N}) \prod_{j=1}^N \proposal(\rmd x^j).}
\end{equation}
It is interesting to note that the marginal $\eTargetmarginx$ has a probability density function, proportional to $\utargetkern \indi{\Xset}(\chunku{x}{1}{N}) = \sum_{i=1}^N \weightfunc(x^i) / N$, with respect to the product measure $\proposal^{\varotimes N}$. Using \eqref{eq:marginal-joint}, we immediately obtain the following result.

\begin{theorem}[duality of extended target] \label{thm:Gibbs:duality}
For every $N \in \nsets$,
\begin{equation} \label{eq:duality}
\eTarget(\rmd (y, \chunku{x}{1}{N})) = \target(\rmd y) \partopopN(y,\rmd \chunku{x}{1}{N}) = \eTargetmarginx(\rmd \chunku{x}{1}{N}) \targetkern(\chunku{x}{1}{N}, \rmd y).
\end{equation}
\end{theorem}
Note that the second identity of the dual representation \eqref{eq:duality} provides also the conditional distribution, under $\eTarget$, of the state given the candidates. Consequently, {\isir} is a systematic scan  two-stage Gibbs sampler which generates a Markov chain $(X_\ki, Y_\ki)_{\ki \in \nset}$ with time-homogeneous Markov kernel
\begin{equation}
\label{eq:joint-kernel}
\MKisirjoint((y_\ki, \chunku{x_\ki}{1}{N}), \rmd(y_{\ki + 1}, \chunku{x_{\ki + 1}}{1}{N})) = \partopopN(y_\ki, \rmd \chunku{x_{\ki + 1}}{1}{N}) \targetkern(\chunku{x_{\ki + 1}}{1}{N}, \rmd y_{\ki + 1})
\end{equation}
on $\Xset^{N + 1} \times \Xsigma^{\varotimes (N + 1)}$. Note that the law $\MKisirjoint(y_k, \chunku{x_k}{1}{N}, \cdot)$ does not depend on $\chunku{x_k}{1}{N}$, which means that only the state $Y_k$ needs to be stored from one iteration to the other. Thus, $(Y_k)_{k \in \nset}$ is a Markov chain with Markov transition kernel
\begin{equation}
\label{eq:definition-MKisir}
{\textstyle
\MKisir(y_k, \rmd y_{k + 1}) = \int \partopopN(y_\ki, \rmd \chunku{x_{\ki + 1}}{1}{N}) \targetkern(\chunku{x_{\ki + 1}}{1}{N}, \rmd y_{\ki + 1}) = \partopopN \targetkern}(y_k, \rmd y_{k + 1})
\end{equation}
(where integration is \wrt\ $\chunku{x_{\ki + 1}}{1}{N}$) on $\Xset \times \Xsigma$. Given some probability distribution $\xijoint$ on $(\Xset^{N+1}, \Xsigma^{\varotimes (N+1)})$, we denote by $\PP _\xijoint$ the law of the canonical Markov chain $(X_\ki, Y_\ki)_{\ki \in \nset}$ with kernel $\MKisirjoint$ and initial distribution $\xijoint$.
Our first results establishes the unbiasedness of the estimator $\targetkern f(\chunku{X}{1}{N})$ under $\eTarget$.
\begin{theorem} \label{thm:unbiasedness}
For every $N \in \nsets$ and $\target$-integrable function $f$,
\[
{\textstyle\int \targetkern f(\chunku{x}{1}{N}) \eTargetmarginx(\rmd \chunku{x}{1}{N}) = \target(f)\eqsp.}
\]
\end{theorem}
The proof of \Cref{thm:unbiasedness} is postponed to \Cref{sec:proof:unbiasedness}. Next, we present theoretical bounds on the discrepancy, in terms of bias, MSE and covariance, between $\targetkern f(\chunku{X_k}{1}{N})$ and $\target(f)$, for bounded target functions $f$, when the {\isir} chain is initialized according to an arbitrary distribution $\xijoint$. We will work under the following assumption.
\begin{assumption} \label{ass:boundedness:weights}
It holds that $\bound = \supnorm{\weightfunc} / \proposal(\weightfunc) < \infty$.
\end{assumption}
Under \Cref{ass:boundedness:weights}, the state and candidate-pool Markov chains $(Y_k)_{k \in \nset}$ and $(\chunku{X}{1}{N})_{k \in \nset}$ can be shown to be uniformly geometrically ergodic with mixing rate and mixing-time upper bound
\begin{equation} \label{eq:mixing:rate:time}
\driftconstisir = (2 \bound - 1) / (2 \bound + N - 2), \quad
\TmixN 
= \lceil - \ln 4 / \ln \driftconstisir \rceil,
\end{equation}
respectively; see \Cref{theo:isir_uniform_ergodicity} below for details.
Here the mixing time $\TmixN$ grows logarithmically with the sample size $N$. The exact value of $\TmixN$ is likely to be grossly pessimistic, but we conjecture that the logarithmic dependence in the minibatch size holds true. In addition, under \Cref{ass:boundedness:weights} we define the constants
\begin{equation}
\label{eq:bias:mse:cov:constants}
\begin{split}
\biasconst &= 4(\kappa[\target,\proposal] + 1 + \bound) \\
\mseconst_i &= 4(\kappa[\target,\proposal] \1_{\{0, 1\}}(i) + (1+\bound)^2 \1_{\{1, 2\}}(i)), \quad  \covconst_i = \biasconst (\mseconst_i)^{1/2}
, \quad i \in \{0, 1, 2\}.
\end{split}
\end{equation}
With these definitions, the following holds true.
\begin{theorem}
\label{theo:bias-i-SIR-recycling}
Assume \Cref{ass:boundedness:weights}. Then for every initial distribution $\xijoint$ on $(\Xset^{N + 1}, \Xsigma^{\varotimes (N + 1)})$, bounded measurable function $f$ on $(\Xset, \Xsigma)$ such that $\supnorm{f} \leq 1$, $N \geq 2$, and $(\ki, \ell) \in (\nsets)^2$,
\begin{enumerate}[label=(\roman*),nosep,leftmargin=*]
\item \label{item:theo:bias-i-SIR-recycling:bias}
$\left| \PE_{\xijoint}[\targetkern f(\chunku{X_\ki}{1}{N})] - \target(f) \right| \leq  \biasconst  (N - 1)^{-1} \driftconstisir^{\ki - 1}$, 
\item \label{item:theo:bias-i-SIR-recycling:mse}
$\PE_{\xijoint}[\{\targetkern f(\chunku{X_\ki}{1}{N}) - \target(f)\}^2] \leq \sum_{i = 0}^2 \mseconst_i (N - 1)^{-1 - i/2}$, 
\item \label{item:theo:bias-i-SIR-recycling:cov}
$\left| \PE_{\xijoint}[\{ \targetkern f(\chunku{X_\ki}{1}{N}) - \target(f) \} \{ \targetkern f(\chunku{X_{\ki + \ell}}{1}{N}) - \target(f) \}] \right| \leq \driftconstisir^{\ell - 1}  \sum_{i = 0}^2 \covconst_i (N - 1)^{-(3 - i/2)/2}$,  
\end{enumerate}
where constants are given in \eqref{eq:mixing:rate:time} and \eqref{eq:bias:mse:cov:constants}.
\end{theorem}
 It is worth noting that the bias decreases inversely with the number of candidates and exponentially with the number of iterations (the mixing time of the chain also depends on $N$). The MSE is also inversely proportional to the number of candidates $N$.
In the light of the previous results, it is natural to consider an estimator formed by an average across the IS estimators $(\targetkern f(\chunku{X_k}{1}{N}))_{\ki \in \nset}$ associated with the candidate pools generated at the different {\isir} iterations. To mitigate the bias, we remove a ``burn-in'' period whose length $\ki_0$ should be chosen proportional to the mixing time $\TmixN$ of the Markov chain $\sequence{Y}[\ki][\nset]$ (which turns out to coincide with that or the chain $\sequence{\chunku{X}{1}{N}}[k][\nset]$; see \Cref{sec:proofs}). This yields the estimator
\begin{equation}
\label{eq:final-estimator}
{\textstyle    \rollingestim[\ki_0][\ki][N][f]= (\ki-\ki_0)^{-1} \sum_{\ell=\ki_0+1}^\ki \targetkern f(\chunku{X_\ell}{1}{N})
}\end{equation}
of $\target(f)$. The total number of samples (generated by the proposal $\proposal$) underlying this estimator is $M = (N - 1) \ki$. Importantly, all the importance weights included in the estimators are obtained as a by-product of the {\isir} schedule; thus, it is, for a given budget of simulations (\ie, under the constraint that $(\ki-\ki_0)N$ is constant), possible to compute $\rollingestim[\ki_0][\ki][N][f]$ for different values of $\ki_0$, $\ki$ and $N$ with a negligible computational cost. We denote by $\burningloss= (\ki-\ki_0)/\ki$ the ratio of the number of candidate pools used in the estimator to the total number of sampled such pools. Note that this type of estimator was already suggested by \cite{Tjelmeland2004UsingAM}, and also appears in \cite{pmlr-v130-schwedes21a}.

Our final main result provides bounds on the bias and the MSE of the estimator \eqref{eq:final-estimator} as well as a high-probability bound for the same. Define $\biasrunconst = 4 \TmixN \biasconst / 3$,
$\mserunconst_i = \mseconst_{(i + 1) \wedge 2} \1_{\{0, 2\}}(i) + (8/3) \TmixN \covconst_i$, $i \in \{0, 1, 2\}$,
$\mserunconst = \mserunconst_0 +   \mserunconst_1{(N-1)}^{-1/4} + \mserunconst_2{(N-1)}^{-1}$,
and $\msesnis = (4/M) \kappa[\target, \proposal]$, see \eqref{eq:bias-MSE-self-normalized}.
\begin{theorem}
\label{theo:bias-mse-rolling}
Assume \Cref{ass:boundedness:weights}. Then the following holds true for every initial distribution $\xijoint$ on $(\Xset^{N+1}, \Xsigma^{\varotimes (N+1)})$, bounded measurable function $f$ on $(\Xset, \Xsigma)$ such that $\supnorm{f} \leq 1$, and $N \geq 2$.
\begin{enumerate}[label=(\roman*),nosep,
leftmargin=*]
    \item \label{item:theo:bias-mse-rolling:bias}
    $\left| \PE_{\xijoint}[\rollingestim[{\ki}_0][\ki][N][f]] - \target(f) \right| \leq  \biasrunconst(\burningloss M)^{-1}  4^{- \ki_0 / \TmixN}$
    \item \label{item:theo:bias-mse-rolling:mse}
    $\PE_{\xijoint}[\{ \rollingestim[\ki_0][\ki][N][f] - \target(f)\}^2]
    \leq \msesnis[\burningloss M]
    + \mserunconst (\burningloss M)^{-1}{(N-1)}^{-1/2}$
    \item 
    For every $\delta \in \ooint{0,1}$, $| \rollingestim[\ki_0][\ki][N][f] - \pi(f) | \leq \hpdconstant (\burningloss M)^{-1/2} (\log(4 / \delta))^{1/2}$ with probability at least $1-\delta$, where $\hpdconstant = 664 \bound$.
\end{enumerate}
\end{theorem}
\paragraph{Bootstrap:} As established in \Cref{theo:bias-mse-rolling}, the bias of the {\SUISIR} estimator decreases exponentially with the burn-in period $k_0$, leading to potentially significant bias reduction with respect to \SNIS. Still, using a large $k_0$ is done at a price of increased overall MSE (mainly through the term $\msesnis[\burningloss M]$  in \Cref{theo:bias-mse-rolling}(ii), which is directly related to $k_0$ via $\burningloss$). A natural way to reduce the variance is to use bootstrap. More precisely, we first apply a random permutation to the samples and re-compute {\SUISIR} on the basis of the bootstrapped samples. After this, we produce a final estimator by averaging over the bootstrapped {\SUISIR} replicates. In most applications, the major computational bottleneck consists of sampling from $\proposal$ and evaluating $\weightfunc$ and $f$ at the samples; thus, the additional operations that this bootstrap approach entails are computationally cheap.
Therefore, in our experiments, we use bootstrap in combination with the choice $k_0 = k - 1$ (in order to minimize the bound in \Cref{theo:bias-mse-rolling}(i)).


\vspace{-9pt}
\subsection{Elements of proofs}
\label{sec:proofs}
\paragraph{Ergodic properties of {\isir}: } The systematic scan two-stage Gibbs sampler is a well-studied MCMC algorithmic  structure, and we summarize its most important properties in \Cref{theo:main-properties-deterministic-scan} below; see \cite{liu1994covariance,andrieu2016random} and \cite[Chapter~9]{robert:casella:2013} as well as the references therein. In particular, as shown in \cite{liu1994covariance}, the state and candidate-pool Markov chains $\sequence{Y}[k][\nset]$ and $\sequence{\chunku{X}{1}{N}}[k][\nset]$ satisfy a duality property referred to as \emph{interleaving} (\Cref{theo:main-properties-deterministic-scan}(iii)).
\begin{theorem}
    \label{theo:main-properties-deterministic-scan}
    Assume that for every $x \in \Xset$, $\weightfunc(x) >0$, $\lambda(\weightfunc) < \infty$ and that there exists a set $C \in \Xsigma$ such that $\lambda(C) > 0$ and $\sup_{x \in C} \weightfunc(x)/\lambda(\weightfunc) < \infty$. Then,
    \begin{enumerate}[label=(\roman*),leftmargin=*,nosep]
        \item \label{item:main-properties-deterministic-scan:harris}
        the Markov kernel $\MKisirjoint$ is Harris recurrent and ergodic with unique invariant distribution $\eTarget$.
        \item \label{item:main-properties-deterministic-scan:reversible}
        the Markov kernel $\MKisir$ is $\target$-reversible, Harris recurrent and ergodic.
        \item \label{item:main-properties-deterministic-scan:interleaving}
        the two Markov chains $\sequence{Y}[k][\nset]$ and $\sequence{\chunku{X}{1}{N}}[k][\nset]$ are conjugate of each other with the interleaving property, {\ie}, for every initial distribution $\xijoint$ and $\ki \in \nset$, under $\PP_{\xijoint}$,
        \begin{enumerate}[leftmargin=*,nosep]
            \item $\chunku{X_k}{1}{N}$ and $\chunku{X_{k+1}}{1}{N}$ are conditionally independent given $Y_k$,
            \item $Y_k$ and $Y_{k+1}$ are conditionally independent given $\chunku{X_{k+1}}{1}{N}$;
            \item moreover, under $\PP_{\eTarget}$, $(Y_k, \chunku{X_{k-1}}{1}{N})$ and $(Y_k, \chunku{X_k}{1}{N})$ are identically distributed.
        \end{enumerate}
    \end{enumerate}
\end{theorem}
The ergodic behavior of the i-SIR algorithm has been studied in many works; see \cite{lee2011auxiliary,lindsten2015uniform,andrieu2018uniform} in particular. The analysis is particularly simple under the assumption that the importance weight function $\weightfunc$ is bounded, as imposed by \Cref{ass:boundedness:weights}. Recall that the \emph{total variation-distance} between two probability measures $\xi$ and $\xi'$ on $(\Xset, \Xsigma)$ is given by $\tvdist{\xi}{\xi'}= \sup_{g : \osc(g) \leq 1} \{ \xi(g) - \xi'(g) \}$, where $\osc(g)= \sup_{(x,x') \in \Xset^2}|g(x) - g(x')|$ denotes the oscillator norm of a measurable function $g$. The following result establishes the uniform geometric ergodicity of the state chain $(Y_k)_{k \in \nset}$.
\begin{theorem}
  \label{theo:isir_uniform_ergodicity}
    Assume \Cref{ass:boundedness:weights}. Then for every $N \geq 2$, $y \in \Xset$ and $k \in \nset$, $\tvdist{\MKisir^{k}(y, \cdot)}{\target} \leq \driftconstisir^k$,  where $\driftconstisir$ is given in \eqref{eq:mixing:rate:time}.
\end{theorem}
The proof is given in \cite{lindsten2015uniform,andrieu2018uniform}, but we provide it in \Cref{supp:theo:isir_uniform_ergodicity} for completeness.
For uniformly ergodic Markov chains, it is often more appropriate to work with the mixing time
\begin{equation}
\label{eq:mixing-time}
\min \{k \in \nset: \sup\nolimits_{y \in \Xset} \tvdist{\MKisir^k(y, \cdot)}{\target}  \leq 1 / 4 \} \leq \TmixN
\end{equation}
(where $ \TmixN$ is given in \eqref{eq:mixing:rate:time}), \ie, the number of time steps required for the distribution of the chain to be within a certain total variation distance from its stationary distribution \cite{aldous1997mixing,hsu2019mixing}.
An interesting consequence of the interleaving property is that if the Markov chain $\sequence{Y}[k][\nset]$ is (geometrically) ergodic, then the Markov chain $\sequence{\chunku{X}{1}{N}}[k][\nset]$ is (geometrically) ergodic as well with the same mixing time; see \cite[Corollary~9.14]{robert:casella:2013}).
\paragraph{Bias of the {\SUISIR} estimator: }
As the {\SUISIR} estimator $\targetkern f (\chunku{X_k}{1}{N})$ (where $\targetkern$ is defined in \eqref{eq:def:utargetkern}) is made up by a ratio of the two unnormalized estimators $\utargetkern f (\chunku{X_k}{1}{N})$ and $\utargetkern \1_{\Xset} (\chunku{X_k}{1}{N})$, a key ingredient in the proof of \Cref{theo:bias-i-SIR-recycling} is to bound the bias and the $p^{\scriptsize{\mbox{th}}}$ order moments of statistics defined as ratios of sums of random variables that are not necessarily independent.
The basic idea is to reduce the study of these relations to the analysis of the moments of the numerator and the denominator of these statistics and to exploit their concentration around the respective (conditional and unconditional) means.
The main results that we will use in the rest of the paper are summarized in \Cref{sec:ratio-statistics}.
\begin{lemma}
\label{lem:key-relation}
For every initial distribution $\xijoint$ on $(\Xset^{N + 1}, \Xsigma^{\varotimes (N + 1)})$, $k \in \nsets$, and bounded measurable function $f: \Xset \to \rset$, it holds that
\begin{enumerate}[label=(\roman*),leftmargin=*,nosep]
\item for every $y \in \Xset$, $\partopopN \utargetkern f(y) = (1 - 1/N) \proposal(\weightfunc f) + (1/N) \weightfunc(y) f(y)$,
\item $\CPE[\xijoint]{\utargetkern f (\chunku{X_k}{1}{N})}{Y_{k-1}} = \partopopN \utargetkern f (Y_{k-1})$, $\PP_\xijoint$-\as,
\item $\CPE[\xijoint]{\{\utargetkern f (\chunku{X_k}{1}{N}) -\partopopN \utargetkern f(Y_{k-1})\}^2}{Y_{k-1}}
= (N-1)/N^2 \proposal( \{ \weightfunc f - \proposal(\weightfunc f) \}^2), \PP_{\xijoint}$-\as
\end{enumerate}
\end{lemma}
We now have all the elements that allow us to determine the first important result of this work, namely the bias and the MSE of the estimator $\targetkern f (\chunku{X_k}{1}{N})$ of $\target(f)$.
\begin{proof}[Proof of \Cref{theo:bias-i-SIR-recycling}]
We establish the bias bound in (i) and postpone the proof of the bounds on the MSE and the covariance in (ii) and (iii) to the supplement. Define the measure $\xi(A) = \xijoint(A \times \Xset)$, $A \in \Xsigma$, and the kernel $\MKisir = \partopopN \targetkern$ on $\Xset \times \Xsigma$. Consequently, $\MKisir f (Y_{k - 1})= \PE_\xijoint [\targetkern f(\chunku{X_k}{1}{N}) \mid Y_{k-1}]$ and $\partopopN \utargetkern f(Y_{k - 1}) = \PE_\xijoint [\utargetkern f(\chunku{X_k}{1}{N}) \mid Y_{k-1}]$, $\PP_\xijoint$-\as\ Since $\sequence{Y}[k][\nset]$ is, under $\PP_\xijoint$, a Markov chain with initial distribution $\xi$ and Markov kernel $\MKisir$ (see \eqref{eq:definition-MKisir}), it holds that
$$
\PE_\xijoint[\targetkern f (\chunku{X_k}{1}{N})] = \PE_\xijoint[\MKisir f (Y_{k-1})]
= \PE_{\xijoint}[\CPE[\xijoint]{\MKisir f (Y_{k-1})}{Y_{0}}]  = \xi \MKisir^{k-1} \MKisir f.
$$
Consequently, the proof is concluded by establishing that for every $k \in \nsets$,
\begin{equation}
\label{eq:bound_phi_n}
    \left| \xi \MKisir^{k-1} \MKisir f - \target(f) \right| \leq \biasconst \driftconstisir^{k-1} (N-1)^{-1} \eqsp.
\end{equation}
On the other hand, since by \Cref{thm:unbiasedness}, $\pi (\MKisir f) = \target(f)$,
we may use \Cref{theo:isir_uniform_ergodicity} to obtain the bound
\[
|\xi \MKisir^{k-1} \MKisir f - \target(f)| = |\xi \MKisir^{k-1} \MKisir f - \target(\MKisir f)| \leq \driftconstisir^{k-1} \osc(\MKisir f).
\]
Finally, we establish \eqref{eq:bound_phi_n} by bounding $\osc(\MKisir f)$. Note that
\begin{equation}
\osc(\MKisir f) \leq 2 \left\| \MKisir f - \partopopN \utargetkern f/(\partopopN \utargetkern \1_\Xset) \right\|_\infty + 2 \left \| \partopopN \utargetkern f/(\partopopN \utargetkern \1_\Xset) - \target(f) \right \|_\infty,
\end{equation}
where, for every $y \in \Xset$, using \Cref{theo:bias-estimator-general},
\begin{multline}
\label{eq:bias-lemma}
\left| \MKisir f(y) -  {\partopopN \utargetkern f(y)}/{\partopopN \utargetkern \1_\Xset(y)} \right| \\
\leq \frac{1}{2} \{\partopopN \utargetkern \1_\Xset(y)\}^{-2}\{\partopopN[ \{ \utargetkern f - \partopopN \utargetkern f(y) \}^2](y) + 3 \partopopN[\{ \utargetkern \indi{\Xset} - \partopopN \utargetkern \indi{\Xset}(y) \}^2](y)\}\eqsp.
\end{multline}
Now, since $\partopopN \utargetkern \1_\Xset(y) \geq (1-1/N) \proposal(\weightfunc)$, we get, using \Cref{lem:key-relation}, 
\begin{align}
\left\| \MKisir f -  \frac{\partopopN \utargetkern f}{\partopopN \utargetkern \1_\Xset} \right\|_\infty
&\leq (2(N-1))^{-1} \{ \proposal(\weightfunc)\}^{-2} \{\proposal( \{ \weightfunc f - \proposal(\weightfunc f) \}^2) + 3 \proposal( \{ \weightfunc - \proposal(\weightfunc)\}^2) \} \\
\label{eq:bound-1}
&\leq 2 (N-1)^{-1} \proposal(\weightfunc^2) / (\proposal(\weightfunc))^2.
\end{align}
On the other hand, using the elementary inequality $a/b - c/d= a(d-b)/bd + (a - c)/d$, we get, as $\target(f) = \proposal(\weightfunc f)/ \proposal(\weightfunc)$,
{\small
\begin{equation}
\frac{\partopopN \utargetkern f(y)}{\partopopN \utargetkern \1_\Xset(y)} - \target(f)
= (1/N) \frac{\partopopN \utargetkern f(y)}{\partopopN \utargetkern \1_\Xset(y)}\{1 - \weightfunc(y)/ \proposal(\weightfunc) \} + (1/N) \{ \weightfunc(y) f(y) - \proposal(\weightfunc f) \}/ \proposal(\weightfunc)  \eqsp.
\end{equation}}
Finally, the bound \eqref{eq:bound_phi_n} is established by noting that
\begin{equation}
\label{eq:bound-2}
\left \| \partopopN \utargetkern f/(\partopopN \utargetkern \1_\Xset) - \target(f) \right \|_\infty \leq 2 N^{-1} \{1 + \weightfunc(y) / \proposal(\weightfunc) \} \leq 2 N^{-1} (1 + \bound)\eqsp.
\end{equation}
\vspace{-12pt}
\end{proof}
\vspace{-5pt}
\vspace{-3pt}
\subsection{Related works}
\label{sec:related:works}
The first use of the IS method, then as a variance reduction technique, dates back to the '50s; see \cite{hesterberg1995weighted,kroese2012monte} and the references therein. Today, the renewed interest in IS parallels the flurry of activity in the probabilistic ML community and its ever-increasing computational demands; thus, it is impossible to fully present the literature. We therefore limit ourselves to describing results that have inspired our work, and refer the readers to the recent reviews \cite{agapiou2017importance,elvira2021advances} for additional references.

There is clearly a plethora of modern ML applications where the standard {\SNIS} estimator may be substantially improved using the {\SUISIR} method. To mention just a selection of examples, \SNIS\ plays a key role for a robust off-policy selection strategy BY \cite{kuzborskij2021confident} (extending \cite{swaminathan2015self,metelli2018policy}),  Bayesian problems (see, \eg, \cite[Section~3]{agapiou2017importance}), Bayesian transfer learning \cite{karbalayghareh2018optimal,maddouri2022robust}, variational autoencoders \cite{chen2022fast}, inference of energy-based models \cite{lawson2019energy}, patch-based image restoration \cite{niknejad2019external} and many more.

Despite  long-standing interest in \SNIS, there are only few theoretical results. For example, \cite[Theorem~2.1]{agapiou2017importance} provides bounds on the bias and variance of {\SNIS}, results that we extend to {\SUISIR} in \Cref{theo:bias-i-SIR-recycling}. Moreover, \cite[Proposition~D.3]{metelli2018policy} provides a suboptimal variance bound based on a bound for the second-order moment. This result can be compared to the sophisticated sub-Gaussian concentration bound for {\SUISIR} obtained in \Cref{theo:bias-mse-rolling} (a result that can be obtained for {\SNIS} using the same proof mechanism; see \Cref{sec:hpd-SNIS}). Finally, \cite{kuzborskij2021confident} obtains a semi-empirical sub-Gaussian concentration inequality using the Efron-Stein estimate of variance and the Harris inequality.

As an MCMC sampling method, the {\ISIR} algorithm that has been applied successfully in many situations. It was recently used---under the alternative name \emph{conditional importance sampling}---in \cite{naesseth2020markovian} for \emph{Markovian score climbing}. In the same work, it is mentioned that it is possible to ``Rao-Blackwellize'' the gradient of the score using the proposed candidates, which is in line with the recycling argument underpinning the estimator suggested by us, but without theoretical justifications. In its most basic form, the {\ISIR} algorithm appeared in the pioneering work of \cite{tjelmeland2004using}. The same idea played a key role in the development of the \emph{particle Gibbs sampler} \cite{andrieu2010particle,andrieu2018uniform,naesseth2019elements}, which extends {\ISIR} principles to \emph{sequential Monte Carlo methods}. An approach very similar to {\SUISIR} can be taken also in this context; however, casting {\SUISIR} into the framework of particle Gibbs methods is a non-trivial problem which is the subject of ongoing work. 


\vspace{-9pt}
\section{Experimental results}
\label{sec:numerics}
\vspace{-9pt}
In this section we compare numerically the performances of \SUISIR{} and \SNIS{} in three different settings: mixture of Gaussians, Bayesian logistic regression and variational autoencoders (VAE). We leave to the supplementary material (\Cref{subsec:app:toy_problem}) the detailed numerical verification of the bounds established in \Cref{sec:main:results}.

\paragraph{Mixture of Gaussian distributions:}
\label{subsec:toy_gauss}
We start with an example where the target distribution $\target$ is a mixture of two Gaussian distributions of dimension $d=7$, as shown in \Cref{fig:toy_problem_illustration}.
The proposal distribution is a Student distribution with $\nu=3$ degrees of freedom. The test function is $f=\indi{A} - \indi{B}$, where $A$ and $B$ are a $d$-dimensional rectangle intersecting each of the modes of $\target$ (see \Cref{subsec:app:toy_problem} for precise definitions). We verify the positive effect of bootstrap in \Cref{fig:mse_with_bootstrap,fig:bias_with_bootstrap} by computing the bias and the MSE over 1000 chains for $N = 129$ for several $k$. The purple, green, and red curves correspond to a number of bootstrap rounds of $1, 21$, and $201$, respectively. We illustrate the decay of the mean Sliced Wasserstein distance (according to \cite{bonneel:hal-00881872})
with $k$ for different values of $N$ ($N=8$ purple, $N=32$  green, $N=64$  orange, and $N=128$  red) in \Cref{fig:skice_wasserstein}. The decay of the Wassertein distance is directly linked to the mixing time of the \ISIR\ kernel (see \eqref{eq:mixing:rate:time}), and hence allows us to represent the effective mixing time of the chain.
Moreover, we represent the theoretical slopes as dashed lines. This illustrates that the effective value of $\TmixN$ is smaller than its theoretical bound.
The bias and MSE for \SNIS\ with $M=25600$ are shown in black dashed lines.

We compare the bias (\Cref{fig:bias_performance}) and MSE (\Cref{fig:mse_performance}) of \SUISIR{} and \SNIS{} for a fixed budget with a total number of $M=16384$ samples.
We run the experiments $10^6$ times; we compute the bias and MSE over batches of $10^4$ replications using the true value of $\target(f)$ computed above (the boxplots in \Cref{fig:toy_perf} are therefore obtained over $100$ replications).
For the algorithm \SUISIR{}, we used $N \in \{129, 513\}$, $\ki_0 = \kmax - 1$ and $\kmax = M / (N-1)$ bootstrap rounds.
\begin{figure}[h]
    \centering
    \begin{subfigure}{0.28\textwidth}
        \includegraphics[width=\textwidth]{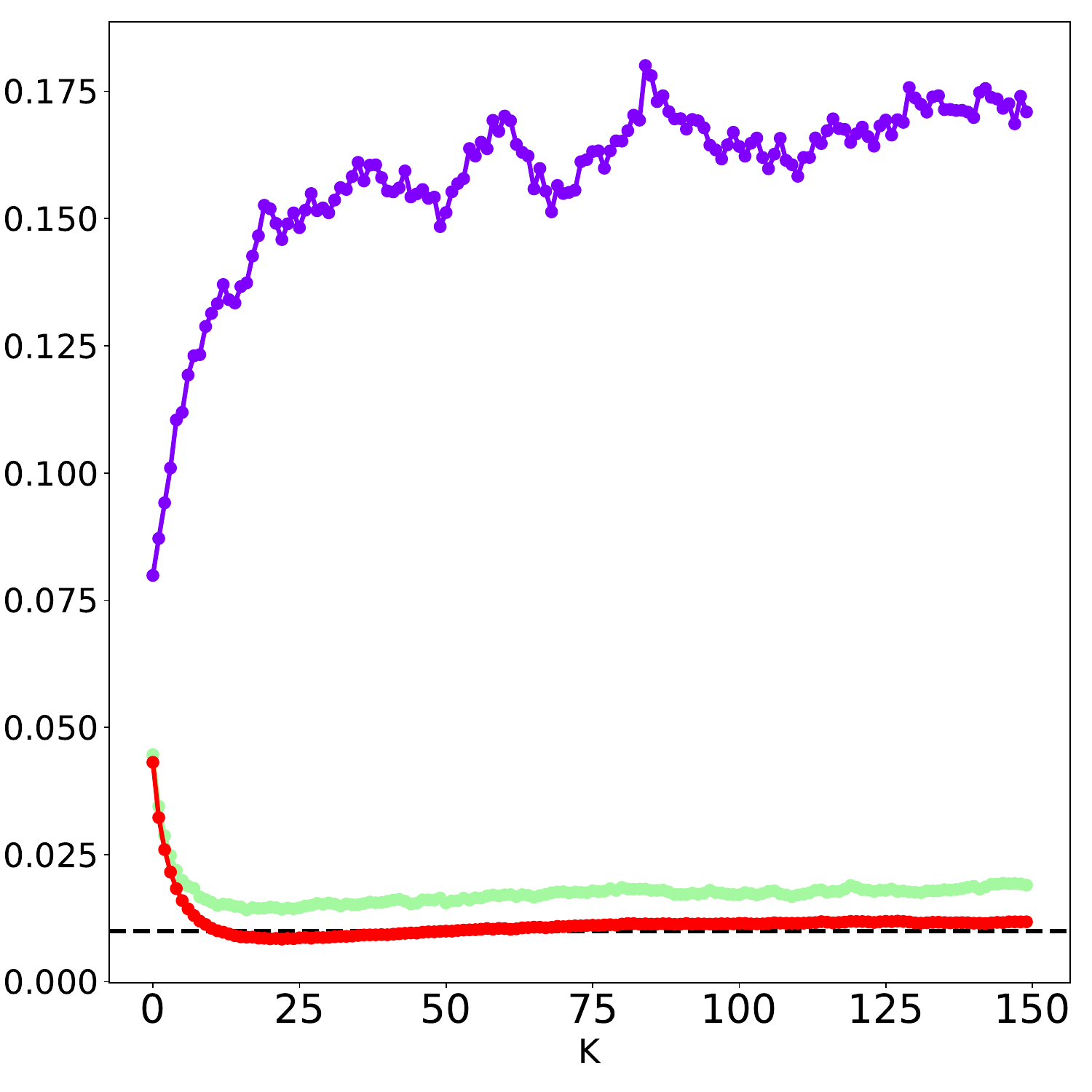}
        \caption{MSE}
        \label{fig:mse_with_bootstrap}
    \end{subfigure}
    \begin{subfigure}{0.28\textwidth}
        \includegraphics[width=\textwidth]{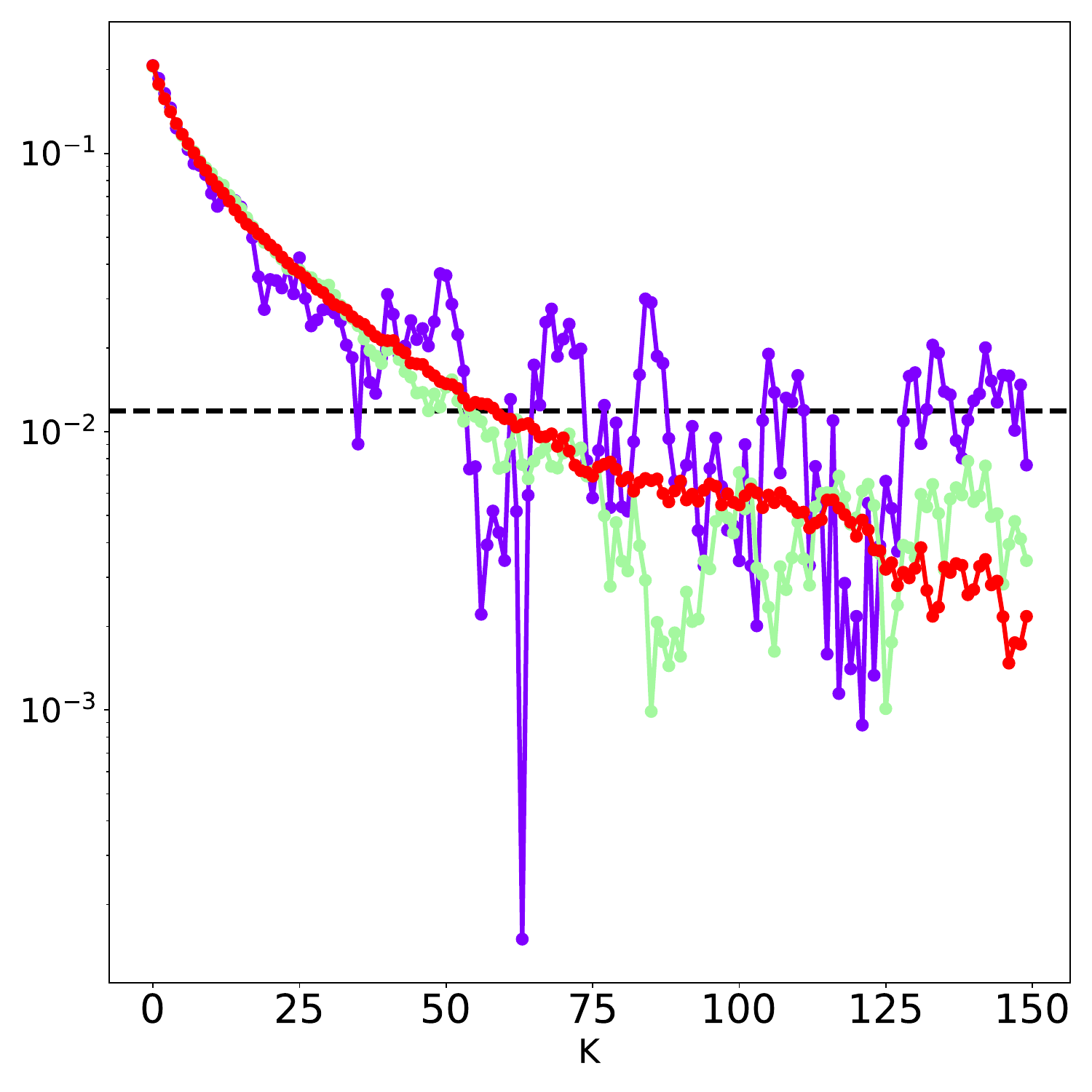}
        \caption{Bias in log scale}
        \label{fig:bias_with_bootstrap}
    \end{subfigure}
    \begin{subfigure}{0.28\textwidth}
        \includegraphics[width=\textwidth]{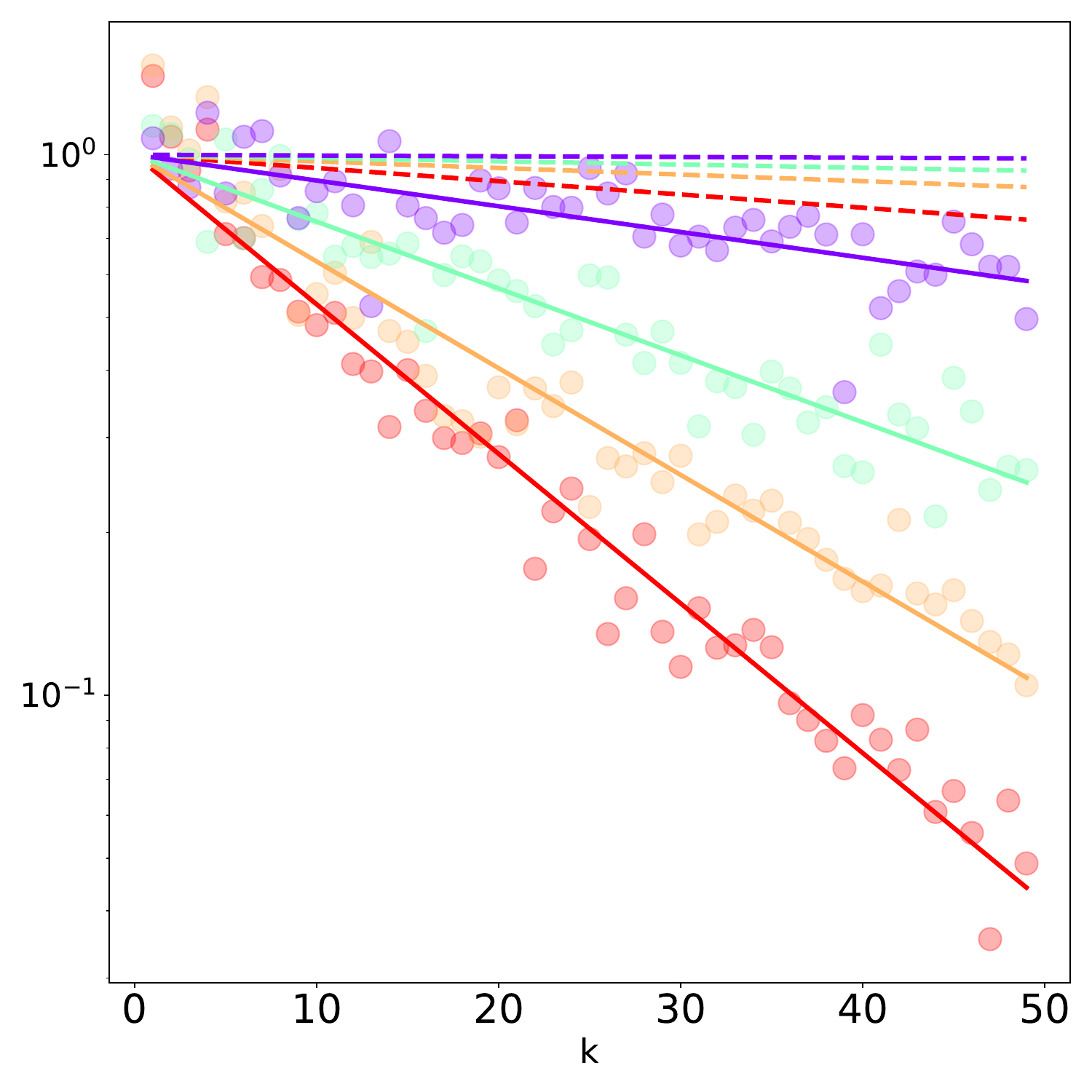}
        \caption{Sliced Wasserstein}
        \label{fig:skice_wasserstein}
    \end{subfigure}
    \caption{}
\end{figure}
\begin{figure}[h]
    \centering
    \begin{subfigure}{0.28\textwidth}
        \includegraphics[width=\textwidth]{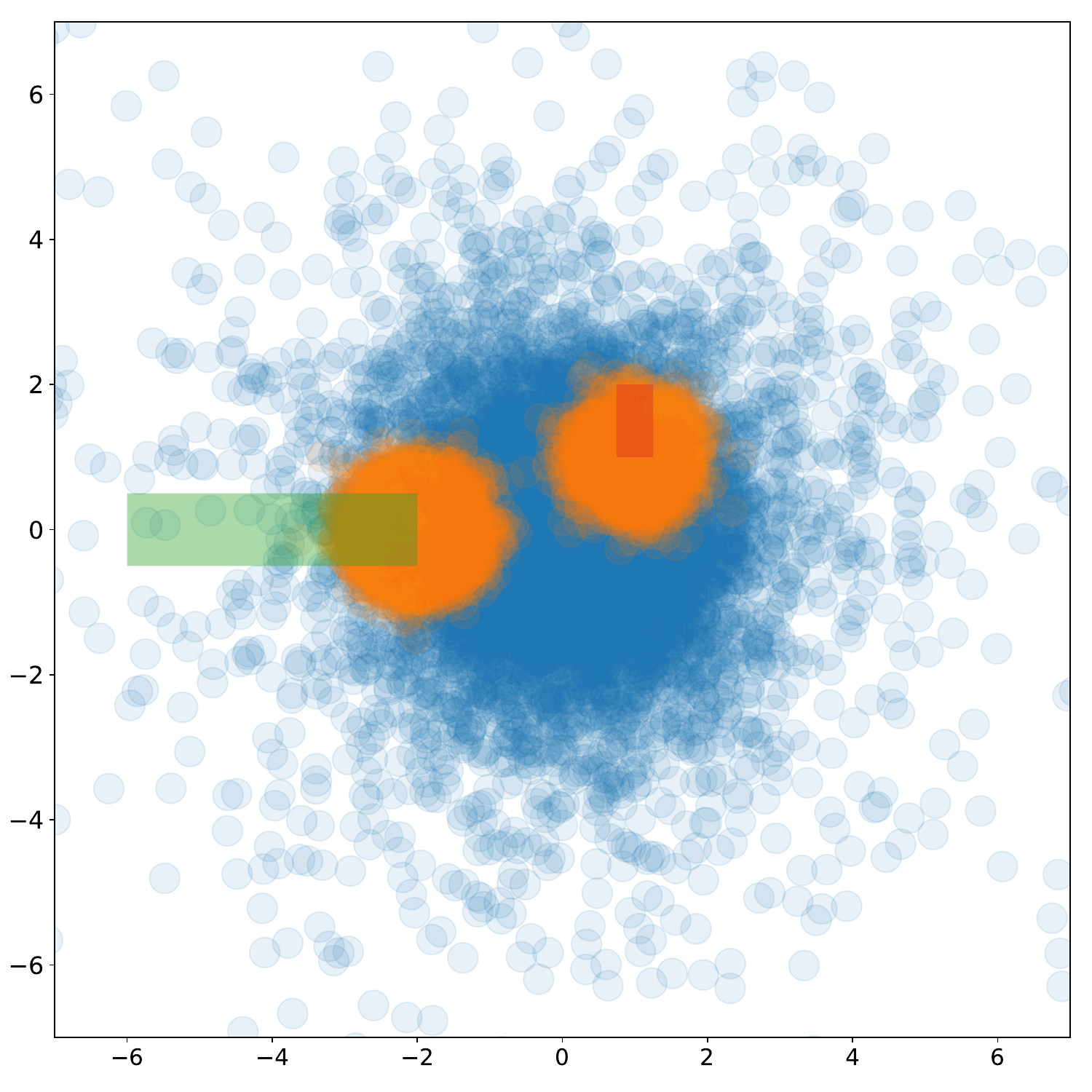}
        \caption{2d projection}
        \label{fig:toy_problem_illustration}
    \end{subfigure}
    \begin{subfigure}{0.28\textwidth}
        \includegraphics[width=\textwidth]{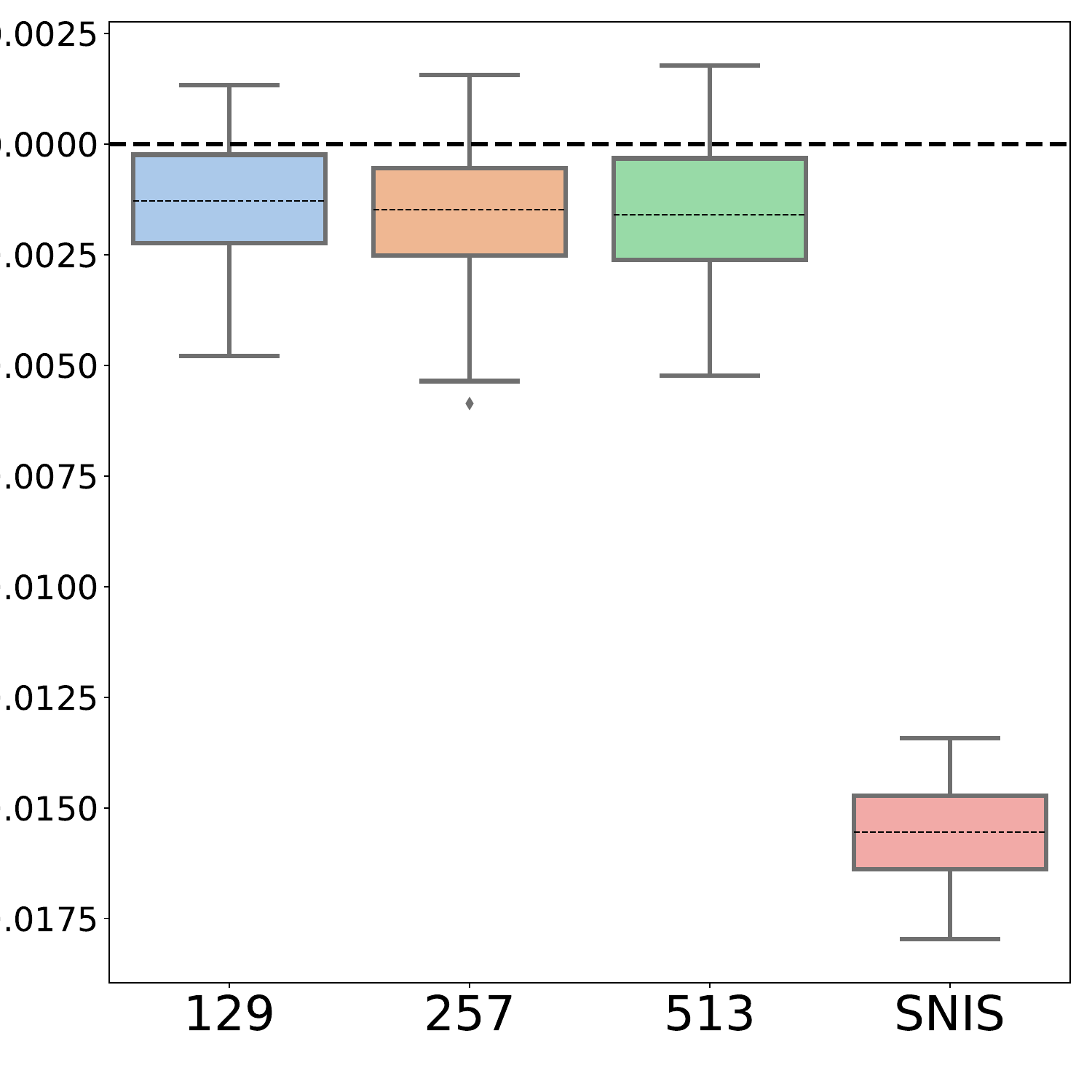}
        \caption{Bias}
        \label{fig:bias_performance}
    \end{subfigure}
    \begin{subfigure}{0.28\textwidth}
        \includegraphics[width=\textwidth]{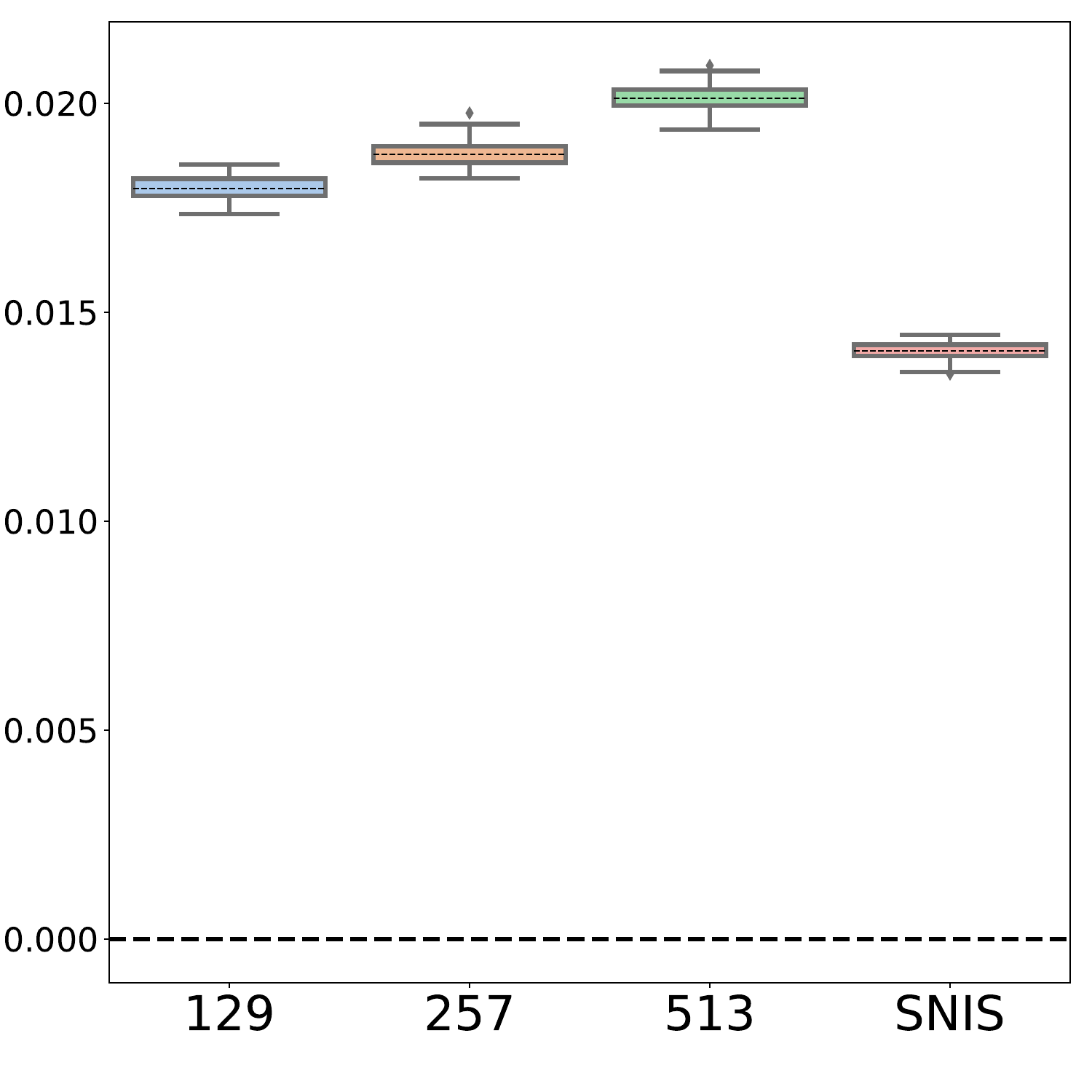}
        \caption{MSE}
        \label{fig:mse_performance}
    \end{subfigure}
    \caption{Comparison between \SNIS{} and \SUISIR{} for the same budget. In each boxplot the dotted line represents the \textbf{mean} value of the samples.}
    \label{fig:toy_perf}
\end{figure}
As can be seen from \Cref{fig:bias_performance}, \SUISIR\ significantly reduces bias (by a factor of almost 10) \wrt\ standard \SNIS\ for both configurations, while MSE increases only slightly (at around $20\%$), as can be seen in \Cref{fig:mse_performance}. The code used for this experiment is available at  \footnote{\href{https://github.com/gabrielvc/br_snis}{https://github.com/gabrielvc/br\_snis}}.
We also show in \Cref{subsec:app:toy_problem} that $\ki_{0} = \lfloor 0.625 \kmax \rfloor$ can lead to about $3$ times less bias \wrt\ standard \SNIS\ while only augmenting the MSE of $10\%$.
We have also compared in \label{subsec:app:toy_problem} \SUISIR\ to zero bias estimators based on \SNIS\ such as \cite{pmlr-v89-middleton19a}, the results are in shown in \Cref{subsec:app:toy_problem}.

%
\paragraph{Bayesian Logistic regression:}
\label{subsec:bayesian_logist}
We consider posterior inference in a Bayesian logistic regression model. Let $\dataset{train} = (\bX_i, y_i)_{i=1}^{T}$ be a dataset, where each $\bX_i \in \rset^d$ is a vector of covariates and $y_i \in \{-1,1\}$ is a binary response. Let $p(y_i \mid \bX_i;\theta) = \{1 + \exp(- y_i \, \bX_i^\top \theta)\}^{-1}$ be the probability of the $i$th observation at $\theta \in \Theta \subseteq \rset^d$ and $\target_0(\rmd \theta)$ be a prior distribution for $\theta$. The Bayesian posterior is given
\[{\textstyle
\target(\rmd \theta) ={\normconst}^{-1} \pi_0(\rmd \theta) \exp(\loglikelihood_{T}(\theta)), \quad
\loglikelihood_{T}(\theta) = \sum_{i=1}^{T} \ln p(y_i \mid \bX_i; \theta), \quad \normconst= \int \exp(\loglikelihood_{T}(\theta)) \target_0(\rmd \theta)\eqsp.}
\]
For numerical illustration, we use the heart failure clinical records ($d=13$, $T=299$), breast cancer detection ($d=30$, $T=569$), and Covertype ($d=55$, $T=4 \cdot 10^4$) datasets from the UCI machine learning repository.
For Covertype, we use Cover type 1 (Spruce/Fir) and Cover type 2 (Lodgepole Pine) classes to define a binary classification problem.
As a prior, we use a Gaussian distribution $\mathrm{N}(0, \tau^{-2} \Idd)$ with $\tau^2 = 5\cdot 10^{-2}$. The importance distribution $\proposal$ is Gaussian with mean and diagonal covariance learned by variational inference; see \Cref{sec:app:bayesian_log_reg} for details.
The boxplots for bias in \Cref{fig:bias_datasets} were constructed in the same way as those in \Cref{fig:toy_perf}.
\begin{table}[h]
\centering
\begin{tabular}{|c|c|c|c|}
  \hline
       .         & CoverType         & Breast   & Heart           \\  \hline
SNIS, M = 32    & 0.0028 +/- 0.0012 & 0.00011 +/- 6.04e-5 & 0.00023 +/- 7.24e-5  \\  \hline
\SUISIR, M= 32  & 0.0014 +/- 0.0003 & 7.9e-5 +/- 5.5e-5 & 0.00012 +/- 6.7e-5   \\  \hline
SNIS, M = 512   & 0.0026 +/- 0.0017 & 4.3e-5 +/- 3.3e-5 & 7.8e-5 +/- 6.8e-5  \\  \hline
\SUISIR, M= 512 & 0.0013 +/- 0.0003 &   3.5e-5 +/- 2.2e-5 & 4.9e-5 +/- 5.2e-5  \\
\hline
\end{tabular}
\caption{Comparison of the TV distance between the posteriors (Lower is better).}
\end{table}
We compare two test functions, $f(\theta) = \theta$, corresponding to evaluation of the posterior mean, and $f(\theta) = p(y \mid \bX, \theta)$, where $(\bX, y) \in\dataset{test}$. This last function allows us to compute a TV distance for the predictive distribution.
Indeed, in a classification context, one can compute the TV distance between any two predictive distributions $p$ and $\hat{p}$ as
\begin{equation}
{\textstyle\tvdist{\hat{p}}{p}= T^{-1}\sum_{i=1}^T \frac{1}{2}\sum_{j=0}^1 \vert \hat{p}(y=j\mid \bX_i, \dataset{train}) -p(y=j\mid \bX_i, \dataset{train}) \vert\eqsp,}
\end{equation} 
where we compare the predictive distribution $p(y\mid x, \dataset{train})= \int p(y\mid x, \theta) \target(\theta)\rmd \theta$ and $\hat{p}$ is the estimation of this quantity, provided in the experiments by \SNIS{} or \SUISIR.
\begin{figure}
    \begin{subfigure}{0.32\textwidth}
        \includegraphics[width=\textwidth]{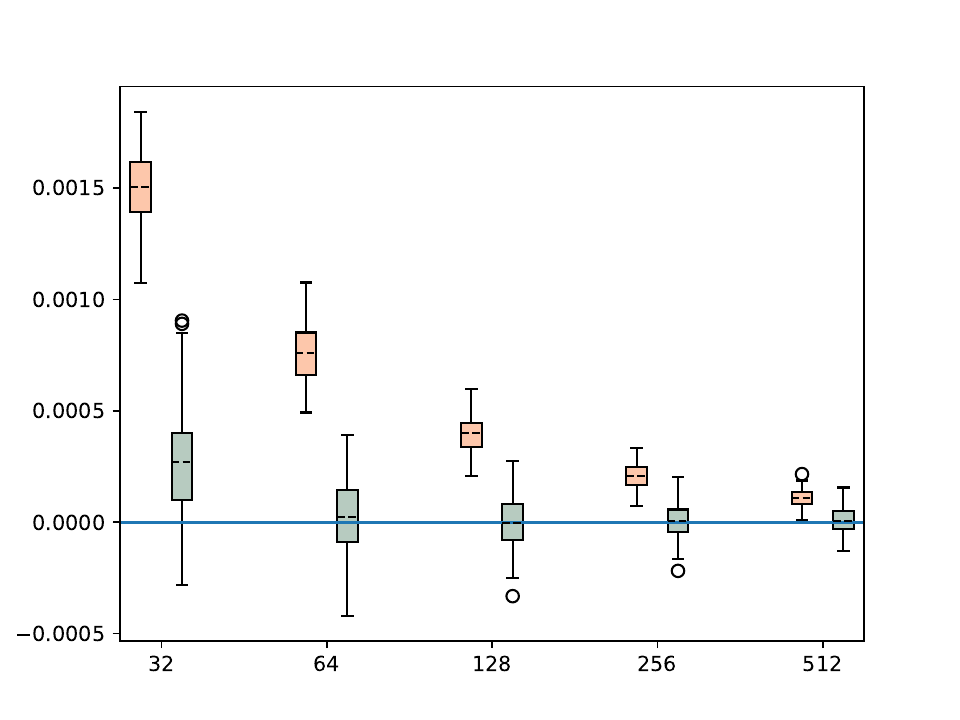}
        \caption{Heart 8}
        \label{fig:heart_component_8}
    \end{subfigure}
    \begin{subfigure}{0.32\textwidth}
        \includegraphics[width=\textwidth]{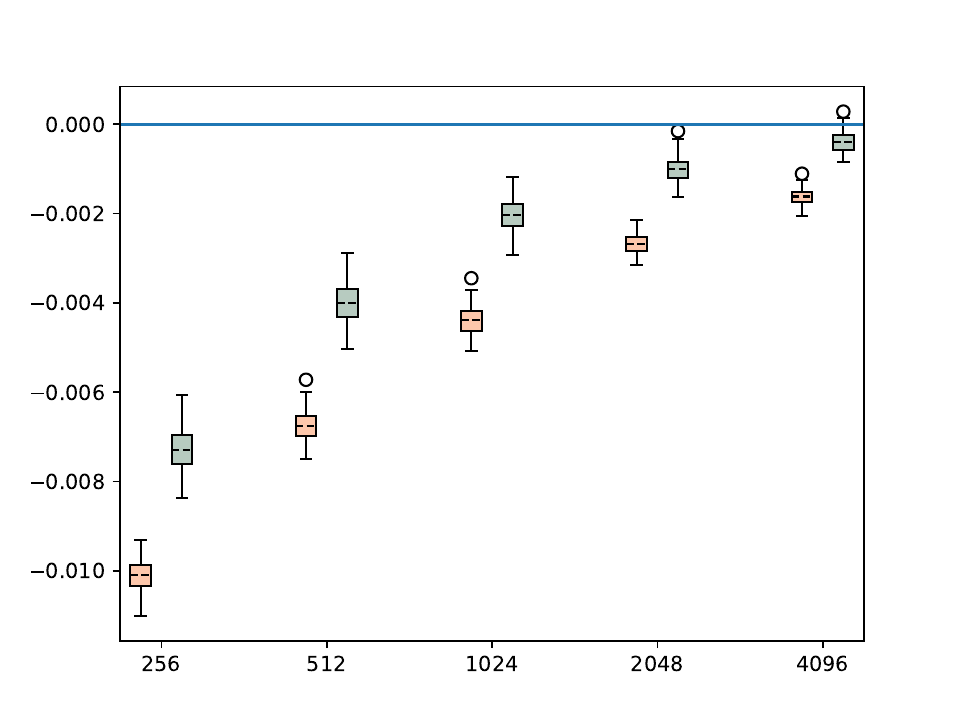}
        \caption{Breast 11}
        \label{fig:breast_component_11}
    \end{subfigure}
    \begin{subfigure}{0.32\textwidth}
        \includegraphics[width=\textwidth]{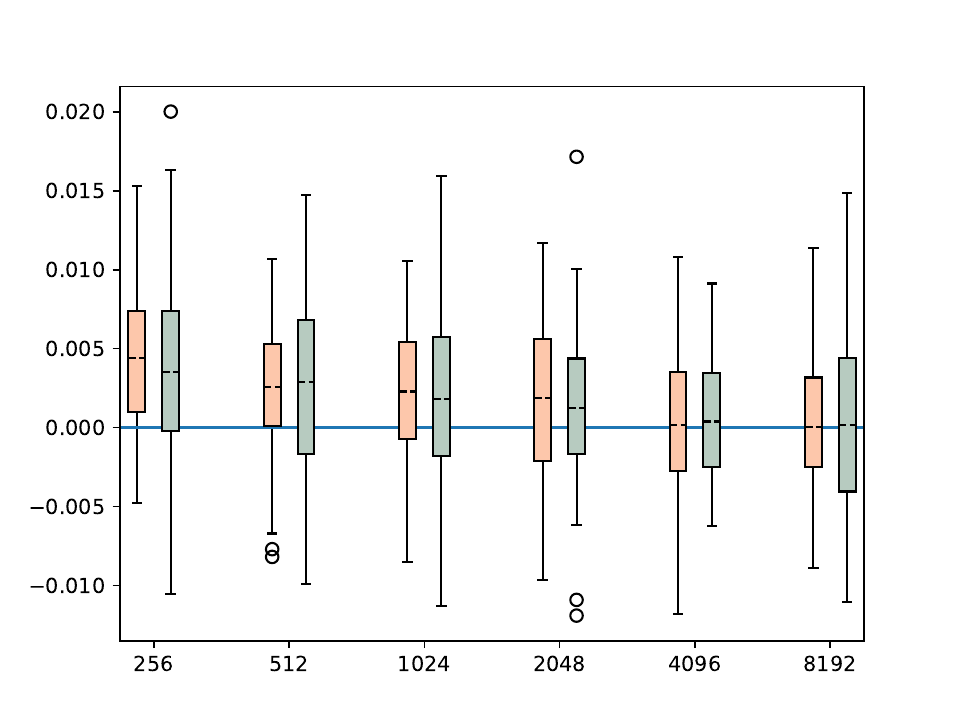}
        \caption{Covertype 6}
        \label{fig:covertype_component_6}
    \end{subfigure}
    \caption{Visualization of the distribution for each datasets. Each boxplot is grouped by budget, the left one represent \SNIS{} and the right represent \SUISIR{}.}
     \label{fig:bias_datasets}
\end{figure}
From \Cref{fig:bias_datasets} we can see that for each dataset we have a constant decrease in bias, while the variance increases only slightly. We plot the bias in other components of $\theta$ and provide further numerical details in \Cref{sec:app:bayesian_log_reg}.
\paragraph{Generative Model: }
\label{sec:vae}
We now extend our methodology to the more complex \emph{deep latent generative models} (DLGM). A DLGM defines a family of probability densities $p_\theta(x)$ over an observation space $x \in \rset^P$ by introducing a latent variable $z\in\rset^d$, defining the joint density function $p_\theta(x, z)$ (with respect to Lebesgue measure) and aiming to find a parameter $\theta$ maximizing the marginal log-likelihood of the model $p_\theta(x) = \int p_\theta(x,z)\rmd z$. Under simple technical assumptions, by Fisher's identity,  
\begin{equation}
\label{eq:fisher_vae}
{\textstyle\nabla_\theta \log p_\theta(x) = \int \nabla_\theta \log p_\theta(x, z) p_\theta(z\mid x) \rmd z,}
\end{equation}
In most cases, the conditional density $p_\theta(z \mid x) = p_\theta(x, z)/p_\theta(x)$ is intractable and can only be sampled. The variational autoencoder \cite{kingma2014stochastic} is based on the introduction of an additional parameter $\phi$ and a family of variational distributions $q_\phi(z \mid x)$. The joint parameters $\{\theta, \phi\}$ are then inferred by maximizing the \emph{evidence lower bound} (ELBO) defined by 
\begin{equation}
  \nonumber
    \elbo(\theta, \phi)
     =\log p_\theta(x) - \mbox{KL}\bigl(q_\phi(\cdot \mid x) ~\|~ p_\theta(\cdot \mid x)\bigr) \leq \log p_\theta(x).
  \end{equation}
This basic setup has been further developed and improved in many directions. Here we consider the \emph{importance weighted autoencoder} (IWAE) \cite{burda2015importance}, which relies on {\SNIS} to design a tighter ELBO on the log-likelihood. The objective of the IWAE is given by
 \begin{equation}
  \label{eq:iwae_elbo}
{\textstyle  \elboiw(\theta, \phi) =
    \int \log\left(M^{-1}\sum_{i=1}^M\weightfunc_{\theta, \phi, x}(z_i)\right) \prod_{\ell =1}^M q_\phi(z_\ell \mid x) \rmd z_i,
}
\end{equation}
where $\weightfunc_{\theta, \phi, x}(z) = {p_\theta(x, z)}/{q_\phi(z \mid x)}$ denote the importance weights. However, writing, following \citep[Eq.~(13)]{burda2015importance},
\[ {\textstyle\nabla_\theta \elboiw(\theta, \phi) = \int  \sum_{i=1}^M \omega^{(i)}_{\theta, \phi, x} \nabla_\theta \log\weightfunc_{\theta, \phi, x}(z_i) \prod_{\ell = 1}^N q_\phi(z_\ell \mid x) \rmd z_\ell,}\]
where $\omega^{(i)}_{\theta, \phi, x}= \weightfunc_{\theta, \phi, x}(z_i) / \sum_{j=1}^M\weightfunc_{\theta, \phi, x}(z_j) $ are normalized importance weights, yields an expression of the gradient that corresponds exactly to the biased {\SNIS} approximation of \eqref{eq:fisher_vae}. Thus, the optimization problem will suffer from bias. 
We hence propose to use {\SUISIR} for learning IWAE. The proposed algorithm proceeds in two steps, which are repeated during the optimization (details are given in \Cref{sec:IWAE})
\begin{itemize}[leftmargin=*,nosep]
\item First, update the parameter $\phi$ as in the IWAE algorithm (using the reparameterization trick and following the methodology of \cite{burda2015importance}) according to $\phi^{(t+1)} = \phi^{(t)} - \eta \nabla_\phi \elboiw(\theta^{(t)}, \phi^{(t)})$.
\item Second, update the parameter $\theta$ by estimating \eqref{eq:fisher_vae} using {\SUISIR} for $\target(z) = p_\theta(x, z)$, $f(z) = \nabla_\theta \log p_\theta(x, z)$ and $\proposal(z) = q_\phi(z \mid x)$.
\end{itemize}
\begin{table}[]
    \centering
\begin{tabular}{|c|c|c|c|c|}
  \hline
  d     & VAE    & IWAE   & BR-IWAE ($k=4$)  & BR-IWAE ($k=8$)\\  \hline
  20    & -115.1 & -91.5 & -90.5           &  -90.4\\ \hline
  50    & -96.4 & -92.9 & -90.1       & -90.1\\
\hline
\end{tabular}
\caption{Comparison of the mean log likelihood over the MNIST validation set (Higher is better).}
\label{table:comparison_log_lik}
\end{table}
We refer to this model as BR-IWAE. As an illustration, we train the model using the binarized MNIST dataset \cite{salakhutdinov2008quantitative}, where $x\in\{0,1\}^{784}$ are binarized digits images in dimension 784.
For both for the encoder $q_\phi$ and the decoder $p_\theta$, we use a simple fully connected architecture with 2 hidden layers (more details are given in \Cref{sec:IWAE}). For a comparison, we estimate the log-likelihood using the VAE, IWAE and BR-IWAE approaches, and the result is reported in  \Cref{table:comparison_log_lik}.
All models are run for 100 epochs, using the Adam optimizer \cite{kingma2014adam} and a learning rate of $10^{-3}$.
 The complete experimental details are given in \Cref{sec:IWAE}.

\section{Conclusion}
In this paper, we have introduced a novel method, {\SUISIR}, which improves over {\SNIS} when it comes to producing close to unbiased estimates of expectations taken \wrt\ to distributions known only up to a normalizing constant, a ubiquitous problem in machine learning and statistics.
The high performance of {\SUISIR} is supported theoretically by non-asymptotic bias, variance and high-probability bounds. We illustrate our method on various examples, which show the practical advantages of \SUISIR{} over \SNIS{}. 
Finally, \SUISIR\ is naturally adapted to other IS based methods, for example \cite{thin2021neo}, which use a Hamiltonian (gradient-based) transform \cite{neal2011mcmc} as part of the IS proposal. The extension of \SUISIR\ to \cite{thin2021neo} would produce an Hamiltonian based sampler able to recycle all samples, contrarily to other classical Hamiltonian-based methods \cite{neal2011mcmc, hoffman2014no}.
\SUISIR\ can also be extended to Particle Markov chain Monte Carlo methods such as Particle Gibbs with Ancestor sampling \cite{https://doi.org/10.48550/arxiv.1401.0604}.


\clearpage

\bibliography{cmot}

\clearpage
\clearpage
\appendix
\section{Proofs}
\subsection{\isir\ Algorithm}
\label{sec:isir-algorithm}
We analyze a slightly modified version of the {\isir} algorithm, with an extra randomization of the state position. 
The $k$-th iteration is defined as follows. Given a state $Y_k \in \Xset$, 
\begin{enumerate}[label=(\roman*),nosep,leftmargin=*]
\item draw $I_{k+1} \in \{1,\dots,N\}$ uniformly at random and set $X^{I_{k+1}}_{k+1}=Y_k$;
\item draw $\chunkum{X_{k+1}}{1}{N}{I_{k+1}}$ independently from the proposal distribution $\proposal$;  
\item compute, for $i \in \{1, \dots, N\}$, the normalized importance weights 
\[ \omega^i_{N,k+1} = \weightfunc(X^i_{k+1})/\sum_{\ell=1}^N \weightfunc(X^\ell_{k
+1});
\] 
\item select $Y_{k+1}$ from the set $\chunku{X_{k+1}}{1}{N}$ by choosing $X_{k+1}^i$ with probability $\omega^i_{N,k+1}$.
\end{enumerate}
Thus, compared to the simplified \isir\ algorithm given in the introduction, the state is inserted uniformly at random into the list of candidates instead of being inserted at the first position. Of course, this change has no impact as long as we are interested in integrating functions that are permutation invariant with respect to candidates, which is the case throughout our work. Still, this randomization makes the analysis much more transparent.

\subsection{Proof of \Cref{thm:Gibbs:duality}}
\label{sec:proof:thm:Gibbs:duality}
We write
\begin{align}
\eTarget(\rmd (y,\chunk{x}{1}{N}))
&= \frac{1}{N} \sum_{i=1}^N \target(\rmd y) \delta_y(\rmd x^i) \prod_{j\ne i} \proposal(\rmd x^j) \\
&= \frac{1}{N \proposal(\weightfunc)} \sum_{i=1}^N \weightfunc(x^i) \proposal(\rmd x^i) \delta_{x^i}(\rmd y) \prod_{j \ne i} \proposal(\rmd x^j) \\
&= \frac{1}{\proposal(\weightfunc)} \prod_{j=1}^N \proposal(\rmd x^j) \utargetkern \indi{\Xset}(\chunku{x}{1}{N}) \sum_{i=1}^N \frac{\weightfunc(x^i)}{\sum_{\ell=1}^N \weightfunc(x^\ell)} \delta_{x^i}(\rmd y),
\end{align}
where we recognize, and after having recalled definitions \eqref{eq:def:utargetkern} and \eqref{eq:marginal-joint} of $\eTargetmarginx$ and $\targetkern$, respectively, the right-hand side as $\eTargetmarginx(\rmd \chunku{x}{1}{N}) \targetkern( \chunku{x}{1}{N}, \rmd y)$. This completes the proof.

\subsection{Proof of \Cref{thm:unbiasedness}}
\label{sec:proof:unbiasedness}
Using \eqref{eq:marginal-joint} we get
\begin{align}
\int \eTargetmarginx(\rmd \chunku{x}{1}{N}) \targetkern f(\chunku{x}{1}{N})
&=  \int \frac{1}{N \proposal(\weightfunc)} \sum_{\ell=1}^N \weightfunc(x^\ell) \targetkern f(\chunku{x}{1}{N}) \prod_{j=1}^N \proposal(\rmd x^j) \\
&= \frac{1}{N \proposal(\weightfunc)} \int \sum_{i=1}^N \weightfunc(x^i) f(x^i) \prod_{j=1}^N \proposal(\rmd x^j)= \target(f),
\end{align}
and the proof is complete.

\subsection{Proof of \Cref{theo:main-properties-deterministic-scan}}
\begin{proof}
We first check that $\eTarget$ is an invariant distribution for $\MKisirjoint$. For every $A \in \Xsigma^{\varotimes (N + 1)}$, using that $\target$ is the marginal of $\eTarget$ with respect to the state and applying \Cref{thm:Gibbs:duality} yields
\begin{align}
\int \eTarget(\rmd (y, \chunku{x}{1}{N})) \MKisirjoint(y, \chunku{x}{1}{N}, A)
&= \int \target(\rmd y) \iint \partopopN(y,\rmd \chunku{\bar{x}}{1}{N}) \targetkern(\chunku{\bar{x}}{1}{N}, \rmd \bar{y}) \1_A(\bar{y}, \chunku{\bar{x}}{1}{N}) \\
&= \iiint \eTargetmarginx(\rmd \chunk{\bar{x}}{1}{N}) \targetkern(\chunku{\bar{x}}{1}{N}, \rmd y) \targetkern(\chunku{\bar{x}}{1}{N}, \rmd \bar{y}) \1_A(\bar{y}, \chunku{\bar{x}}{1}{N}) \\
&= \eTarget(A),
\end{align}
which establishes invariance.
We now show that $\MKisir$ is reversible with respect to $\target$. For this purpose, let $g$ and $h$ be two nonnegative measurable functions and write, using \Cref{thm:Gibbs:duality} twice,
\begin{align}
\iint \target(\rmd y) \MKisir(y, \rmd \bar{y}) g(y) h(\bar{y}) &= \int \target(\rmd y) \partopopN(y, \rmd \chunku{x}{1}{N}) \targetkern(\chunku{x}{1}{N}, \rmd \bar{y}) g(y) h(\bar{y}) \\
&= \int \eTargetmarginx(\rmd \chunku{x}{1}{N})  \targetkern(\chunku{x}{1}{N}, \rmd y)  \targetkern(\chunku{x}{1}{N}, \rmd \bar{y})  g(y) h(\bar{y}) \\
&= \int  \target(\rmd \bar{y}) \partopopN(\bar{y},\rmd \chunku{x}{1}{N}) \targetkern(\chunku{x}{1}{N}, \rmd y) g(y) h(\bar{y}) \\
&= \iint \target(\rmd \bar{y}) \MKisir(\bar{y}, \rmd y) g(y) h(\bar{y}).
\end{align}
\end{proof}

\subsection{Proof of \Cref{theo:isir_uniform_ergodicity}}
\label{supp:theo:isir_uniform_ergodicity}
  For completeness, we repeat the arguments in \cite{lindsten2015uniform,andrieu2018uniform}. Under \Cref{ass:boundedness:weights},
  we have, for $(x, \msa) \in \Xset \times \Xsigma$,
  \begin{align}
    \MKisir(x, \msa) &=\int \updelta_x(\rmd x^{1}) \sum_{i=1}^N \frac{\weightfunc(x^i)}{\sum_{j=1}^N \weightfunc(x^j)}\indi{\msa}(x^i) \prod_{j=2}^N\proposal(\rmd x^j) \\
    &=  \int \frac{\weightfunc(x)}{\weightfunc(x) + \sum_{j=2}^N \weightfunc(x^j)}\indi{\msa}(x) \prod_{j=2}^N\proposal(\rmd x^j) + \int \sum_{i=2}^N \frac{\weightfunc(x^i)}{\weightfunc(x) + \sum_{j=2}^N \weightfunc(x^j)}\indi{\msa}(x^i) \prod_{j=2}^N\proposal(\rmd x^j)\\
    &\geq \sum_{i=2}^N\int  \frac{\weightfunc(x^i)}{\weightfunc(x) +\weightfunc(x^i)+ \sum_{j=2, j\neq i}^N \weightfunc(x^j)}\indi{\msa}(x^i) \prod_{j=2}^N\proposal(\rmd x^j)\\
    &\geq \sum_{i=2}^N\int\target(\rmd x^i)\indi{\msa}(x^i)\int \frac{\proposal(\weightfunc)}{\weightfunc(x) +\weightfunc(x^i) +\sum_{j=2, j\neq i}^N \weightfunc(x^j)} \prod_{j=2, j\neq i}^N\proposal(\rmd x^j).
  \end{align}
  Finally, since the function $f\colon z\mapsto (z+a)^{-1}$ is convex on $\rset_+$ and $a>0$, we get for $i\in\{2,\dots, N\}$,
  \begin{align}
  \label{eq:lower-bound}
    &\int \frac{\proposal(\weightfunc)}{\weightfunc(x) + \weightfunc(x^i) +\sum_{j=2, j\neq i}^N \weightfunc(x^j)} \prod_{j=2, j\neq i}^N\proposal(\rmd x^j) \\
    &\quad \geq \frac{\proposal(\weightfunc)}{\int \weightfunc(x) + \weightfunc(x^i) +\sum_{j=2, j\neq i}^N \weightfunc(x^j)\prod_{j=2, j\neq i}^N\proposal(\rmd x^j)}  \\
    &\quad \geq \frac{1}{\weightfunc(x) / \proposal(\weightfunc) + \weightfunc(x^i) / \proposal(\weightfunc) + N-2}\geq \frac{1}{2\bound + N-2}.
  \end{align}
  We finally obtain the  inequality
  \begin{equation}
  \label{eq:minorise_condition}
    \MKisir(x, \msa)\geq \target(\msa)\times \frac{N-1}{2\bound + N -2} =  \epssmallisir \target(\msa).
  \end{equation}
  This means that the whole space $\Xset$ is $(1,\epssmallisir \target)$-small (see \cite[Definition~9.3.5]{douc:moulines:priouret:2018}). Since $\MKisir(x, \cdot)$ and $\target$ are probability measures, \eqref{eq:minorise_condition} implies
  \begin{equation}
  \tvnorm{\MKisir(x, \cdot) - \target} = \sup_{\msa \in \Xsigma} |\MKisir(x, \msa) - \target(\msa)| \leq 1-\epssmallisir = \driftconstisir.
  \end{equation}
  Now the statement follows from \cite[Theorem~18.2.4]{douc:moulines:priouret:2018} applied with $m=1$.

\subsection{Proof of \Cref{theo:bias-i-SIR-recycling}}
\begin{proof}[Proof of \eqref{item:theo:bias-i-SIR-recycling:mse}]
Considering the identity $(a+b)^2 \leq (1 + \epsilon^2) a^2 + (1 + \epsilon^{-2})b^2 $, we obtain the decomposition $\{ \targetkern f(\chunku{X_k}{1}{N}) - \target(f)\}^2 \leq (1 + (N-1)^{-1/2}) I + (1 + (N-1)^{1/2}) II$, with
\begin{align}
I &:= \{\targetkern f(\chunku{X_k}{1}{N}) - a_N(Y_{k-1})/b_N(Y_{k-1})\}^2 \\
II &:= \{ a_N(Y_{k-1})/b_N(Y_{k-1}) - \target(f) \}^2 \eqsp.
\end{align}
By using the identity $a/b - c/d = (1/d)[(a/b)(d-b) - (c-a)]$, we obtain 
\begin{multline}
\targetkern f(\chunku{X_k}{1}{N}) - a_N(Y_{k-1})/b_N(Y_{k-1}) = b_N(Y_{k-1})^{-1}[\targetkern f(\chunku{X_k}{1}{N})(b_N(Y_{k-1}) - \utargetkern \indi{\Xset}(\chunku{X_k}{1}{N})) \\
- (a_N(Y_{k-1}) - \utargetkern f(\chunku{X_k}{1}{N}))] \eqsp.
\end{multline}
Therefore, using $(a+b)^2 \leq 2(a^2 + b^2)$ we get
\begin{equation}
    I \leq \frac{2}{b_N(Y_{k-1})^2}[\targetkern f(\chunku{X_k}{1}{N})^2\{\utargetkern \indi{\Xset}(\chunku{X_k}{1}{N}) - b_N(Y_{k-1})\}^2 + \{\utargetkern f(\chunku{X_k}{1}{N})- a_N(Y_{k-1})\}^2].
\end{equation}
Using the fact that $\targetkern f(\chunku{X_k}{1}{N})^2 \leq 1$ $\PP_\xijoint$-a.s.\ and $b_N(y) \geq (N-1)/N \proposal(\weightfunc)$, we get, $\PP_\xijoint$-\as, 
\begin{equation}
    I \leq \frac{2N^2}{(N-1)^2\proposal(w)^2} \left[\{\utargetkern \indi{\Xset}(\chunku{X_k}{1}{N}) - b_N(Y_{k-1})\}^2 + \{\utargetkern f(\chunku{X_k}{1}{N})- a_N(Y_{k-1})\}^2\right]
\end{equation}
Therefore, using Lemma \eqref{lem:key-relation}, 
\begin{align}
    \lefteqn{\PE_{\xijoint}[\{\targetkern f(\chunku{X_k}{1}{N}) - a_N(Y_{k-1})/b_N(Y_{k-1})\}^2]} \\
    &= \PE_{\xijoint}[\CPE[\xijoint]{\{\targetkern f(\chunku{X_k}{1}{N}) - a_N(Y_{k-1})/b_N(Y_{k-1})\}^2}{Y_{k-1}}] \\
    &\leq  \frac{2N^2}{(N-1)^2\proposal(w)^2} \left[(N-1)/N^2 \proposal( \{ \weightfunc  - \proposal(\weightfunc ) \}^2) + (N-1)/N^2 \proposal( \{ \weightfunc f - \proposal(\weightfunc f) \}^2)\right] \\
    &\leq 4 (N-1)^{-1}  \kappa[\target,\proposal] \eqsp.
\end{align}
We consider now $II$. From \eqref{eq:bound-2} we have that $II \leq  4 N^{-2} (1 + \bound)^2$, which 
completes the proof.
\end{proof}
\begin{proof}[Proof of \eqref{item:theo:bias-i-SIR-recycling:cov}]
Note that
\begin{align}
III &:= \PE_{\xijoint}[ \{ \targetkern f(\chunku{X_k}{1}{N}) - \target(f) \}
\{  \targetkern f(\chunku{X_{k+\ell}}{1}{N}) - \target(f) \} ]
\\
&= \PE_{\xijoint} [  \{ \targetkern f(\chunku{X_k}{1}{N}) - \target(f) \} \CPE[\xijoint]{ \targetkern f(\chunku{X_{k+\ell}}{1}{N}) - \target(f)}{Y_{k+\ell-1}}] \eqsp.
\end{align}
As $\CPE[\xijoint]{\targetkern f(\chunku{X_{k+\ell}}{1}{N})}{Y_{k+\ell-1}} = \PhiN{Y_{k+\ell-1}}$
$\PP_\xijoint$-\as, we have that
\begin{align}
III&=\PE_{\xijoint} [\{ \targetkern f(\chunku{X_k}{1}{N}) - \target(f) \} \{\PhiN{Y_{k+\ell-1}} - \target(f)\}] \\
&= \PE_{\xijoint} [\{ \targetkern f(\chunku{X_k}{1}{N}) - \target(f) \} \{\CPE[\xijoint]{\PhiN{Y_{k+\ell-1}}}{Y_{k}} - \target(f)\}] \eqsp.
\end{align}
By the Markov property, we get that
\begin{equation}
\CPE[\xijoint]{\PhiN{Y_{k+\ell-1}}}{Y_{k}} = \MKisir^{\ell-1}\PhiN{Y_{k}} = \delta_{Y_{k}}\MKisir^{\ell-1}\Phi_N \eqsp, \quad \PP_\xijoint\mbox{-a.s.},
\end{equation}
which, combined with \eqref{eq:bound_phi_n}, implies that 
\begin{equation}
    \supnorm{\MKisir^{\ell-1}\Phi_N - \target(f)} \leq \biasconst  (N-1)^{-1} \driftconstisir^{\ell-1}.
\end{equation}
Combining the results above, we finally establish that
\begin{align}
|III|
& \leq \biasconst  (N-1)^{-1} \driftconstisir^{\ell-1}  \PE_{\xijoint} [\{ \targetkern f(\chunku{X_k}{1}{N}) - \target(f) \}^2]^{1/2}  \\
&\leq \biasconst  (N-1)^{-1} \driftconstisir^{\ell-1} \bigl\{\sum\nolimits_{i=0}^{2} \mseconst_{i} (N - 1)^{-1 - i/2}\bigr\}^{1/2}  \eqsp.
\end{align}
\end{proof}
\subsection{Proof of \Cref{theo:bias-mse-rolling}}
\label{sec:proof:theo:bias-mse-rolling}
We first consider the bias term. The bound is given by 
\begin{align}
    \left| \PE_{\xijoint}[\rollingestim[K_0][K][N][f]] - \target(f) \right| \leq  (K - K_0)^{-1}  \sum_{\ell=K_0 + 1}^{K} \left| \PE_{\xijoint}[\targetkern f(\chunku{X_{\ell}}{1}{N})] - \target(f) \right| \\
    \leq (K - K_0)^{-1}(N-1)^{-1} \biasconst  \sum\nolimits_{\ell=K_0 + 1}^{K} \driftconstisir^{\ell-1} \eqsp.
\end{align}
The proof is concluded by noting that
\begin{equation}
    \sum\nolimits_{\ell=K_0 + 1}^{K} \driftconstisir^{l-1} \leq \frac{\driftconstisir^{K_0}}{1 - \driftconstisir} \leq \frac{4 \TmixN (1/4)^{K_0/\TmixN}}{3} \eqsp.
\end{equation}

We now consider the MSE, using the decomposition
\begin{align}
    \PE_{\xijoint}[(\rollingestim[K_0][K][N][f] - \target(f))^2] \leq (K-K_0)^{-2} \left\{ \sum\nolimits_{\ell=K_0 + 1}^{K} \PE_{\xijoint}[\targetkern f(\chunku{X_{\ell}}{1}{N})] - \target(f) \right\}^2 \\
    + 2 \sum\nolimits_{\ell=K_0 + 1}^{K} \sum\nolimits_{j= \ell + 1}^{K} \PE_{\xijoint}[\{\targetkern f(\chunku{X_{\ell}}{1}{N}) - \target(f)\} \{\targetkern f(\chunku{X_j}{1}{N}) - \target(f)\}] \eqsp.
\end{align}
From the MSE bound of \Cref{theo:bias-estimator-general}, we have that
\begin{equation}
    \sum\nolimits_{\ell=K_0 + 1}^{K} \PE_{\xijoint}[\{\targetkern f(\chunku{X_{\ell}}{1}{N}) - \target(f)\}^2] \leq (K - K_0) (N - 1)^{-1}\sum\nolimits_{i=0}^{2} \mseconst_{i} (N - 1)^{-i/2} \eqsp.
\end{equation}
From the covariance bound of \Cref{theo:bias-estimator-general}, we have that
\begin{multline}
    \sum_{\ell=K_0 + 1}^{K} \sum_{j= \ell + 1}^{K} \PE_{\xijoint}[\{\targetkern f(\chunku{X_{\ell}}{1}{N}) - \target(f)\} \{\targetkern f(\chunku{X_j}{1}{N}) - \target(f)\}] \\
    \leq \sum_{i=0}^{2} \covconst_{i} (N - 1)^{-\frac{3 - i/2}{2}} \left[\sum_{\ell=K_0 + 1}^{K} \sum_{j= \ell + 1}^{K} \driftconstisir^{(j - \ell) - 1} \right]  \eqsp.
\end{multline}
As  $\sum\nolimits_{\ell=K_0 + 1}^{K} \sum\nolimits_{j= \ell + 1}^{K} \driftconstisir^{(j - \ell) - 1}  \leq (K - K_0) (4/3) \TmixN$ we have
\begin{align}
  \PE_{\xijoint}[(\rollingestim[K_0][K][N][f] - \target(f))^2] \leq ((K-K_0)(N-1))^{-1}\left[\sum\nolimits_{i=0}^{2} \mseconst_{i} (N - 1)^{-i/2}\right] \\
  + (8/3)(K-K_0)^{-1}(N-1)^{-3/2}\left[\sum_{i=0}^{2} \covconst_{i} (N - 1)^{- i/4}\right] \eqsp.
\end{align}
The proof for the MSE is concluded by noting that $(K-K_0) (N-1)= \burningloss M$ and noting that $\msesnis:= \mseconst_0 M^{-1}$.

The high-probability bound requires more complex derivations. We use the  decomposition
\begin{multline}
\rollingestim[K_0][K][N][f] - \target(f) = (K-K_0)^{-1} \sum_{k=K_0+1}^K \targetkern f(\chunku{X_k}{1}{N}) - \PhiN{Y_{k-1}} \\
+ (K-K_0)^{-1} \sum_{k=K_0+1}^{K-1} \PhiN{Y_{k-1}} - \target(\Phi_N) \eqsp.
\end{multline}
where we have used $\target(f)= \target(\Phi_N)$. Therefore, for any $t \geq 0$, we get that
\begin{multline}
\PP_{\xijoint}( |\rollingestim[K_0][K][N][f] - \target(f)| \geq t) \leq
\PP_{\xijoint}\left((K-K_0)^{-1} \left| \sum\nolimits_{k=K_0+1}^K \targetkern f(\chunku{X_k}{1}{N}) - \PhiN{Y_{k-1}} \right| \geq t/2\right) \\
+ \PP_{\xijoint} \left( (K-K_0)^{-1} \left| \sum\nolimits_{k=K_0+1}^{K-1} \PhiN{Y_{k-1}} - \target(\Phi_N) \right| \geq t/2 \right) \eqsp.
\end{multline}
We will show that for all $t > 0$, and for some absolute constants $\zeta_I$, $\zeta_{II}$,
\begin{align}
&I := \PP_{\xijoint}\left((K-K_0)^{-1} \left| \sum\nolimits_{k=K_0+1}^K \targetkern f(\chunku{X_k}{1}{N}) - \PhiN{Y_{k-1}} \right| \geq t\right) \leq 2 \exp( - t^2 \burningloss M/ (4\zeta_I)) \eqsp, \\
&II := \PP_{\xijoint} \left( (K-K_0)^{-1} \left| \sum\nolimits_{k=K_0+1}^{K-1} \PhiN{Y_{k-1}} - \target(\Phi_N) \right| \geq t  \right) \leq 2 \exp( - t^2 \zeta_{II} (K-K_0) (N-1)^2/ \TmixN) \eqsp.
\end{align}
Consider first $I$. Note first that
\begin{equation}
I= \PE_{\xijoint}\left[\CPP[\xijoint]{(K-K_0)^{-1} \left| \sum\nolimits_{k=K_0+1}^K \targetkern f(\chunku{X_k}{1}{N}) - \PhiN{Y_{k-1}} \right| \geq t}{\chunk{Y}{K_0}{K-1}} \right] \eqsp.
\end{equation}
By \Cref{theo:main-properties-deterministic-scan}, the random elements $\{ \chunku{X_k}{1}{N} \}_{k=K_0+1}^K$ are independent
conditionally to $\{ Y_k \}_{k=K_0}^{K-1}$.
Using the generalized Hoeffding inequality (see \cite[Theorem~2.6.2]{vershynin2018high} or \cite[Proposition~2.1]{wainwright2019high}), we get, with $\Delta_{N,k} :=  \targetkern f(\chunku{X_k}{1}{N}) - \PhiN{Y_{k-1}}$, that
\begin{equation}
\CPP[\xijoint]{(K-K_0)^{-1} \left| \sum\nolimits_{k=K_0+1}^K  \Delta_{N,k} \right| \geq t}{\chunk{Y}{K_0}{K-1}} \leq 2 \exp \left( - \frac{ t^2 (K-K_0)^2}{4\sum_{k=K_0+1}^K \| \Delta_{N,k} \|_{\psi_2,Y_k}^2} \right),
\end{equation}
where $\psi_2(x) = \exp(x^2)-1$ and
\[
\| \Delta_{N,k} \|_{\psi_2,Y_{k-1}} =  \inf\left\{ \lambda > 0 ~:~ \CPE[\xijoint]{\psi_2(|\Delta_{N,k}|/\lambda)}{Y_{k-1}} \leq 1 \right\}.
\]
We have to bound $\|\Delta_{N, k}\|_{\psi_2,Y_{k-1}}$. We use the decomposition
\begin{equation}
    \Delta_{N,k} = \left\{\frac{\utargetkern f(\chunku{X_k}{1}{N})}{\utargetkern \indi{\Xset}(\chunku{X_k}{1}{N})} - \frac{a_N(Y_{k-1})}{b_N(Y_{k-1})}\right\} + \left\{\frac{a_N(Y_{k-1})}{b_N(Y_{k-1})} - \PhiN{Y_{k-1}}\right\} =: \Delta_{N,k}^{I} + \Delta_{N, k}^{II} \eqsp.
\end{equation}
combined with  \Cref{lem:main_bound} with $\phi= \psi_2$ and $\chi= \psi_2$ and \cite[Proposition~2.6.1]{vershynin2018high}.
By \eqref{eq:bound-1} and by \cite[Equation~2.17]{vershynin2018high} we have that
\begin{equation}
    \|\Delta_{N, k}^{II}\|_{\psi_2,Y_{k-1}} \leq 2 \log(2)^{-1/2} (N-1)^{-1} \kappa[\proposal, \target] \eqsp.
\end{equation}
Using \Cref{lem:main_bound} with $\phi=\chi=\psi_2$ and the fact that $b_N(y) \geq (1 - 1/N)\proposal(w)$  we have that
\begin{multline}
    \|\Delta_{N, k}^{I}\|_{\psi_2,Y_{k-1}} \leq \frac{2}{(1 - 1/N)\proposal(w)}\\ \times
    \left[\|\utargetkern f(\chunku{X_k}{1}{N}) - a_N(Y_{k-1})\|_{\psi_2,Y_{k-1}} + 2\|\utargetkern \indi{\Xset}(\chunku{X_k}{1}{N}) - b_N(Y_{k-1})\|_{\psi_2,Y_{k-1}}\right] \eqsp.
\end{multline}
Using \cite[Proposition~2.6.1, Eq~2.17]{vershynin2018high}, we get that, $\PP_{\xijoint}$-\as,
\begin{align}
\| \utargetkern f(\chunku{X_{k-1}}{1}{N}) - a_N(Y_{k-1}) \|^2_{\psi_2,Y_{k-1}}
&\leq  64\rme/\log{2} N^{-1} \| \weightfunc(X_k^1) f(X_k^{1}) - \CPE[\xijoint]{\weightfunc(X_k^1) f(X_k^{1})}{Y_{k-1}} \|^2_{\psi_2,Y_{k-1}}   \eqsp, \\
&\leq 256\rme/\log{(2)}^2 N^{-1} \supnorm{\weightfunc}^2.
\end{align}
The same bound applies to $\| \utargetkern \indi{\Xset}(\chunku{X_k}{1}{N}) - b_N(Y_{k-1}) \|^2_{\psi_2,Y_{k-1}}$ and we can write
\begin{equation}
    \|\Delta_{N, k}^{I}\|_{\psi_2,Y_{k-1}} \leq 96 e^{1/2}\log(2)^{-1}(N-1)^{-1/2} \bound \eqsp.
\end{equation}
We can now conclude by writing
\begin{align}
 \|\Delta_{N, k}\|_{\psi_2,Y_{k-1}}^2 & \leq 2 (\|\Delta_{N, k}^I\|_{\psi_2,Y_{k-1}}^2 + \|\Delta_{N, k}^{II}\|_{\psi_2,Y_{k-1}}^2)\\
 & \leq (N-1)^{-1}(\zeta_{I, 1} \bound^2  + \zeta_{I, 2} \kappa[\lambda, \pi]^2(N-1)^{-1}) \eqsp,
\end{align}
where $\zeta_{I, 1} = 18432 \rme\log(2)^{-2}$ and $\zeta_{I, 2} = 8\log(2)^{-1}$ are absolute constants, which imply
\begin{equation}
\|\Delta_{N, k}\|_{\psi_2,Y_{k-1}}^2 \leq \zeta_{I} (N-1)^{-1}\eqsp,
\end{equation}
with $\zeta_{I}= 1.1\cdot 10^5 \bound ^2$. This finally concludes $I \leq 2 \exp(-t^2 \burningloss M / 4 \zeta_{I})$.

Consider now $II$. We use \Cref{lem:bounded_differences_norms_markovian} with $g_i=\PhiN{Y_{K_0 + i -1}} - \target(\Phi_N)$.
As $\supnorm{g_i} \leq \osc(\Phi_N) \leq (N-1)^{-1}\biasconst$, we obtain
\begin{equation}
    II \leq 2 \exp\left(-t^2 \zeta_{II}(K - K_0)(N-1)^2 / \TmixN \right)
\end{equation}
where $\zeta_{II} = \frac{2}{(3\biasconst)^2}$.
Finally, we obtain
\begin{align}
    \PP_{\xijoint}( |\rollingestim[K_0][K][N][f] - \target(f)| \geq t) \leq 2 \exp(-t^2 \burningloss M / 4 \zeta_{I}) [1 + \exp(-t^2 \burningloss M \{\zeta_{II}(N-1) / \TmixN - (4 \zeta_{I})^{-1}\})].
\end{align}
We conclude by noting that, for any $\delta \in (0, 1)$, if $(N-1) \geq \TmixN (4 \zeta_{I} \zeta_{II})^{-1}$, we have that
\begin{align}
    \PP_{\xijoint}(|\rollingestim[K_0][K][N][f] - \target(f)| \geq t) \leq 4 \exp(-t^2\burningloss M / 4 \zeta_{I}) \leq \delta
\end{align}
for all $t \geq 2 \zeta_I^{1/2} (\burningloss M)^{-1/2} \log(4 / \delta)^{1/2}$. This concludes the proof with $\hpdconstant = 2 \zeta_I^{1/2}$.

\subsection{High probability inequality for SNIS}
\label{sec:hpd-SNIS}

\begin{theorem}
    Assume that $\bound = \supnorm{\weightfunc} / \proposal(\weightfunc) < \infty$. For any bounded measurable function $f$ on $(\Xset,\Xsigma)$ such that $\supnorm{f} \leq 1$, for any $M > 1$
and for any $\delta$ in $(0, 1)$,
\begin{equation}
    |\estsnis{f} - \target(f) | \leq 12 \bound [M\log(2)]^{-1/2}  \log(2/\delta)^{1/2}
\end{equation}
with probability larger than $1 - \delta$.
\end{theorem}
\begin{proof}
    Let $A_M = M^{-1}\sum\nolimits_{i=1}^{M}\weightfunc(X^i)f(X^i)$, $B_M = M^{-1}\sum\nolimits_{i=1}^{M}\weightfunc(X^i)$, $a=\PE(A_M) = \proposal(\weightfunc f)$ and $b=\PE(B_M) = \proposal(\weightfunc)$.
    Note that $\estsnis{f} = A_M / B_M$ and $\target(f) = a / b$. By \Cref{lem:main_bound} with $\phi = \chi = \exp(x^2) - 1$, we have that:
    \begin{equation}
        \|\estsnis{f} - \target(f)\|_{\psi_2} \leq 2 \proposal(\weightfunc)^{-1}[\|A_M - a\|_{\psi_2} + 2 \|B_M - b\|_{\psi_2}] \eqsp .
    \end{equation}
    Using \cite[Eq~2.17]{vershynin2018high}, we get that, $\PP_{\xijoint}$-\as
    \begin{equation}
        \|A_M - a\|_{\psi_2}^2 \leq M^{-1} \|\weightfunc(X^i) f(X^i) - \proposal(\weightfunc f)\|_{\psi_2}^2 \leq 4(\log(2)M)^{-1} \supnorm{\weightfunc}^2 \eqsp .
    \end{equation}
    In the same way, $\|B_M - b\|_{\psi_2}^2 \leq 4(\log(2)M)^{-1} \supnorm{\weightfunc}^2$. Therefore, we have that:
    \begin{equation}
        \|\estsnis{f} - \target(f)\|_{\psi_2}^2 \leq (12 \bound)^2 (\log(2)M)^{-1} \eqsp.
    \end{equation}
    Combining it with \cite[Proposition~2.5.2]{vershynin2018high}, we have that:
    \begin{equation}
        \PP(|\estsnis{f} - \target(f) | \geq t) \leq 2 \exp(-t^2\zeta^{\operatorname{snis}}M)
    \end{equation}
    where $\zeta^{\operatorname{snis}} = \log(2)(12 \bound)^{-2}$.
    The high probability inequality follows directly.
\end{proof}

\section{Moments and high-probability bounds for ratio statistics}
\label{sec:ratio-statistics}
Let $(U_i,V_i)_{i\in\{1,\dots,n\}}$ be (possibly dependent) random variables defined on some probability space $(\Omega,\mcf,\PP)$. Assume that $U_i\geq 0$, $\PP$-\as. Moreover, let $\hat{A}_n = n^{-1} \sum_{i=1}^n U_i V_i$, $\hat{B}_n = n^{-1} \sum_{i=1}^n U_i$, $\hat{R}_n = \hat{A}_n / \hat{B}_n$ and $a = \Exp{\hat{A}_n}$, $b = \Exp{\hat{B}_n}$, and $r = a/b$.

A continuous, even, convex function $\phi:\rset^+ \to \ccint{0,+\infty}$ is a Young function if $\phi$ is monotonically increasing for $x > 0$, $\phi(0)=0$,  $\lim_{x \to \infty} \phi(x)/x= \infty$, and $\lim_{x \to 0^+} \phi(x)/x= 0$. We denote by $\phi^*$ the Fenchel-Legendre conjugate of $\phi$.
Let $X$ be a random variable and $\phi$ be a Young function. Then the \emph{Orlicz norm} of $X$ is
\begin{equation}
\|X\|_\phi = \inf\left\{ \lambda > 0 ~:~ \Exp{\phi\left(|X|/\lambda\right)} \leq 1 \right\}\,,
\end{equation}
with the convention that $\inf \emptyset= \infty$. The Orlicz space of random variable $\mathcal{L}_{\phi}(\Omega)$
is the family of equivalence classes of random variables $X$ such that $\| X \|_\phi < \infty$. $\mathcal{L}_{\phi}(\Omega)$ is a Banach space. If $\phi_p(x)= |x|^p$ for $p \geq 1$, then $\mathcal{L}_{\phi}(\Omega)= \mathcal{L}^p(\Omega)$ and we denote $\| \cdot \|_p= \| \cdot \|_{\phi_p}$. If $X \in \mathcal{L}_\phi(\Omega)$, then, for any $x > 0$
\[
\PP(|X|\geq x) \leq 1/\phi(x/\|X\|_{\phi}) \eqsp, \quad \text{and} \quad \| \indiacc{|X| \geq x} \|_\phi = 1 / \phi^{-1}(1 / \PP(|X| \geq x)) \eqsp.
\]

\begin{lemma}\label{lem:main_bound}
Let $\phi,\chi$ be  Young functions.  If $\max_i\|V_i\|_\infty\leq c|r|$, then
\begin{equation}
\|\hat{R}_n - r\|_\phi/|r| \leq 2\|\hat{A}_n - a\|_\phi/b + 2\|\hat{B}_n - b\|_\phi/b + c/\{(\phi^{-1}\circ\chi)(b / 2\|(\hat{B}_n - b)_-\|_\chi)\}\,.
\end{equation}
\end{lemma}
\begin{proof}
We decompose the computation in two parts: first, when $\hat{B}_n > b/2$, we have
\begin{align}
|\hat{R}_n - r| &= \left|\frac{\hat{A}_n - a}{\hat{B}_n} + a\left(\frac{1}{\hat{B}_n} - \frac{1}{b}\right)\right| \leq \frac{|\hat{A}_n - a|}{\hat{B}_n} + \frac{|a||\hat{B}_n - b|}{\hat{B}_n b}\\
&\leq \frac{|\hat{A}_n - a|}{b/2} + \frac{|a||\hat{B}_n - b|}{(b/2)b}= \frac{2|\hat{A}_n - a|}{b} + \frac{2|r||\hat{B}_n - b|}{b}\,.
\end{align}
Then, when $\hat{B}_n \leq b/2$, we have
\begin{align}
|\hat{R}_n - r| &\leq |\hat{R}_n| + |r| \leq |\hat{R}_n| + \frac{2|r||\hat{B}_n - b|}{b} \leq \max_i |V_i| + \frac{2|r||\hat{B}_n - b|}{b}\,,
\end{align}
where the second inequality follows from $|\hat{B}_n - b| \geq b/2$. Combining the two previous inequalities, we have
\begin{equation}
|\hat{R}_n - r| \leq \frac{2|\hat{A}_n - a|}{b} + \frac{2|r||\hat{B}_n - b|}{b} + \max_i |V_i| \indiacc{\hat{B}_n \leq b/2}\,.
\end{equation}
Recall that, if $|X|\leq |Y|$ a.s., then $\|X\|_\phi \leq \|Y\|_\phi$. Hence,  we get that
\begin{align}
\|\hat{R}_n - r\|_\phi &\leq \left\|\frac{2|\hat{A}_n - a|}{b} + \frac{2r|\hat{B}_n - b|}{b} + \max_i |V_i| \indiacc{\hat{B}_n \leq b/2}\right\|_\phi\\
&\leq \left\|\frac{2\hat{A}_n - a}{b}\right\|_\phi + \left\|\frac{2|r||\hat{B}_n - b|}{b}\right\|_\phi + \left\|\max_i |V_i| \indiacc{\hat{B}_n \leq b/2}\right\|_\phi\\
&= \frac{2\|\hat{A}_n - a\|_\phi}{b} + \frac{2|r|\|\hat{B}_n - b\|_\phi}{b} + \|\max_i |V_i| \indiacc{\hat{B}_n \leq b/2}\|_\phi\\
&\leq \frac{2\|\hat{A}_n - a\|_\phi}{b} + \frac{2|r|\|\hat{B}_n - b\|_\phi}{b} + c |r|  \|\indiacc{\hat{B}_n \leq b/2}\|_{\phi}\\
&= \frac{2\|\hat{A}_n - a\|_\phi}{b} + \frac{2|r|\|\hat{B}_n - b\|_\phi}{b} + c |r| /\phi^{-1}\left(1/\PP(\hat{B}_n \leq b/2)\right)\,.
\end{align}
Finally, we obtain the desired result by noting that, for any Young function $\chi$, $\PP(\hat{B}_n \leq b/2) = \PP(|(\hat{B}_n - b)_-|\geq b/2) \leq 1/{\chi(b/2\|(\hat{B}_n - b)_-\|_\chi)}$.
\end{proof}

\begin{theorem}
\label{theo:bound-moment}
Let $p\geq 1$. If $\max_i\|V_i\|_\infty\leq c|r|$, then
\begin{equation}
\frac{\|\hat{R}_n - r\|_p}{|r|} \leq \frac{2\|\hat{A}_n - a\|_p}{b} + \frac{2( 1 + c) \|\hat{B}_n - b\|_p}{b}\,.
\end{equation}
\end{theorem}
\begin{proof}
Apply \Cref{lem:main_bound} with $\chi(x)= \phi(x) = x^p$.
\end{proof}

\begin{theorem}[Bias of the estimator]
\label{theo:bias-estimator-general}
If $|\hat{A}_n / \hat{B}_n | \leq 1$, then
\begin{equation}
\left| \PE[\hat{R}_n] - r \right| \leq  (2 b^2)^{-1}  \{ 3 \PE[\{ \hat{B}_n - b \}^2] + \PE[ \{ \hat{A}_n - a \}^2] \} \eqsp.
\end{equation}
\end{theorem}
\begin{proof}
We use the elementary identity
\begin{equation}
\frac{\hat{A}_n}{\hat{B}_n} - \frac{a}{b}= \frac{\hat{A}_n}{\hat{B}_n} \frac{(b-\hat{B}_n)^2}{b^2} + \frac{(\hat{A}_n-a)(b-\hat{B}_n)}{b^2} + \frac{a(b-\hat{B}_n)}{b^2} + \frac{\hat{A}_n-a}{b} \eqsp,
\end{equation}
which implies that
\[
\PE[\hat{R}_n] - r = \PE\left[\frac{\hat{A}_n}{\hat{B}_n} \frac{(b-\hat{B}_n)^2}{b^2}\right] + \frac{\PE[(\hat{A}_n-a)(b-\hat{B}_n)]}{b^2}
\]
\end{proof}


We conclude with a lemma that gives the concentration of a uniformly ergodic Markov chain. We think that this Lemma is of independent interest, and we give it under general conditions.
\begin{lemma}
\label{lem:bounded_differences_norms_markovian}
Let $(\Zset,\Zsigma)$ be a state-space and $\MKQ$ a Markov kernel on $(\Zset,\Zsigma)$ which is uniformly ergodic with mixing time $\taumix$ and stationary distribution $\invariantQ$. Let $\{g_i\}_{i=1}^n$ be a family of measurable functions from $\Zset$ to  $\rset^{d}$ such that $\supnorm{g}= \max_{i \in\{1,\ldots,n\}}\supnorm{g_i} < \infty$
and $\invariantQ(g_i)= 0$ for any $i \in\{1,\ldots,n\}$.
Then, for any initial probability $\xi$ on $(\Zset,\Zsigma)$, $n \in \nset$, $t \geq 0$, it holds
\begin{equation}
\label{eq:prob_for_norms_markov}
\PP_{\xi}\biggl(\normop{\sum\nolimits_{i=1}^{n}g_i(\State_{i})}\geq t\biggr) \leq 2 \exp\biggl\{-\frac{2t^2}{u_n^{2}}\biggr\}\eqsp, \text{ where } u_n = 3 \supnorm{g} \sqrt{n} \sqrt{\taumix}\eqsp.
\end{equation}
\end{lemma}
\begin{proof} The function $\varphi(\state_1,\dots,\state_n) := \norm{\sum_{i=1}^{n}g_i(\state_{i})}$ on $\Zset^n$ satisfies the bounded differences property. Applying \cite[Corollary 2.10]{paulin2015concentration}, we get for $t \geq \PE_{\xi}[\norm{\sum_{i=1}^{n}g_i(\State_{i})}]$,
\begin{align}
\PP_{\xi}\biggl(\norm{\sum\nolimits_{i=1}^{n}g_i(\State_{i})} \geq t\biggr) \leq  \exp\left\{-\frac{2(t-\PE_{\xi}[\norm{\sum_{i=1}^{n}g_i(\State_{i})}])^{2}}{9n\supnorm{g}^2 \taumix }\right\}\eqsp.
\end{align}
It remains to upper bound $\PE_{\xi}[\norm{\sum_{i=1}^{n}g_i(\State_{i})}]$. Note that
\begin{align}
\PE_{\xi}[\norm{\sum\nolimits_{i=1}^{n}g_i(\State_{i})}^{2}] = \sum\nolimits_{i=1}^n  \PE_{\xi}[\norm{g_i(\State_{i})}^{2}] + 2\sum\nolimits_{k=1}^{n-1}\sum\nolimits_{\ell = 1}^{n-k} \PE_{\xi}[g_k(\State_{k})^{\top} g_{k+\ell}(\State_{k+\ell})]
\end{align}
and, using $\invariantQ(g_{k+\ell}) = 0$, we obtain
\begin{align}
\textstyle
 \bigl\vert \PE_{\xi}[g_k(\State_{k})^{\top} g_{k+\ell}(\State_{k+\ell})] \bigr\vert &= \absD{ \int_{\Zset}g_k(z)^{\top}\left(\MKQ^{\ell}g_{k+\ell}(z) - \invariantQ(g_{k+\ell})\right) \xi\MKQ^{k}(\rmd z) } \leq \supnorm{g}^{2} (1/4)^{\lceil \ell/\taumix \rceil} \eqsp.
\end{align}
which implies
\begin{align}
\textstyle
\sum_{k=1}^{n-1}\sum_{\ell = 1}^{n-k}\abs{ \PE_{\xi}[g_k(\State_{k})^{\top} g_{k+\ell}(\State_{k+\ell})] } & \leq \sum_{k=1}^{n-1}\supnorm{g}^{2} (1/4)^{\lceil \ell/\taumix \rceil}  \leq (4/3) \supnorm{g}^2 \taumix n \eqsp.
\end{align}
Combining the bounds above, we upper bound $\PE_{\xi}[\norm{\sum_{i=1}^{n}g_{i}(\State_{i})}]$ as
\begin{align}
\PE_{\xi}[\norm{\sum\nolimits_{i=1}^{n}g_i(\State_{i})}] &\leq \bigl\{ \PE_{\xi}[\norm{\sum\nolimits_{i=1}^{n}g_i(\State_{i})}^2] \bigr\}^{1/2} \leq
2 \sqrt{n} \supnorm{g} \sqrt{\taumix} =: v_n \eqsp.
\end{align}
Plugging this result in \eqref{eq:prob_for_norms_markov},  we obtain that
\begin{equation}
\label{eq:MacDiarmid_markov_new}
\PP_{\xi}\biggl(\norm{\sum\nolimits_{i=1}^{n}g_i(\State_{i})} \geq t\biggr) \leq
\begin{cases}
1, \quad t < v_n, \\
 \exp\left\{-\frac{2(t-v_{n})^{2}}{9 v_{n}^2}\right\}\eqsp, \quad t \geq v_{n}\eqsp.
\end{cases}
\end{equation}
Now it is easy to see that \rhs\ of \eqref{eq:MacDiarmid_markov_new} is upper bounded by $2\exp\{-2t^2/(9 v_n^2)\}$ for any $t \geq 0$, and the statement follows.
\end{proof}

\section{Experiments}
\label{sec:experiments}

\subsection{Gaussian Mixture}
\label{subsec:app:toy_problem}
\paragraph{Bias MSE trade-off:}
We display in \Cref{fig:bias_toy_smallest,fig:mse_toy_smallest} the bias and the MSE of the \SUISIR{} estimators for the same configuration as in \Cref{fig:toy_perf} but with $\ki_0 = \lfloor 0.625 \kmax \rfloor$.
We observe $3$ times less bias than the \SNIS{} estimators but only with a $10\%$ increase of the MSE for the $N=129$ setting.
This can be also seen in \Cref{fig:ratio_bias_mse_toy}, where we show the ratio between \SUISIR{} and \SNIS{} for bias and MSE with $N = 129$.
\begin{figure}[h]
    \centering
    \begin{subfigure}{0.32\textwidth}
        \includegraphics[width=\textwidth]{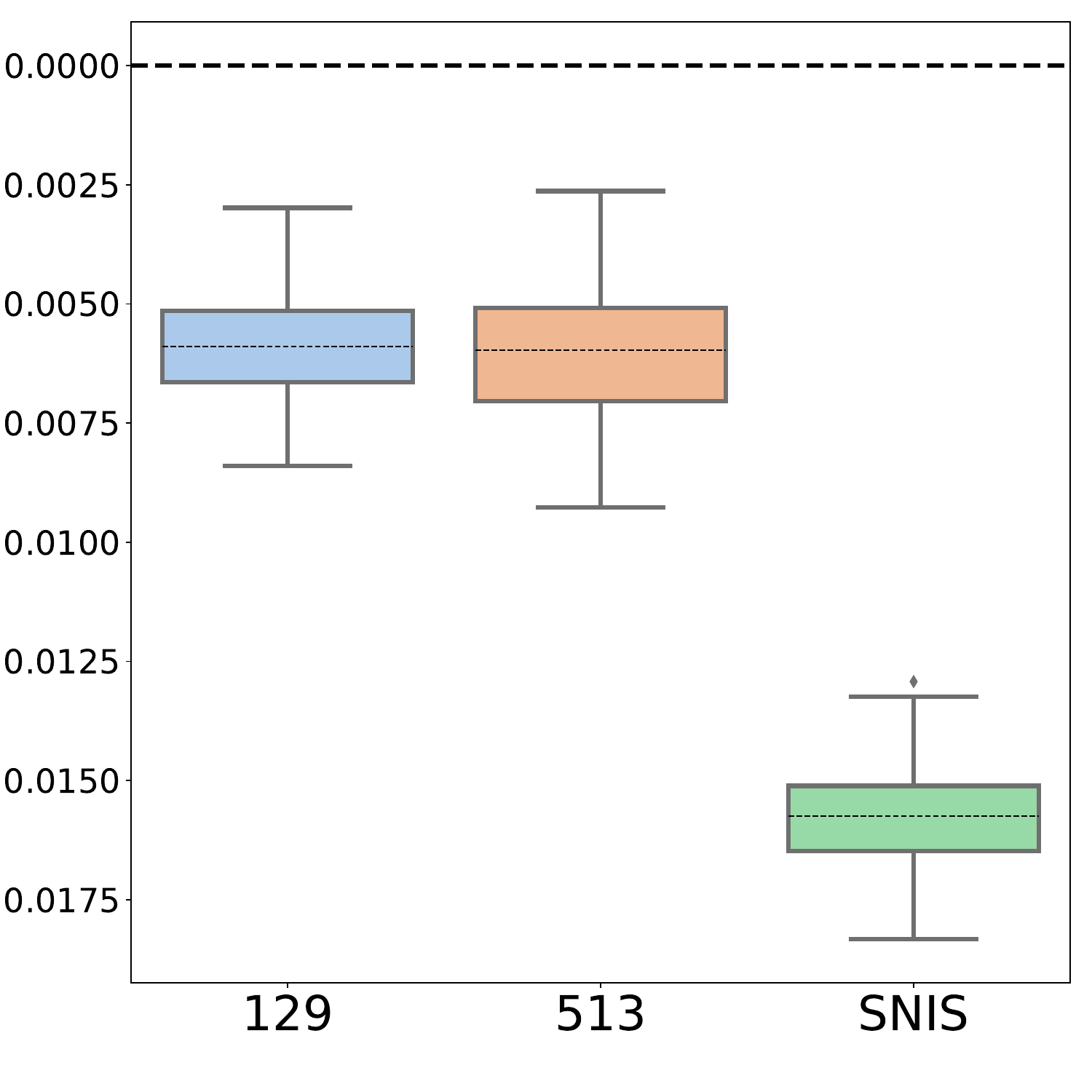}
        \caption{Bias}
        \label{fig:bias_toy_smallest}
    \end{subfigure}
    \begin{subfigure}{0.32\textwidth}
        \includegraphics[width=\textwidth]{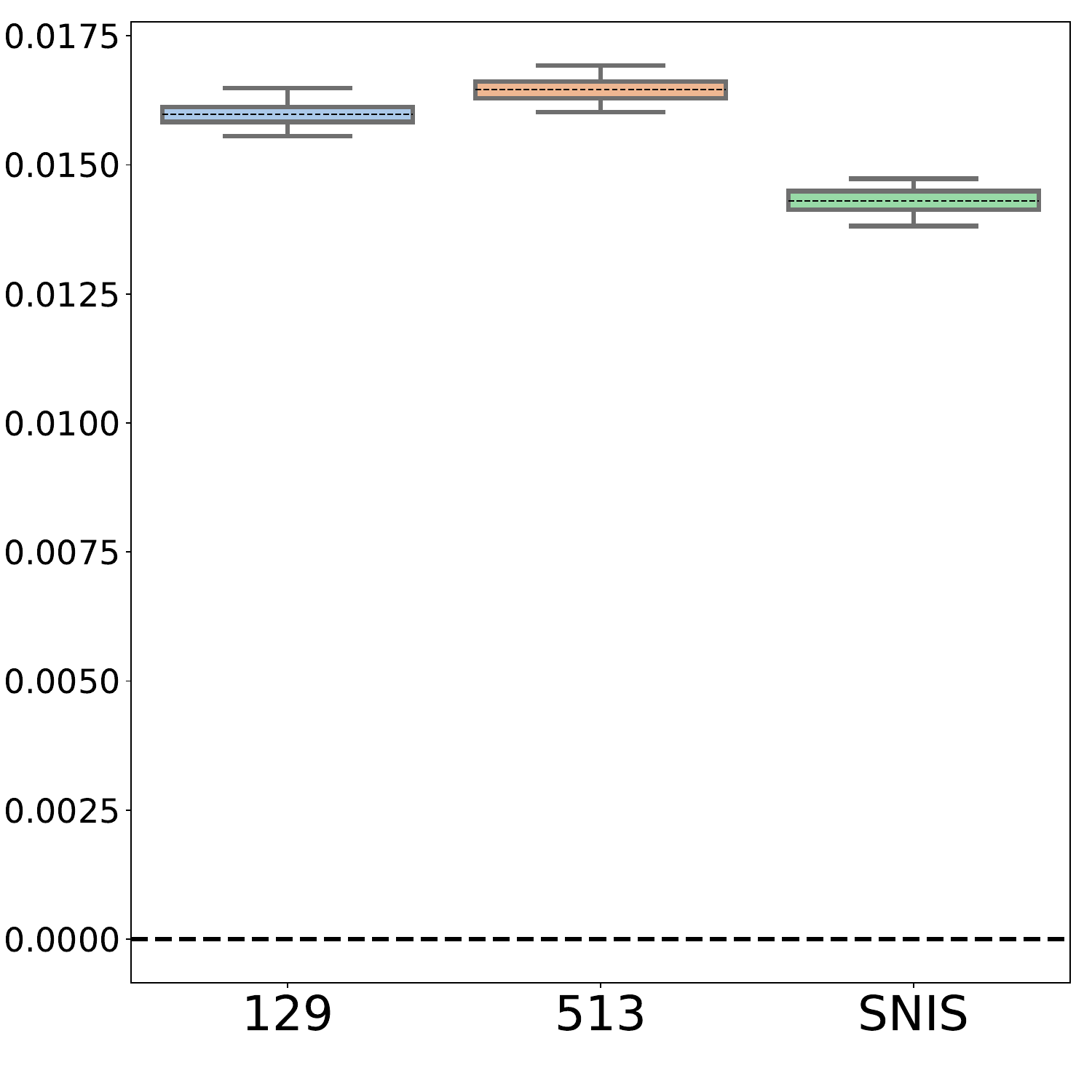}
        \caption{MSE}
        \label{fig:mse_toy_smallest}
    \end{subfigure}
    \begin{subfigure}{0.32\textwidth}
        \includegraphics[width=\textwidth]{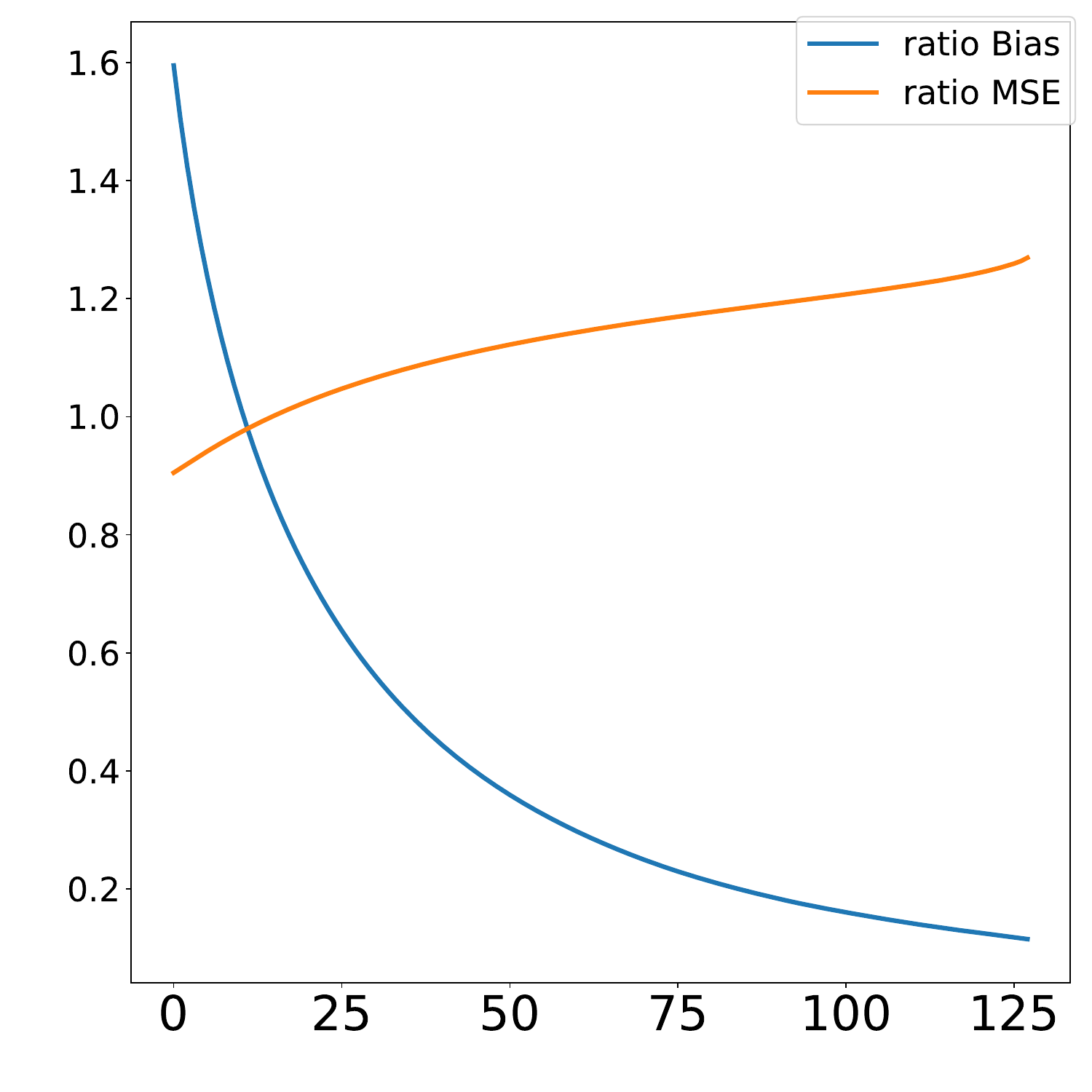}
        \caption{Ratios as $\ki_0$}
        \label{fig:ratio_bias_mse_toy}
    \end{subfigure}
    \caption{Comparison between \SNIS{} and \SUISIR{} for the same budget. In each boxplot the dotted line represents the \textbf{mean} value of the samples. In \Cref{fig:ratio_bias_mse_toy} we display the ratio between \SUISIR{} and \SNIS{} for bias and MSE with $N = 129$.}
    \label{fig:toy_smallest_bias}
\end{figure}

\paragraph{Parameters Gaussian mixture:}

The $\target$ in \Cref{subsec:toy_gauss} is a Mixture of two Gaussians in dimension $7$ with mean vectors $\vectmean_1 = (1, \dots, 1)^\intercal$ and $\vectmean_2 = (-2, 0, \dots, 0)^\intercal$ and covariance matrices $\covmat_1 = d^{-1}\Idd$ and $\covmat_2 = d^{-1}\Idd$,
where $p = 1/3$ and $\Idd$ is the identity matrix
In this setting, the quantities $\kappa[\target, \proposal]$ and $\bound$ can be estimated by Monte Carlo and Gradient ascent respectively.
Their values are approximately $7 \cdot 10^2$ and $1 \cdot 10^4$, respectively.

The sets $A$ and $B$ used to define the function $f$ are the following:
\begin{equation}
    A := [-2, 6] \times [-1, 1]^6 \eqsp, \quad B := [0.75, 1.25] \times [1, 2] \times [-0.1, 0.1]^5 \eqsp.
\end{equation}

We used this example to illustrate numerically the bounds in \Cref{theo:bias-i-SIR-recycling,theo:bias-mse-rolling}, where
each expectation was calculated by Monte Carlo using $2\cdot 10^4$ samples. We displayed in each figure the equivalent \SNIS\ estimation in a green dashed line.
For all the  bias related bounds(\Cref{theo:bias-i-SIR-recycling}(i) in \Cref{fig:bias_decay}, \Cref{theo:bias-mse-rolling}(i) in \Cref{fig:bias_rolling_decay}), we fixed a total budget of $M=6\cdot10^3$.
For \Cref{fig:bias_decay} we added a fit of the type $y=\exp(ak + b)$ to illustrate the exponential decay \wrt\ $k$.

We then increased the budget to $M=8\cdot10^4$ for the MSE and covariance bounds, in order to fully observe the stabilisation of the MSE in \Cref{fig:mse} for all the minibatch sizes $N$.
For the true value of $\target(f)$ needed for calculating the MSE, we use an estimation obtained by Monte Carlo (sampling directly from $\target$) with $4 \cdot 10^7$ samples.
In \Cref{fig:rolling_mse} we added dashed lines with the theoretical value of the $\msesnis[\burningloss M]$ with the same color as $\burningloss$.
\begin{figure}[h]
    \centering
    \begin{subfigure}{0.24\textwidth}
        \includegraphics[width=\textwidth]{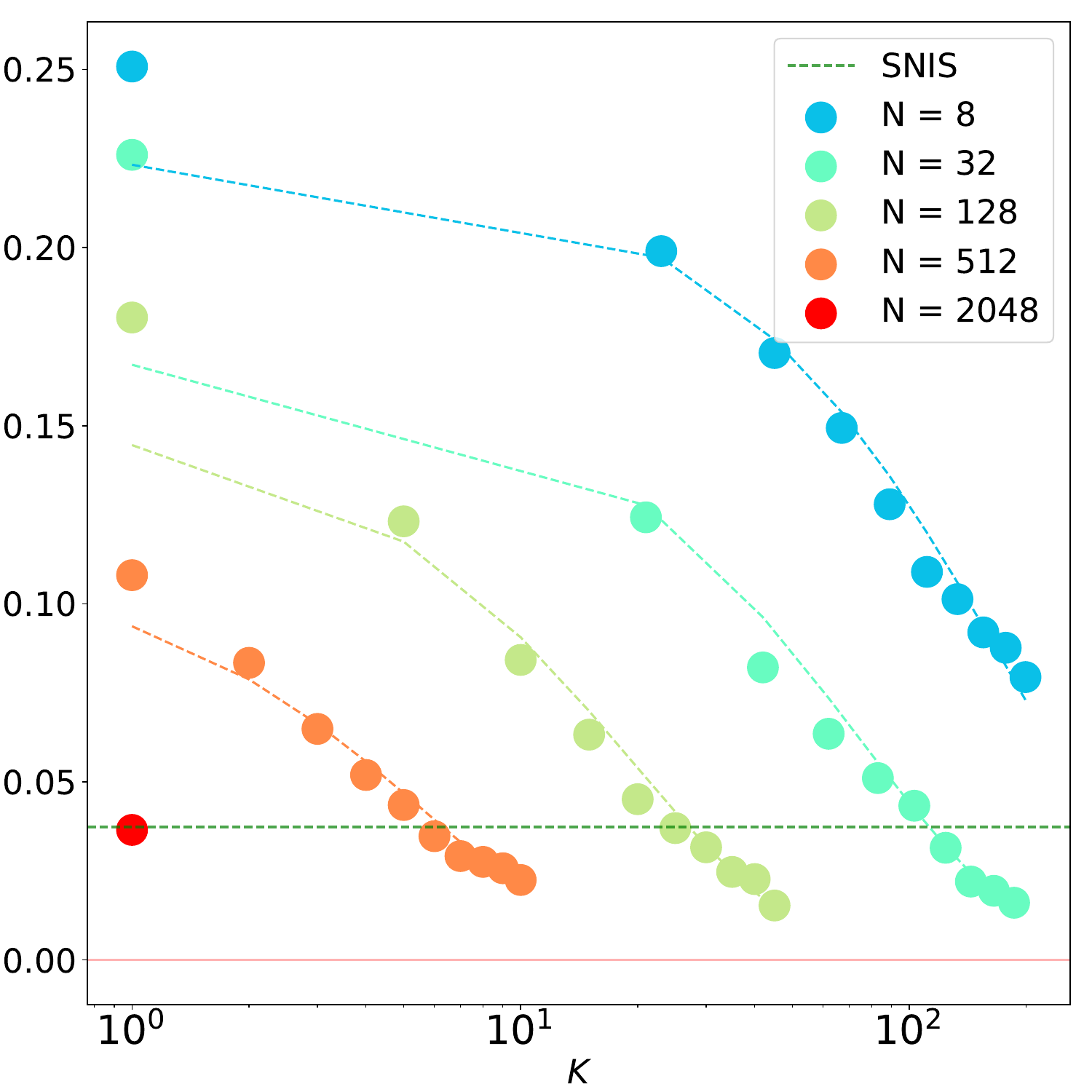}
        \caption{Bias}
        \label{fig:bias_decay}
    \end{subfigure}
    \begin{subfigure}{0.24\textwidth}
        \includegraphics[width=\textwidth]{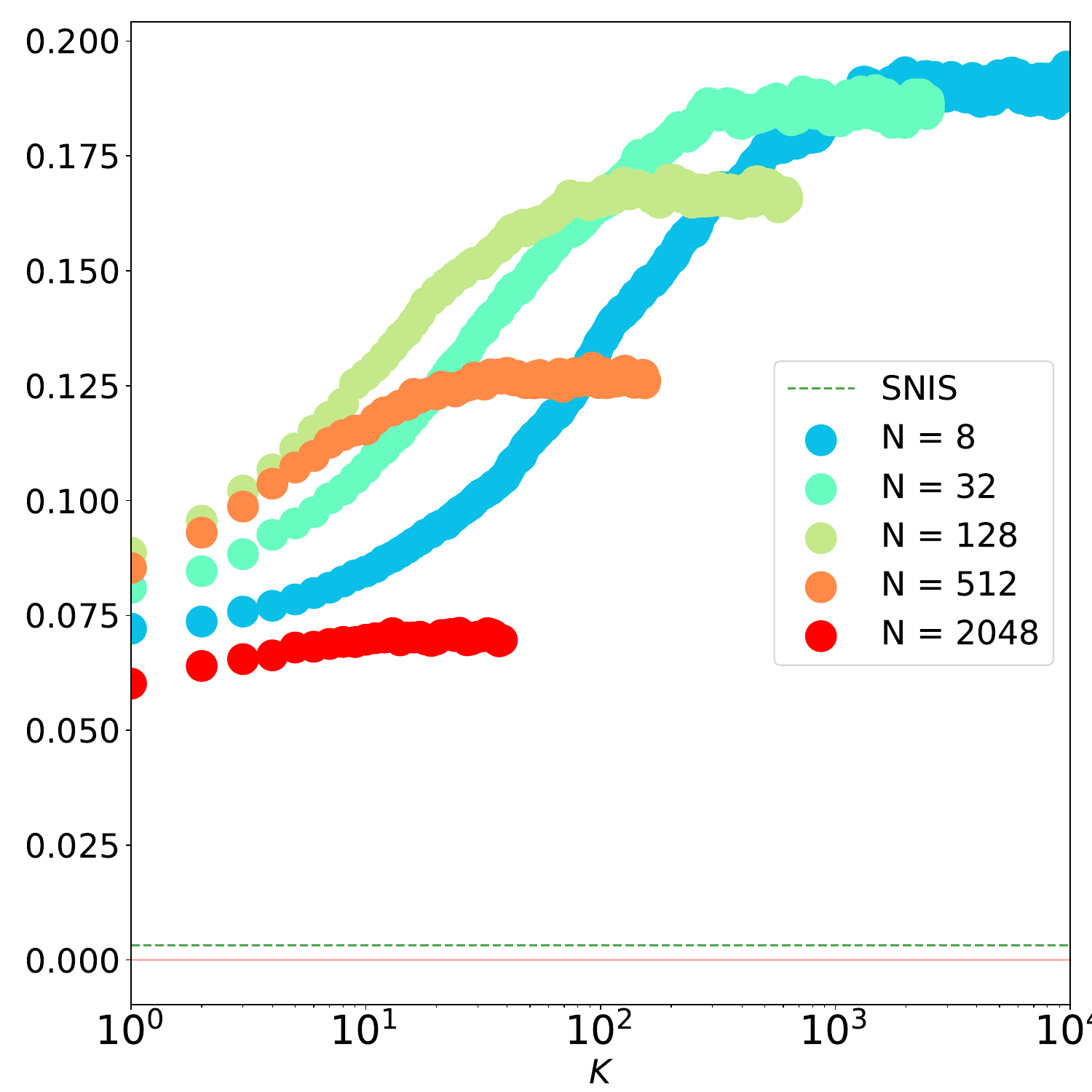}
        \caption{MSE}
        \label{fig:mse}
    \end{subfigure}
    \begin{subfigure}{0.24\textwidth}
        \includegraphics[width=\textwidth]{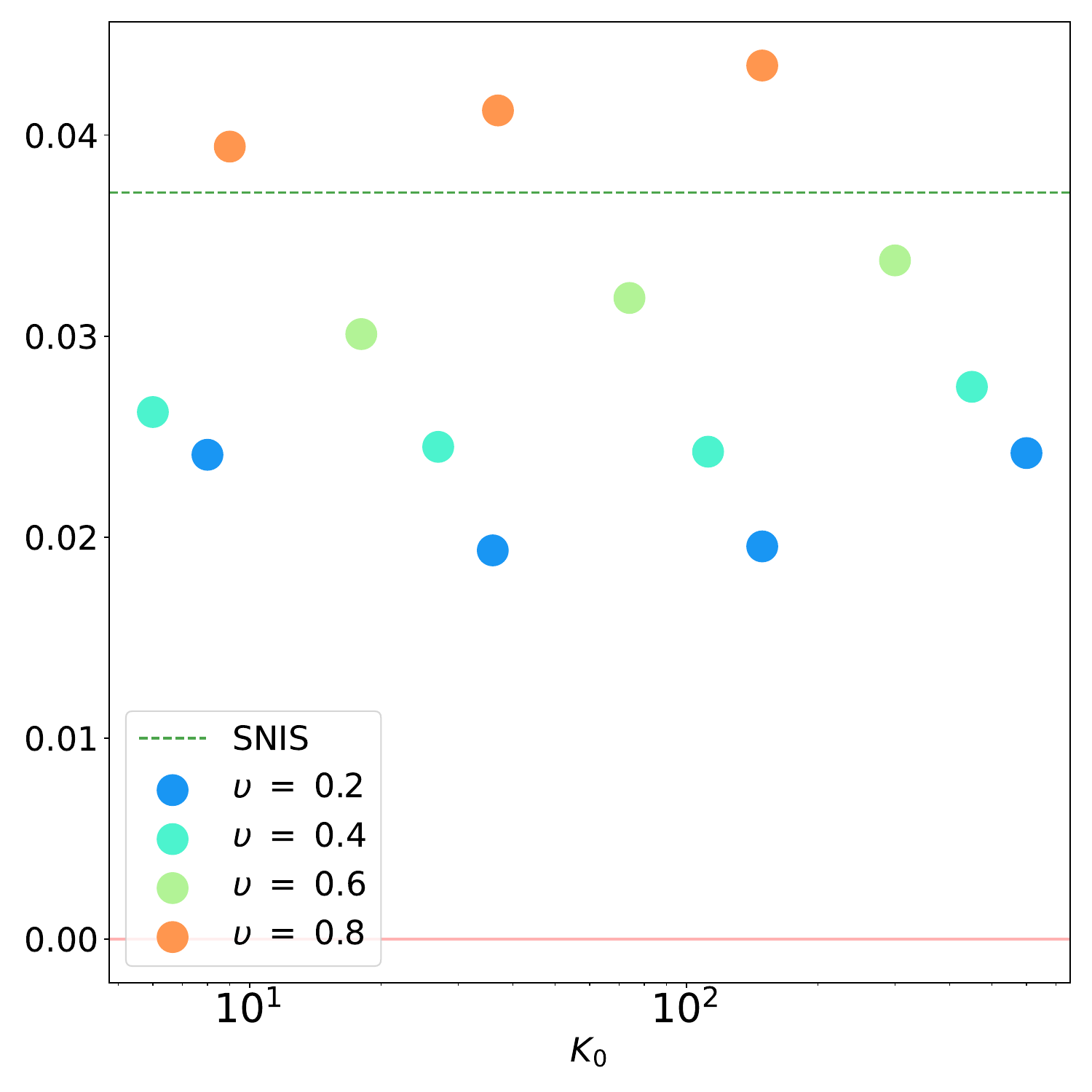}
        \caption{Rolling bias}
        \label{fig:bias_rolling_decay}
    \end{subfigure}
    \begin{subfigure}{0.24\textwidth}
        \includegraphics[width=\textwidth]{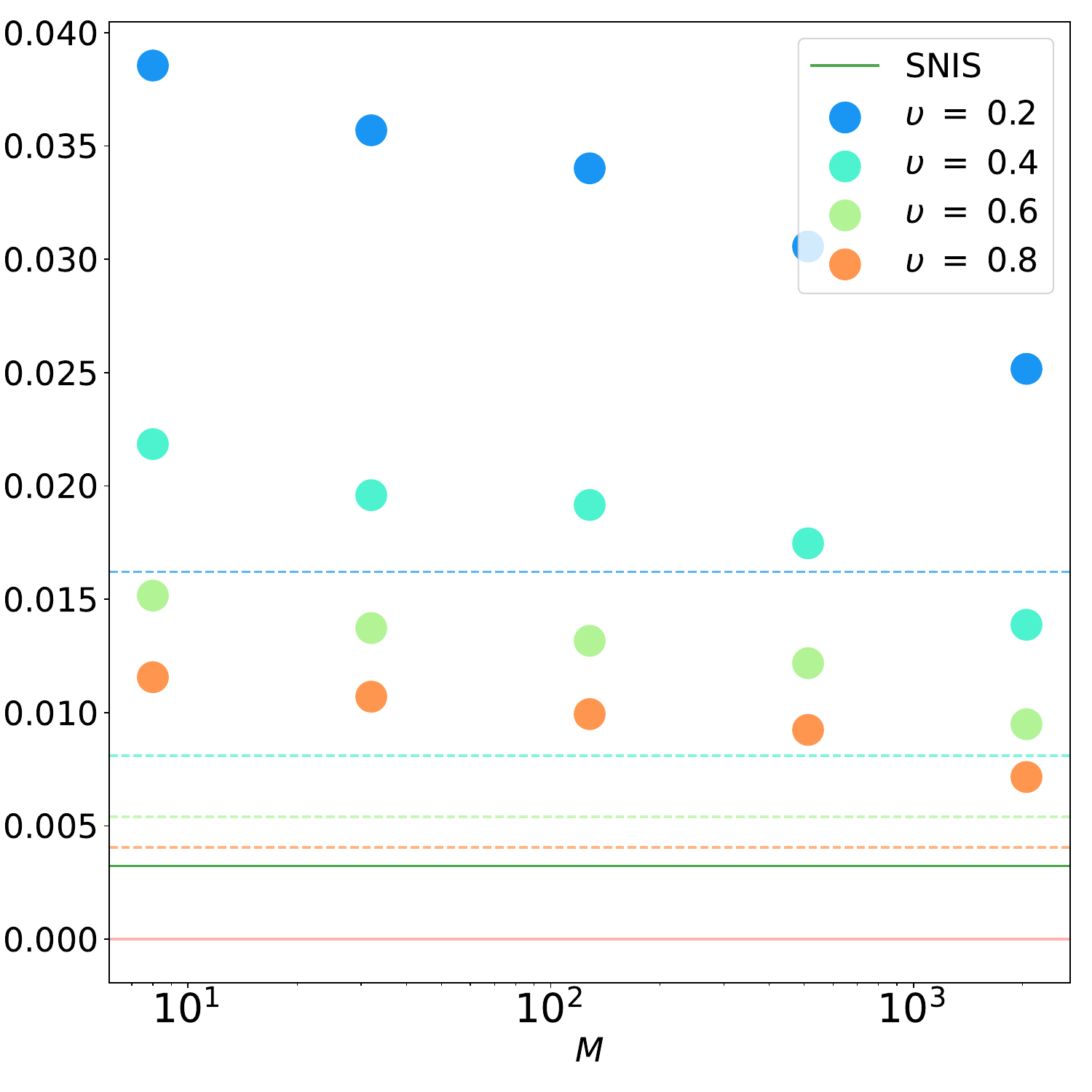}
        \caption{Rolling MSE}
        \label{fig:rolling_mse}
    \end{subfigure}
    \caption{Visualization of the theoretical bounds from \Cref{theo:bias-i-SIR-recycling,theo:bias-mse-rolling}.}
    \label{fig:bounds_theorem}
\end{figure}

\paragraph{Comparison with zero bias SNIS methods:}

    There exists estimators based on \SNIS\ that have no bias, such as the estimator proposed in \cite{pmlr-v89-middleton19a} and refered to as \UnbiasedPIMH\,.
    One of the main differences between such estimator and \SUISIR\ is that \SUISIR\ works under a pre-established budget of samples, whereas in Unbiased-PIMH the number of samples used to produce an estimate varies due to the accept-reject procedure.
    Even though the two estimators have different goals, it can be of interest to compare both of them in the case where there is a restriction in the total number of samples available.

    We proceed to a fixed-budget ($M$) comparison between \SUISIR\ and the "Rao Blackwellized" version of the algorithm proposed at \cite{pmlr-v89-middleton19a} in the Gaussian Mixture example.
    In order to do so, it's necessary to impose the fixed-budget constraint to the \UnbiasedPIMH\ estimator. A single iteration of the estimator from \UnbiasedPIMH\
    with batch-size $N$ needs $rN$ samples where $r \in \nset$ is a random number satisfying $r\geq 2$. Therefore, there are two ways of applying the constraint to
    \UnbiasedPIMH\,:
    \begin{itemize}
        \item \textbf{Soft}: For a given $N$, generate estimations using \UnbiasedPIMH\ until the number of samples is larger than $M$ and \textbf{keep} the last estimation. Therefore, all the estimators from \UnbiasedPIMH\ will have used \textbf{at least} $M$ samples. All the estimations generated are averaged to generate a single estimate.
        \item \textbf{Hard}: For a given $N$, generate estimations using \UnbiasedPIMH\ until the number of total samples used is larger than $M$ and \textbf{discard} the last estimation. Therefore, all the estimators from \UnbiasedPIMH\ will have used \textbf{at most} $M$ samples. \textbf{If no estimations were produced under the budget cap (first iteration used more than $M$ samples), then we consider it a miss}. All the estimations generated are averaged to create a single estimate.
    \end{itemize}
    The code used to run the experiments is available at \footnote{\href{https://github.com/gabrielvc/br_snis/blob/master/notebooks/Comparison_Unbiased-PIMH.ipynb}{https://github.com/gabrielvc/br\_snis/blob/master/notebooks/Comparison\_Unbiased-PIMH.ipynb}}.
    For both cases, the following values of $M$ are used in the comparison: $2^{16}, 2^{12}, 2^{9}$. For each estimator, a total of $1024$ Monte Carlo replications are used to estimate the mean and the standard deviation of the estimator. Note that in the \textbf{Hard} framework, \textbf{it can happen that less than 1024 replications are used for the \UnbiasedPIMH\ estimator}. The number of failed estimations is reported in the tables for the framework \textbf{Hard} for each configuration.

    For each configuration of the \SUISIR\ estimator (defined by $N$, $\kmax$), we have used $90\%$ burn-in period ($\ki_{0} = \lfloor 0.9 \kmax \rfloor)$ and $\kmax$ rounds of bootstrap ($\kmax$ permutations of the input samples).

    The following values were calculated:
    \begin{itemize}
        \item \textbf{Bias}: The mean of the estimations minus $\operatorname{ref}$ over $1024$ replications
        \item \textbf{Std}: The standard deviation of the estimations over $1024$ replications.
        \item \textbf{Fails}: The number of replications that failed to produce a single estimation for a given budget $M$. This is only applicable for the \UnbiasedPIMH\ estimator and in the \textbf{Hard} framework.
        \item \textbf{average M}: The average (over the $1024$ replications) total cost of the estimator. For \SUISIR\ and \SNIS\ this is always $M$. For \UnbiasedPIMH\ in the \textbf{Soft} framework it is larger than $M$. In the \textbf{Hard} framework it is smaller than $M$.
    \end{itemize}
%

\SetKwComment{Comment}{/* }{ */}
\begin{algorithm}
    \caption{Unbiased-PIMH}\label{algorithm:pimh}
    \KwData{$N \geq 0$}
    $e_1, \operatorname{lwav}_1 \gets \SNIS{}(N)$ \Comment*[r]{SNIS also returning the average log weights}
    $e_2, \operatorname{lwav}_2 \gets \SNIS{}(N)$\;
    \If{$\operatorname{lwav}_1 < \operatorname{lwav}_2$}{
        $\operatorname{swap}(e_1, \operatorname{lwav}_1; e_2, \operatorname{lwav}_2)$
    }
    $u = \log \operatorname{rand}() $ \;
    \If{$u < \operatorname{lwav}_1$ and $u < \operatorname{lwav}_2$}{
        $\tau = 1$\;
    }
    $t \gets 1$\;
    $\tau = \infty$\;
    \While{$\tau = \infty$}{
      $e_1 = e_1 + (e_1 - e_2)$ \;
      $e_{p}, \operatorname{lwav}_{p} = \SNIS{}(N)$\;
      $t = t + 1$\;
     $u = \log \operatorname{rand}()$;
        \If{$u < \operatorname{lwav}_{p} - \operatorname{lwav}_{1} $}{
        $e_1, \operatorname{lwav}_{1} = e_p, \operatorname{lwav}_{p}$\;
    }
        
        \If{$u < \operatorname{lwav}_{p} - \operatorname{lwav}_{1} $}{
        $e_2, \operatorname{lwav}_{1} = e_p, \operatorname{lwav}_{p}$\;
    }
        \If{$u < \operatorname{lwav}_1$ and $u < \operatorname{lwav}_2$}{
        $\tau = t$\;
    }
    }
\end{algorithm}
\begin{filecontents*}{comparison_pimh_67108864.csv}
N,algorithm,bias,std deviation,replications,average M,k
65536,SNIS,-0.0029,0.0605,1024,65536.0,
65,BR-SNIS,-0.0010,0.0658,1024,65536.0,1024
129,BR-SNIS,-0.0006,0.0689,1024,65536.0,512
257,BR-SNIS,0.0003,0.0678,1024,65536.0,256
513,BR-SNIS,0.0019,0.0670,1024,65536.0,128
16384,Unbiased-PIMH,0.0065,0.1005,1024,71904.0,
8192,Unbiased-PIMH,0.0058,0.1066,1024,71040.0,
4096,Unbiased-PIMH,0.0082,0.1139,1024,69316.0,
2048,Unbiased-PIMH,0.0053,0.1174,1024,67764.0,
\end{filecontents*}
\DTLloaddb{comparison_pimh_16}{comparison_pimh_67108864.csv}
\begin{table}[]
    \centering
    \begin{tabular}{|c|c|c|c|c|c|}%
         \hline
          N & k & algorithm & Bias & std & average M \\
          \hline
          \DTLforeach*{comparison_pimh_16}{\N=N, \k=k, \algo=algorithm, \bias=bias, \averageM =average M, \std=std deviation}{
            \N & \k & \algo & \bias & \std & \averageM \DTLiflastrow{}{\\}} 
          \\\hline 
    \end{tabular}
    \caption{$M = 2^{16}$ in the \textbf{Soft} framework.}
    \label{table:comparison_pimh_soft_16}
\end{table}
\begin{filecontents*}{comparison_pimh_4194304.csv}
N,algorithm,bias,std deviation,replications,average M,k
4096,SNIS,-0.0365,0.1946,1024,4096.0,
65,BR-SNIS,-0.0314,0.2211,1024,4096.0,64
129,BR-SNIS,-0.0358,0.2214,1024,4096.0,32
257,BR-SNIS,-0.0281,0.2282,1024,4096.0,16
513,BR-SNIS,-0.0296,0.2351,1024,4096.0,8
1024,Unbiased-PIMH,0.0587,0.6073,1024,5388.0,
512,Unbiased-PIMH,0.0678,0.8086,1024,5027.5,
256,Unbiased-PIMH,0.1258,1.1492,1024,4730.0,
128,Unbiased-PIMH,0.2364,1.9521,1024,4629.6,
\end{filecontents*}
\DTLloaddb{comparison_pimh_12}{comparison_pimh_4194304.csv}
\begin{table}[]
    \centering
    \begin{tabular}{|c|c|c|c|c|c|}%
         \hline
          N & k & algorithm & Bias & std & average M \\
          \hline
          \DTLforeach*{comparison_pimh_12}{\N=N, \k=k, \algo=algorithm, \bias=bias, \averageM =average M, \std=std deviation}{
            \N & \k & \algo & \bias & \std & \averageM \DTLiflastrow{}{\\}} 
          \\\hline 
    \end{tabular}
    \caption{$M = 2^{12}$ in the \textbf{Soft} framework.}
    \label{table:comparison_pimh_soft_12}
\end{table}
\begin{filecontents*}{comparison_pimh_524288.csv}
N,algorithm,bias,std deviation,replications,average M,k
512,SNIS,-0.1458,0.2420,1024,512.0,
65,BR-SNIS,-0.1537,0.2468,1024,512.0,8
129,BR-SNIS,-0.1543,0.2444,1024,512.0,4
257,BR-SNIS,-0.1426,0.2600,1024,512.0,2
128,Unbiased-PIMH,-0.0048,1.3924,1024,841.5,
64,Unbiased-PIMH,0.1997,2.5677,1024,796.4,
32,Unbiased-PIMH,0.2365,4.1642,1024,708.1,
16,Unbiased-PIMH,0.3670,5.1533,1024,685.3,
\end{filecontents*}
\DTLloaddb{comparison_pimh_9}{comparison_pimh_524288.csv}
\begin{table}[]
    \centering
    \begin{tabular}{|c|c|c|c|c|c|}%
         \hline
          N & k & algorithm & Bias & std & average M \\
          \hline
          \DTLforeach*{comparison_pimh_9}{\N=N, \k=k, \algo=algorithm, \bias=bias, \averageM =average M, \std=std deviation}{
            \N & \k & \algo & \bias & \std & \averageM \DTLiflastrow{}{\\}} 
          \\\hline 
    \end{tabular}
    \caption{$M = 2^{9}$ in the \textbf{Soft} framework.}
    \label{table:comparison_pimh_soft_9}
\end{table}

\begin{filecontents*}{comparison_pimh_67108864_mechant.csv}
N,algorithm,bias,std deviation,replications,average M,k,fails
65536,SNIS,-0.0029,0.0605,1024,65536.0,,
65,BR-SNIS,-0.0006,0.0650,1024,65536.0,1024,
129,BR-SNIS,-0.0023,0.0645,1024,65536.0,512,
257,BR-SNIS,-0.0024,0.0657,1024,65536.0,256,
513,BR-SNIS,0.0000,0.0693,1024,65536.0,128,
16384,Unbiased-PIMH,-0.0028,0.0885,1017,57520.0,,7
8192,Unbiased-PIMH,-0.0008,0.1029,1024,59264.0,,0
4096,Unbiased-PIMH,-0.0014,0.1026,1024,61956.0,,0
2048,Unbiased-PIMH,0.0008,0.1106,1024,63244.0,,0
\end{filecontents*}
\DTLloaddb{comparison_pimh_16_mechant}{comparison_pimh_67108864_mechant.csv}
\begin{table}[]
    \centering
    \begin{tabular}{|c|c|c|c|c|c|c|}%
         \hline
          N & k & algorithm & Bias & std & average M & Fails \\
          \hline
          \DTLforeach*{comparison_pimh_16_mechant}{\N=N, \k=k, \algo=algorithm, \bias=bias, \averageM =average M, \std=std deviation, \fails=fails}{
            \N & \k & \algo & \bias & \std & \averageM & \fails\DTLiflastrow{}{\\}} 
          \\\hline 
    \end{tabular}
    \caption{$M = 2^{16}$ in the \textbf{Hard} framework.}
    \label{table:comparison_pimh_hard_16}
\end{table}
\begin{filecontents*}{comparison_pimh_4194304_mechant.csv}
N,algorithm,bias,std deviation,replications,average M,k,fails
4096,SNIS,-0.0365,0.1946,1024,4096.0,,
65,BR-SNIS,-0.0252,0.2270,1024,4096.0,64,
129,BR-SNIS,-0.0296,0.2221,1024,4096.0,32,
257,BR-SNIS,-0.0338,0.2218,1024,4096.0,16,
513,BR-SNIS,-0.0486,0.2243,1024,4096.0,8,
1024,Unbiased-PIMH,-0.0901,0.2353,921,2922.0,,103
512,Unbiased-PIMH,-0.0833,0.3368,1000,3343.0,,24
256,Unbiased-PIMH,-0.0547,0.4815,1015,3554.8,,9
128,Unbiased-PIMH,-0.0634,0.4433,1020,3683.1,,4
\end{filecontents*}
\DTLloaddb{comparison_pimh_12_mechant}{comparison_pimh_4194304_mechant.csv}
\begin{table}[]
    \centering
    \begin{tabular}{|c|c|c|c|c|c|c|}%
         \hline
          N & k & algorithm & Bias & std & average M & Fails \\
          \hline
          \DTLforeach*{comparison_pimh_12_mechant}{\N=N, \k=k, \algo=algorithm, \bias=bias, \averageM =average M, \std=std deviation, \fails=fails}{
            \N & \k & \algo & \bias & \std & \averageM & \fails\DTLiflastrow{}{\\}} 
          \\\hline 
    \end{tabular}
    \caption{$M = 2^{12}$ in the \textbf{Hard} framework.}
    \label{table:comparison_pimh_hard_12}
\end{table}
\begin{filecontents*}{comparison_pimh_524288_mechant.csv}
N,algorithm,bias,std deviation,replications,average M,k,fails
512,SNIS,-0.1458,0.2420,1024,512.0,,
65,BR-SNIS,-0.1376,0.2636,1024,512.0,8,
129,BR-SNIS,-0.1456,0.2565,1024,512.0,4,
257,BR-SNIS,-0.1358,0.2585,1024,512.0,2,
128,Unbiased-PIMH,-0.1962,0.2200,814,306.9,,210
64,Unbiased-PIMH,-0.1947,0.3200,951,367.8,,73
32,Unbiased-PIMH,-0.1999,0.4001,988,398.0,,36
16,Unbiased-PIMH,-0.2057,0.7366,1008,423.2,,16
\end{filecontents*}
\DTLloaddb{comparison_pimh_9_mechant}{comparison_pimh_524288_mechant.csv}
\begin{table}[]
    \centering
    \begin{tabular}{|c|c|c|c|c|c|c|}%
         \hline
          N & k & algorithm & Bias & std & average M & Fails \\
          \hline
          \DTLforeach*{comparison_pimh_9_mechant}{\N=N, \k=k, \algo=algorithm, \bias=bias, \averageM =average M, \std=std deviation, \fails=fails}{
            \N & \k & \algo & \bias & \std & \averageM & \fails\DTLiflastrow{}{\\}} 
          \\\hline 
    \end{tabular}
    \caption{$M = 2^{9}$ in the \textbf{Hard} framework.}
    \label{table:comparison_pimh_hard_9}
\end{table}
    We have compared both estimators in two different frameworks (\textbf{Hard} and \textbf{Soft}) with three different budgets $M=2^{16}$ (\cref{table:comparison_pimh_hard_16,table:comparison_pimh_soft_16}), $M=2^{12}$ (\cref{table:comparison_pimh_hard_12,table:comparison_pimh_soft_12}) and $M=2^{9}$ (\cref{table:comparison_pimh_hard_9,table:comparison_pimh_soft_9}).
    We observed that in general the \SUISIR\ estimator has smaller standard deviation, with the difference of standard deviation being important for the smaller budgets (3 times less for $M=2^{12}$ and 10 times less  for $M=2^{9}$ in the \textbf{Soft} framework).

    For the \textbf{Hard} framework, we can see that the empirical bias of \SUISIR\ is always at most equal to the empirical bias of \UnbiasedPIMH\ .
    For the \textbf{Soft} framework, we observed that for $M=2^{16}$ that both methods have similar performance, with \SUISIR\ having negligible bias in this setting.
    For $M=2^{12}$ and $M=2^{9}$, \SUISIR\ has in general a smaller empirical biais and the standard deviation of \UnbiasedPIMH\ is considerably higher.



\subsection{Bayesian Logistic regression}
\label{sec:app:bayesian_log_reg}
The importance distribution used in the Bayesian logistic regression example is given by the mean-field variational distribution \cite{blei2017variational}. More precisely, given the target $\target$ given in \Cref{subsec:bayesian_logist}, the proposal $\proposal$ is a Gaussian distribution with mean $\mu$ and diagonal covariance $\operatorname{diag}(\sigma)$, where $\mu, \sigma$ are learnt by maximization of the Evidence Lower Bound (ELBO):
\begin{equation}
    \elbo(\mu, \sigma) = \int \log (\pi(\theta)/\proposal(\theta)) \proposal(\theta) \rmd \theta\eqsp.
\end{equation}
In both \Cref{fig:other_components_posterior_mean,fig:bias_datasets}, the optimal $\ki$ for a given budget $M$ was chosen by grid search over all the factors of $M$. The final settings are shown in \Cref{table:optimal_configs}.
\begin{filecontents*}{optimal_config_bayesian_logistic.csv}
Dataset,component,M,k,bias_ratio,mse_ratio,N
breast,8,256,4,0.7166017361041406,0.5161329,65
breast,8,512,8,0.6098939764393542,0.37699804,65
breast,8,1024,16,0.4974424605511296,0.25788137,65
breast,8,2048,16,0.3642247776123418,0.15253006,129
breast,8,4096,64,0.2730326264661665,0.10058451,65
breast,11,256,4,0.7235418966205331,0.5252126,65
breast,11,512,8,0.5926082494866624,0.35458276,65
breast,11,1024,16,0.4651847123034435,0.22112013,65
breast,11,2048,32,0.3780792657753272,0.15330017,65
breast,11,4096,64,0.2390215075271501,0.076986626,65
breast,14,256,4,0.7120025870607917,0.5092884,65
breast,14,512,8,0.6013441663019053,0.3661307,65
breast,14,1024,16,0.4733181440141556,0.23166706,65
breast,14,2048,32,0.3802873817660793,0.15989645,65
breast,14,4096,64,0.2393190912467712,0.07741983,65
heart,5,32,4,0.179713334731288,0.06454286,9
heart,5,64,8,0.0997023788299792,0.08700924,9
heart,5,128,8,0.0172661866712789,0.120308064,17
heart,5,256,32,0.0079787232423595,0.15524691,9
heart,5,512,4,0.0539215701399102,0.26200792,129
heart,8,32,4,0.1795685341667951,0.053348962,9
heart,8,64,8,0.0301507526129045,0.045561727,9
heart,8,128,8,0.0047846889952153,0.08113677,17
heart,8,256,16,0.0253456230109297,0.14407352,17
heart,8,512,32,0.0655021844972533,0.23542936,17
heart,12,32,4,0.1748768466744745,0.04214102,9
heart,12,64,8,0.0305851065381288,0.035442915,9
heart,12,128,16,0.0058774140884672,0.072220124,9
heart,12,256,4,0.0,0.12453477,65
heart,12,512,32,0.0349206344977408,0.22000328,17
covertype,6,512,4,,,129
covertype,6,1024,8,,,129
covertype,6,2048,16,,,129
covertype,6,4096,2,,,2049
covertype,6,8192,4,,,2049
covertype,17,512,2,,,257
covertype,17,1024,2,,,513
covertype,17,2048,2,,,1025
covertype,17,4096,2,,,2049
covertype,17,8192,4,,,2049
covertype,23,512,2,,,257
covertype,23,1024,2,,,513
covertype,23,2048,4,,,513
covertype,23,4096,16,,,257
covertype,23,8192,32,,,257
\end{filecontents*}
\DTLloaddb{optimal_confs}{optimal_config_bayesian_logistic.csv}
\begin{table}[]
    \centering
    \begin{tabular}{|c|c|c|c|c|}%
          Dataset & component & M & $\kmax$ & N \\ 
          \hline
          \DTLforeach*{optimal_confs}{\Dataset=Dataset,\component=component, \M=M, \k=k, \N=N}{
            \Dataset & \component &\M & \k & \N\\
          }
    \end{tabular}
    \caption{Optimal configurations for \Cref{fig:other_components_posterior_mean,fig:bias_datasets}}
    \label{table:optimal_configs}
\end{table}
\begin{figure}
    \begin{subfigure}{0.32\textwidth}
        \includegraphics[width=\textwidth]{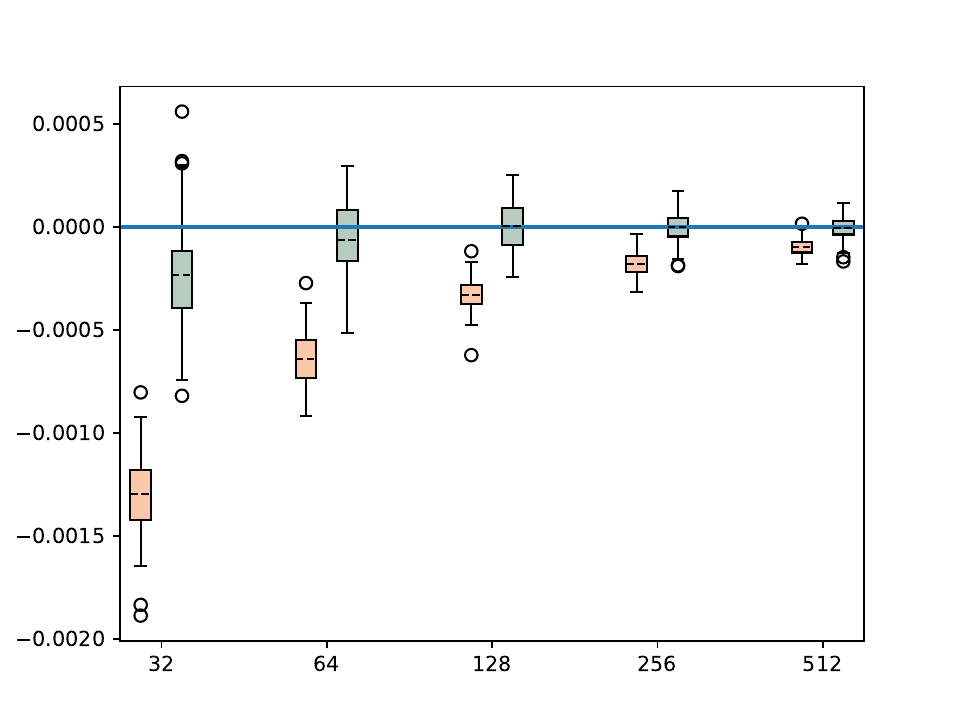}
        \caption{Heart 5}
        \label{fig:heart_component_5}
    \end{subfigure}
    \begin{subfigure}{0.32\textwidth}
        \includegraphics[width=\textwidth]{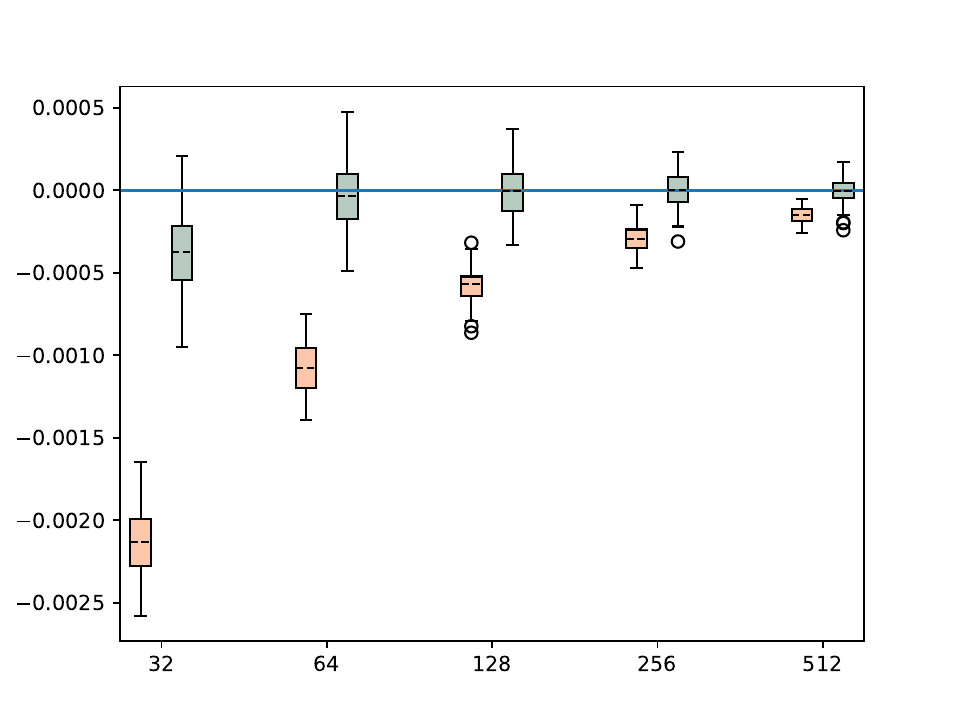}
        \caption{Heart 12}
        \label{fig:heart_component_12}
    \end{subfigure}
        \begin{subfigure}{0.32\textwidth}
        \includegraphics[width=\textwidth]{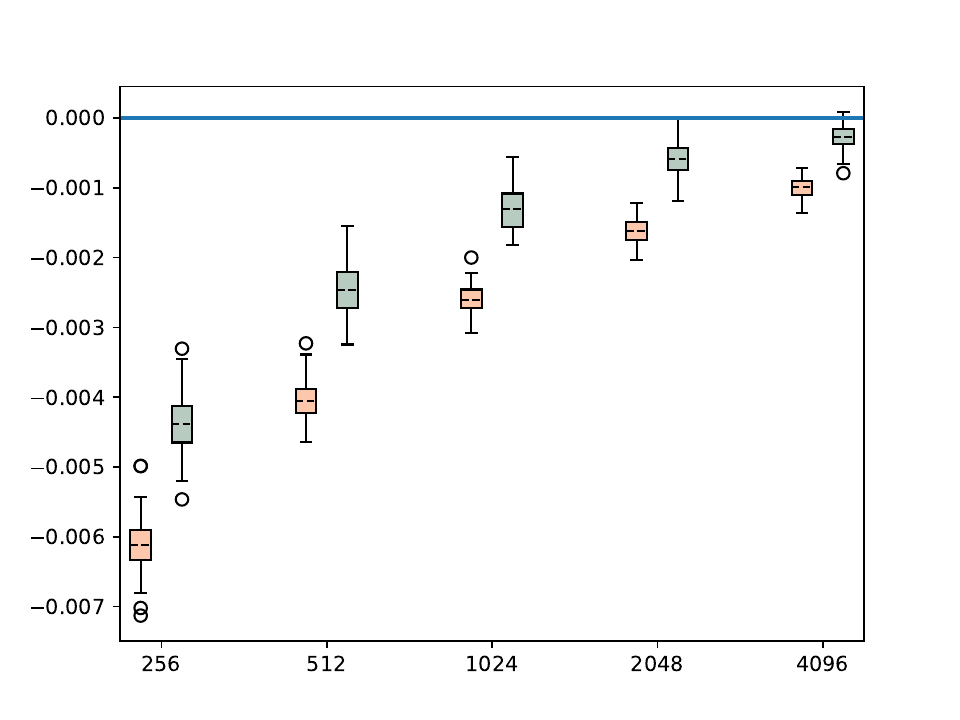}
        \caption{Breast 8}
        \label{fig:breast_component_8}
    \end{subfigure}
    \begin{subfigure}{0.32\textwidth}
        \includegraphics[width=\textwidth]{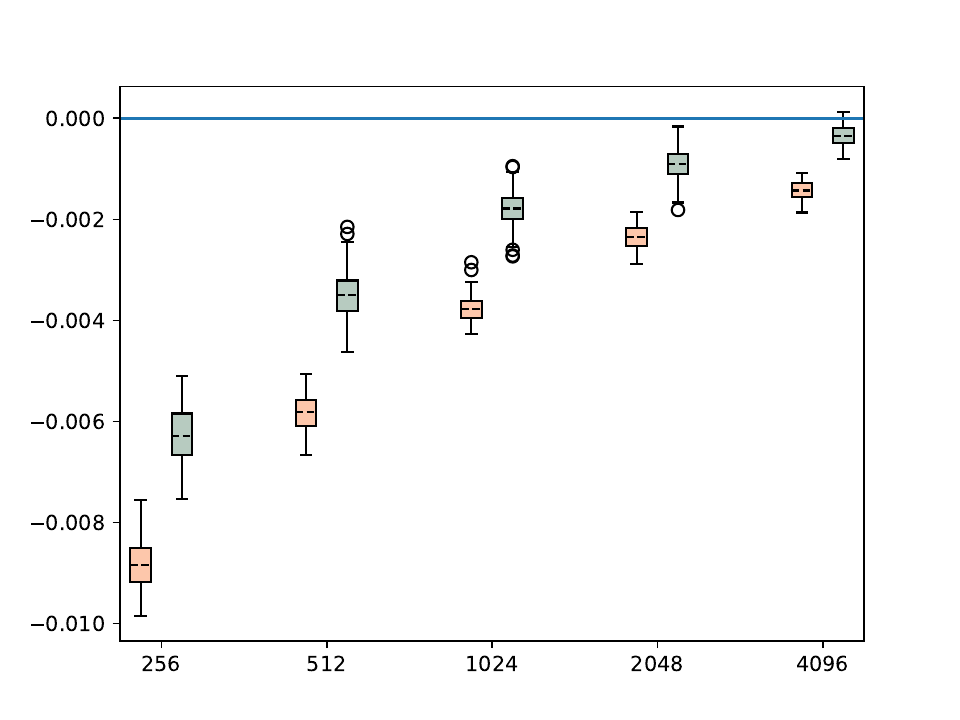}
        \caption{Breast 14}
        \label{fig:breast_component_14}
    \end{subfigure}
    \begin{subfigure}{0.32\textwidth}
        \includegraphics[width=\textwidth]{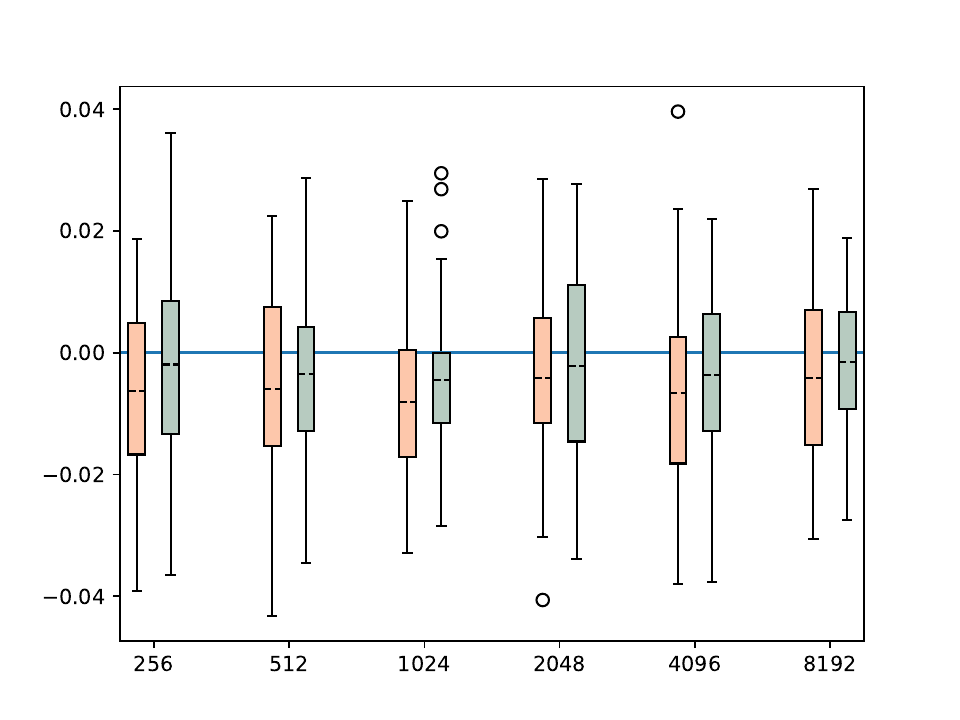}
        \caption{Covertype 17}
        \label{fig:covertype_component_17}
    \end{subfigure}
    \begin{subfigure}{0.32\textwidth}
        \includegraphics[width=\textwidth]{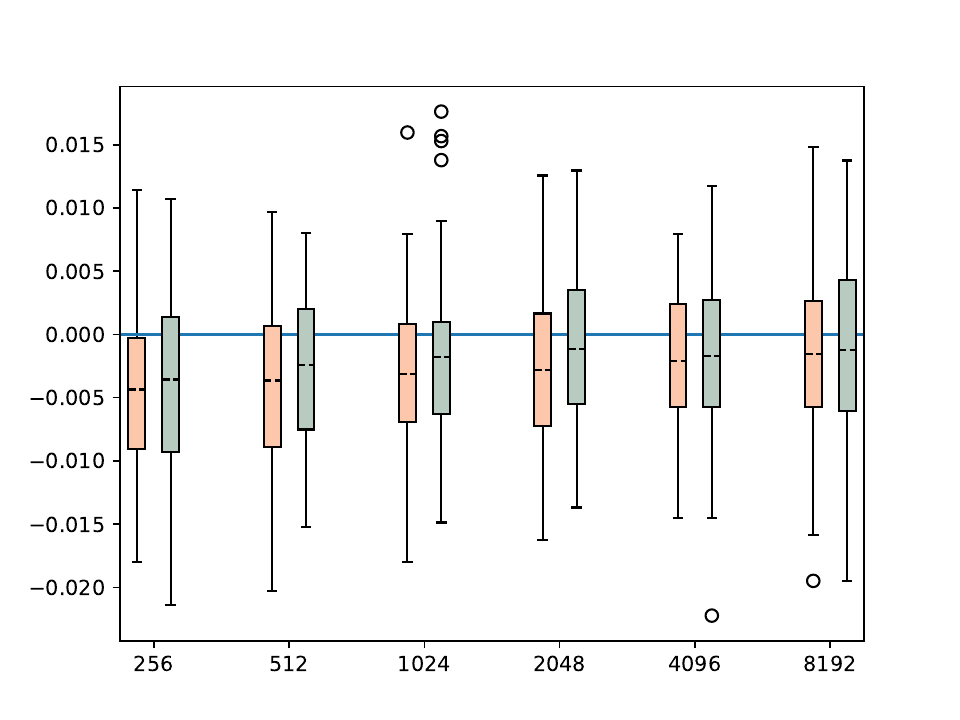}
        \caption{Covertype 23}
        \label{fig:covertype_component_23}
    \end{subfigure}
    \caption{Visualisation of the distribution of the bias for the Heart Failure and Breast cancer dataset for other components of $\theta$}
    \label{fig:other_components_posterior_mean}
\end{figure}
\subsection{Importance Weighted Auto-Encoders}
\label{sec:IWAE}
The details of training procedure are given in \Cref{table:training_parameters_iwae}. The train ELBO for each latent dimension is shown in \Cref{fig:ELBO_train}.
For each network we display generated images in \Cref{fig:generation_dim_20,fig:generation_dim_50}. The generation is done by sampling from the prior $p_{\theta}$.

For the log likelihood comparison in \Cref{table:comparison_log_lik}, we use \SNIS\ with the posterior as importance distribution and a total of $10^4$ samples.
Therefore, the estimation of the log likelihood is:
\begin{equation}
    \hat{\mathcal{L}} = T^{-1} \sum_{j=1}^{T}  \sum_{i=1}^M \omega_{\theta, \phi, x_j} \log p_{\theta}(x_j \mid z^j_i)
\end{equation}
with $\omega_{\theta, \phi, x}(z) = p_\theta(x)/q_\phi(z \mid x)$ where $z_i^j$ is sampled from $q_{\phi}(\cdot \mid x_j)$.
\begin{table}[h]
\centering
\begin{tabular}{|c|c|c|c|}
  \hline
       Name         & Value         & Name       & Value       \\  \hline
    Batch size      & 32            & Epochs     &  100        \\  \hline
  IS samples(IWAE, Un IWAE)&  64    & Encoder inner layer size  &  256   \\ \hline
  Learning rate &  $10^{-3}$    & Decoder inner layer size  &  256   \\
\hline
\end{tabular}
\caption{Training Parameters}
\label{table:training_parameters_iwae}
\end{table}
\begin{figure}[h]
    \centering
    \begin{subfigure}{0.24\textwidth}
        \includegraphics[width=\textwidth]{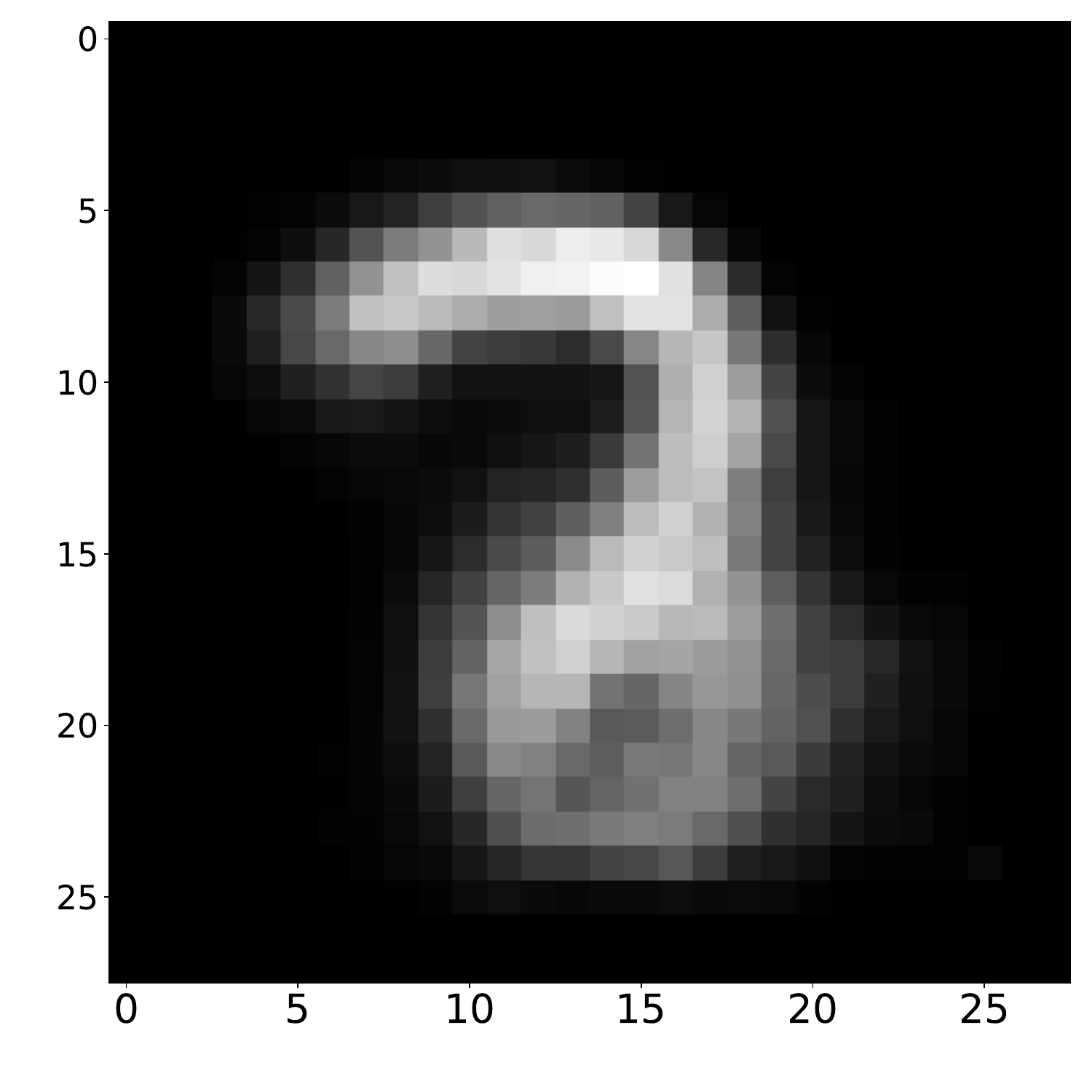}
        \caption{VAE}
        \label{fig:vae_dim_20_img}
    \end{subfigure}
    \begin{subfigure}{0.24\textwidth}
        \includegraphics[width=\textwidth]{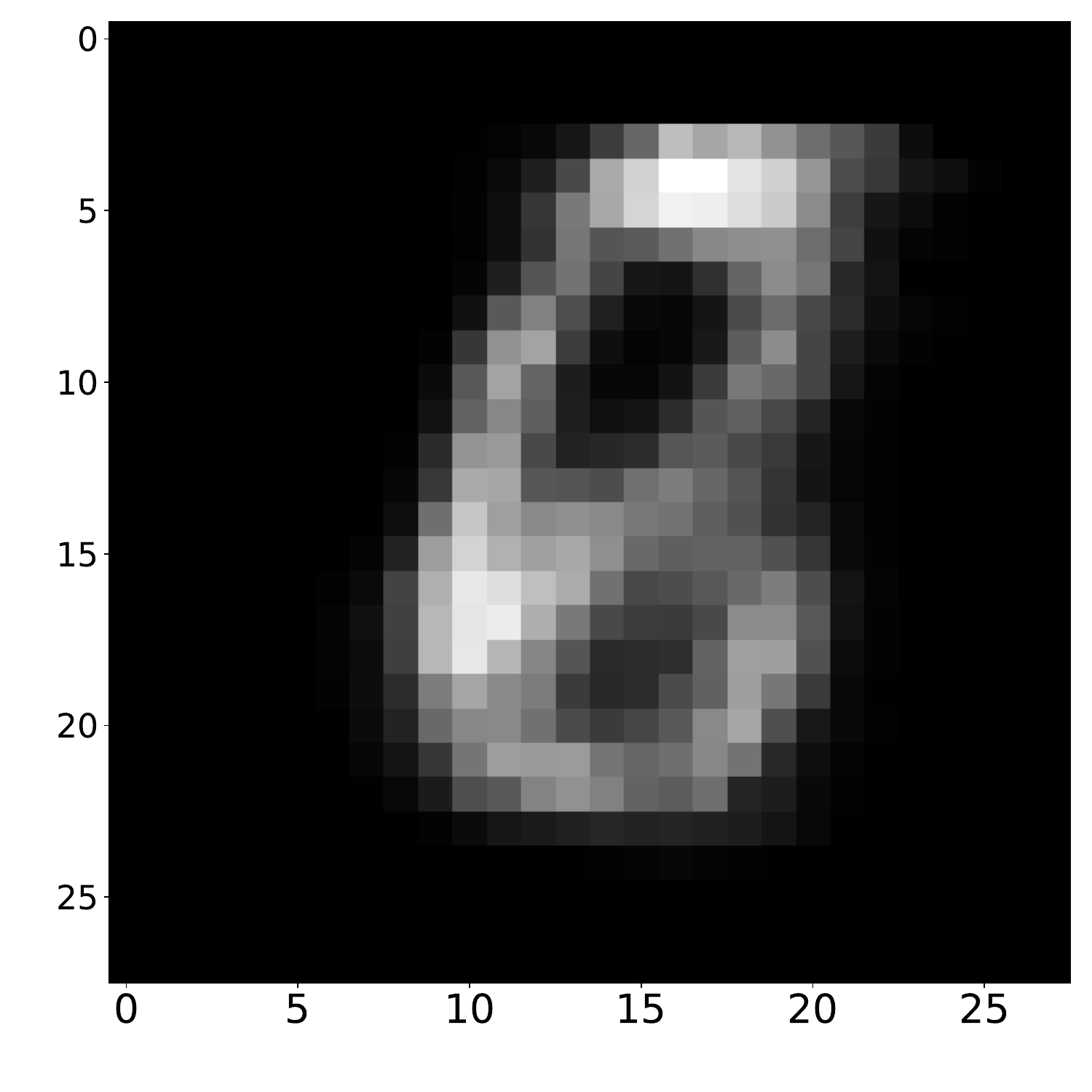}
        \caption{IWAE}
        \label{fig:iwae_dim_20_img}
    \end{subfigure}
    \begin{subfigure}{0.24\textwidth}
        \includegraphics[width=\textwidth]{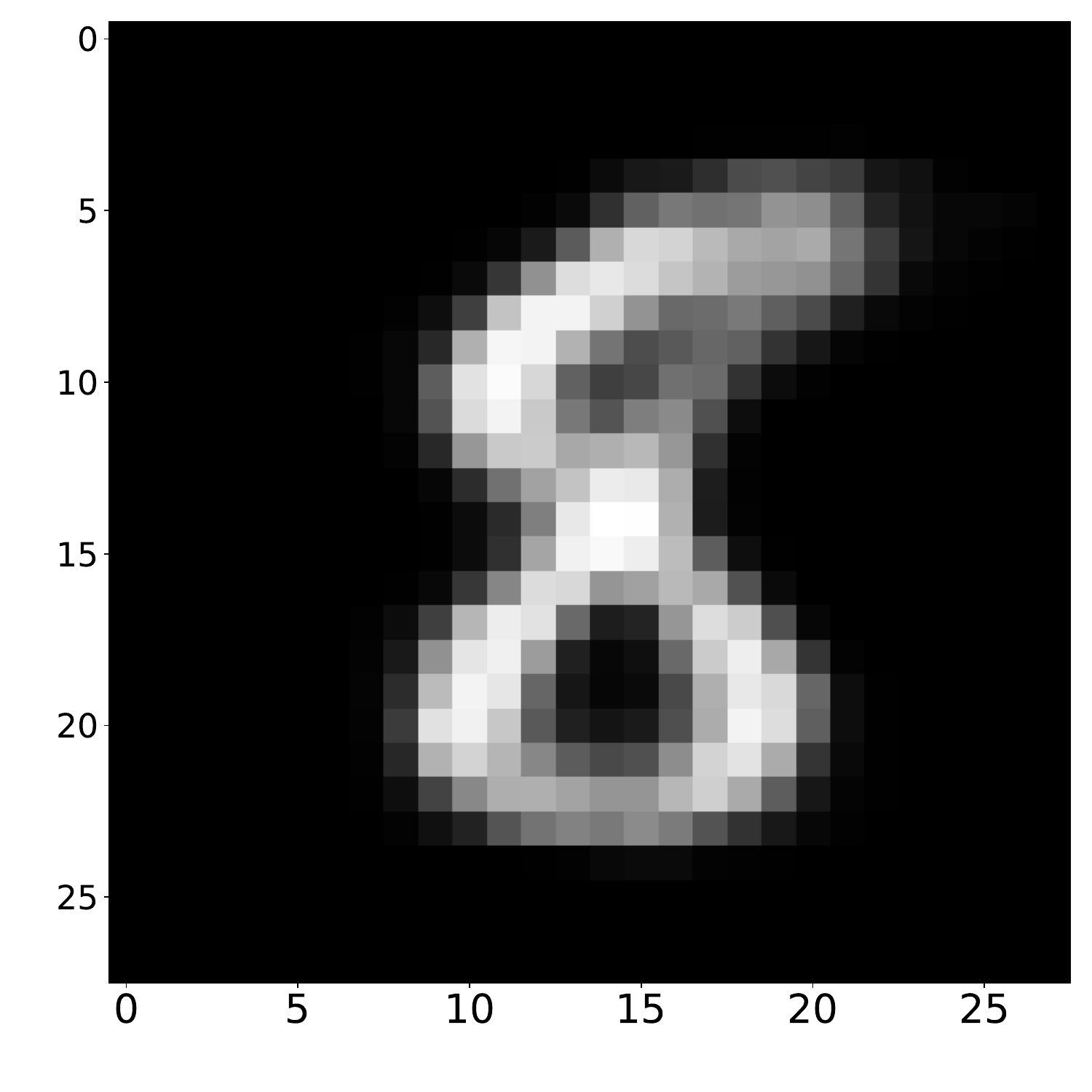}
        \caption{Un-IWAE ($k=4$)}
        \label{fig:un_iwae_dim_4_img}
    \end{subfigure}
    \begin{subfigure}{0.24\textwidth}
        \includegraphics[width=\textwidth]{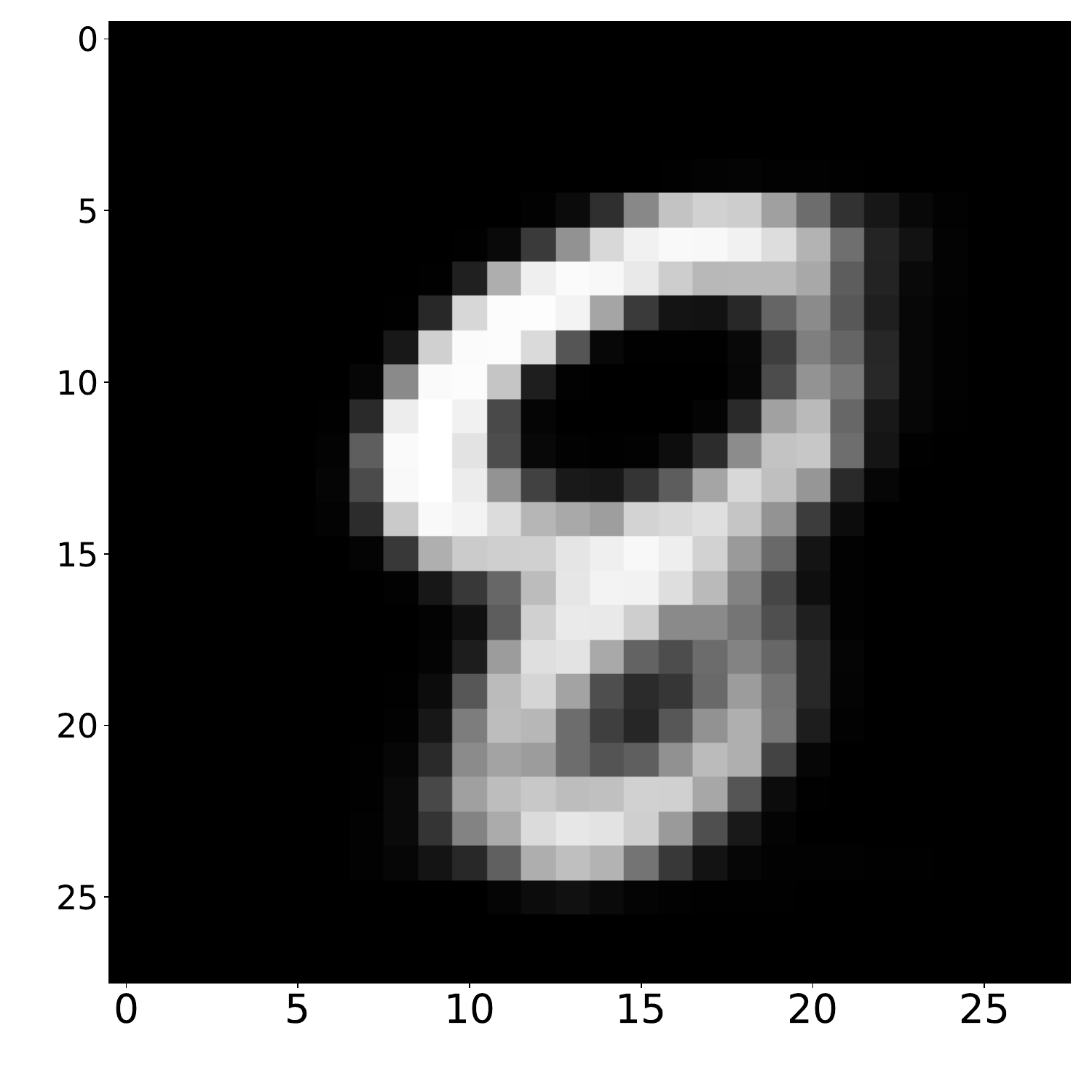}
        \caption{Un-IWAE ($k=8$)}
        \label{fig:un_iwae_dim_8_img}
    \end{subfigure}
    \caption{Images generated by the autoencoders from a sample of the prior in the latent space for dimension $20$.}
    \label{fig:generation_dim_20}
\end{figure}
\begin{figure}[h]
    \centering
    \begin{subfigure}{0.24\textwidth}
        \includegraphics[width=\textwidth]{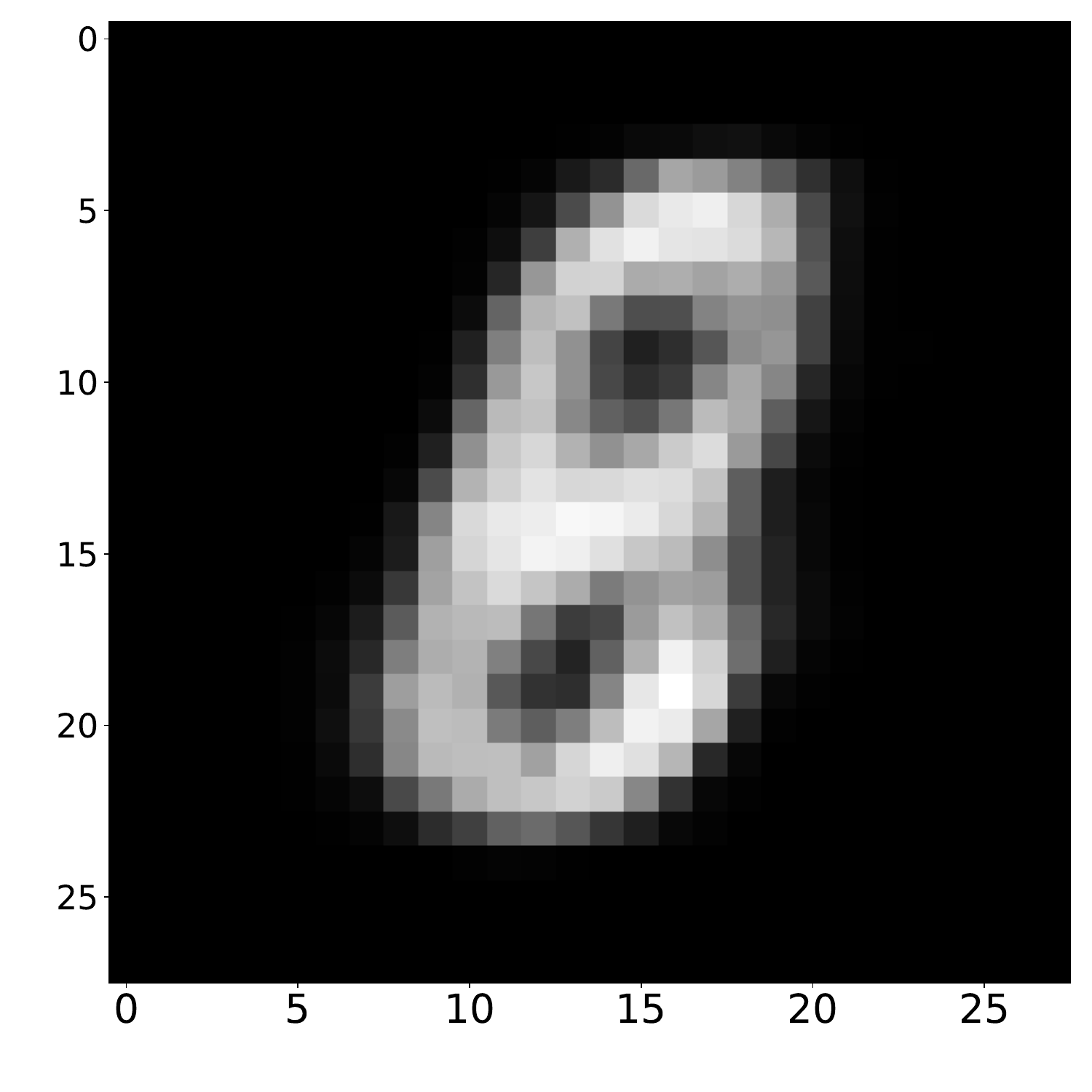}
        \caption{VAE}
        \label{fig:vae_dim_50_img}
    \end{subfigure}
    \begin{subfigure}{0.24\textwidth}
        \includegraphics[width=\textwidth]{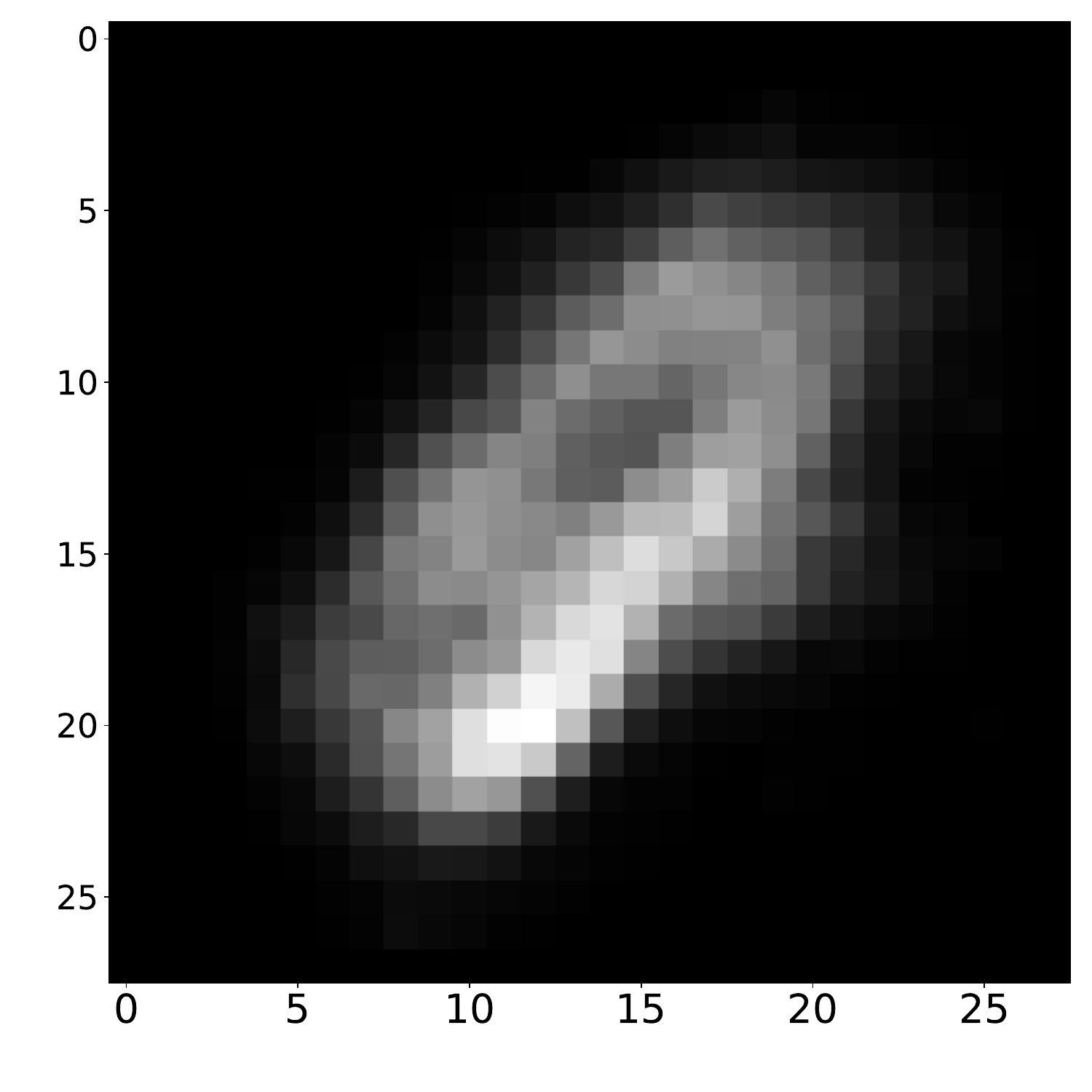}
        \caption{IWAE}
        \label{fig:iwae_dim_50_img}
    \end{subfigure}
    \begin{subfigure}{0.24\textwidth}
        \includegraphics[width=\textwidth]{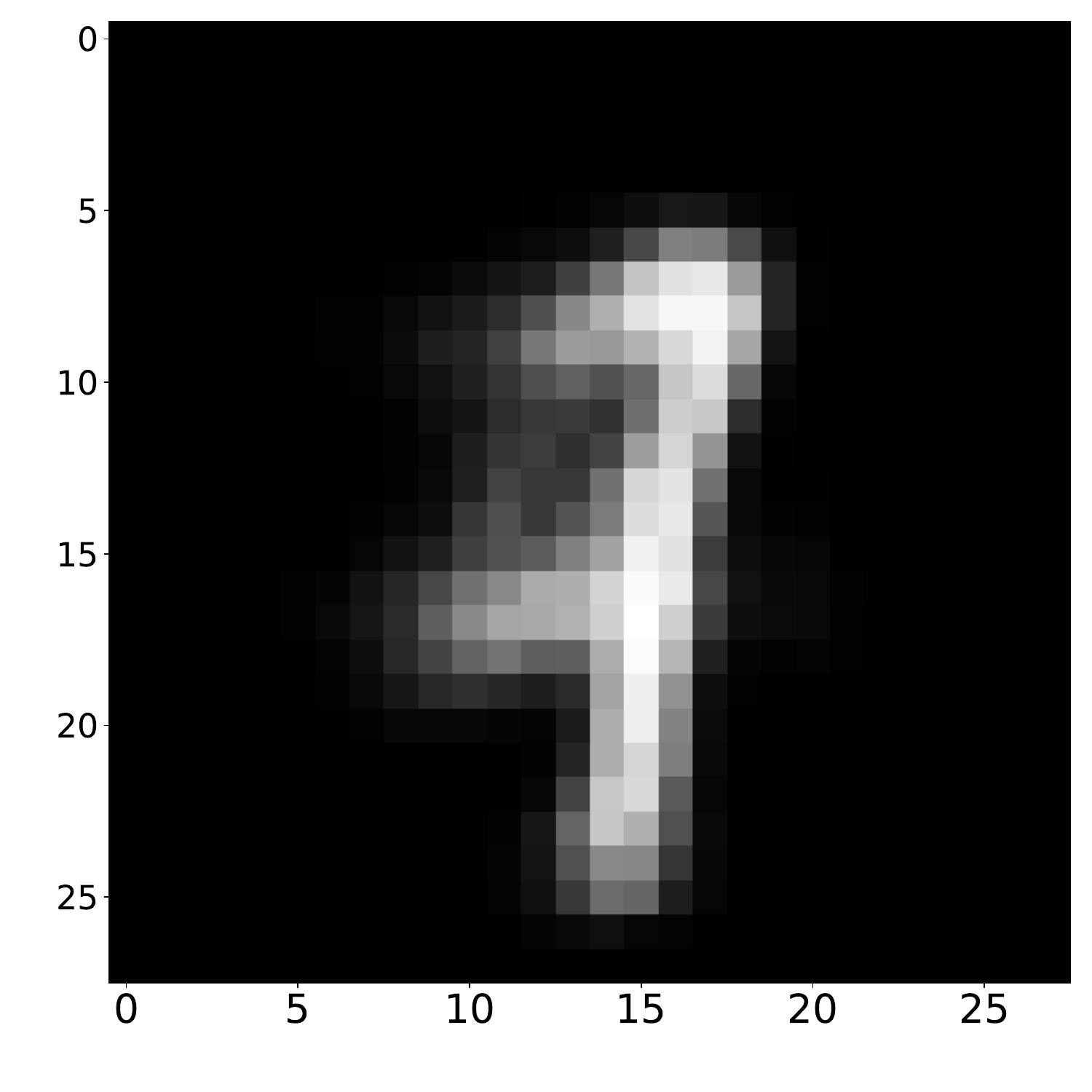}
        \caption{Un-IWAE ($k=4$)}
        \label{fig:un_iwae_dim_4_img}
    \end{subfigure}
    \begin{subfigure}{0.24\textwidth}
        \includegraphics[width=\textwidth]{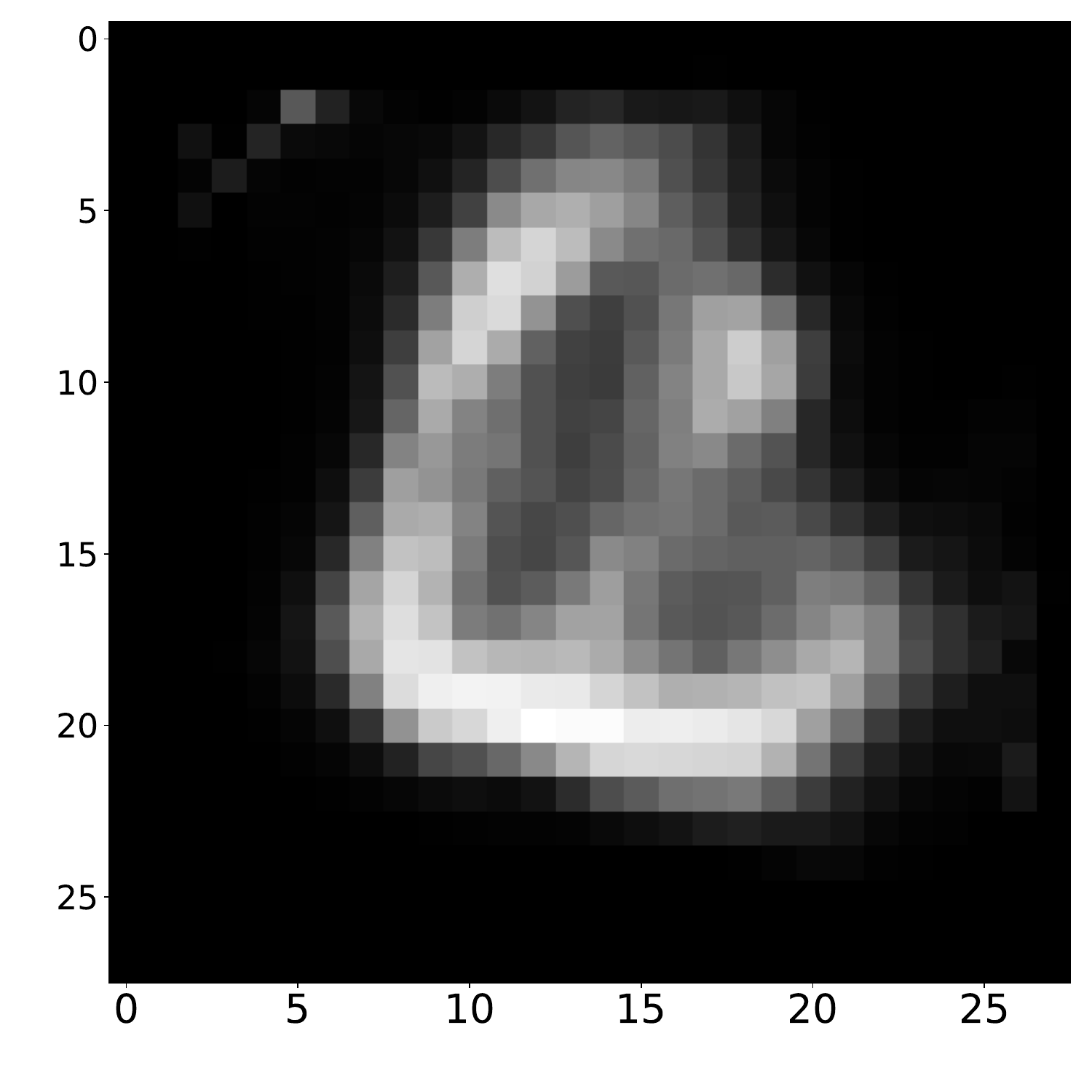}
        \caption{Un-IWAE ($k=8$)}
        \label{fig:un_iwae_dim_8_img}
    \end{subfigure}
    \caption{Images generated by the autoencoders from a sample of the prior in the latent space for dimension $50$.}
    \label{fig:generation_dim_50}
\end{figure}
\begin{figure}[h]
    \centering
    \begin{subfigure}{0.32\textwidth}
        \includegraphics[width=\textwidth]{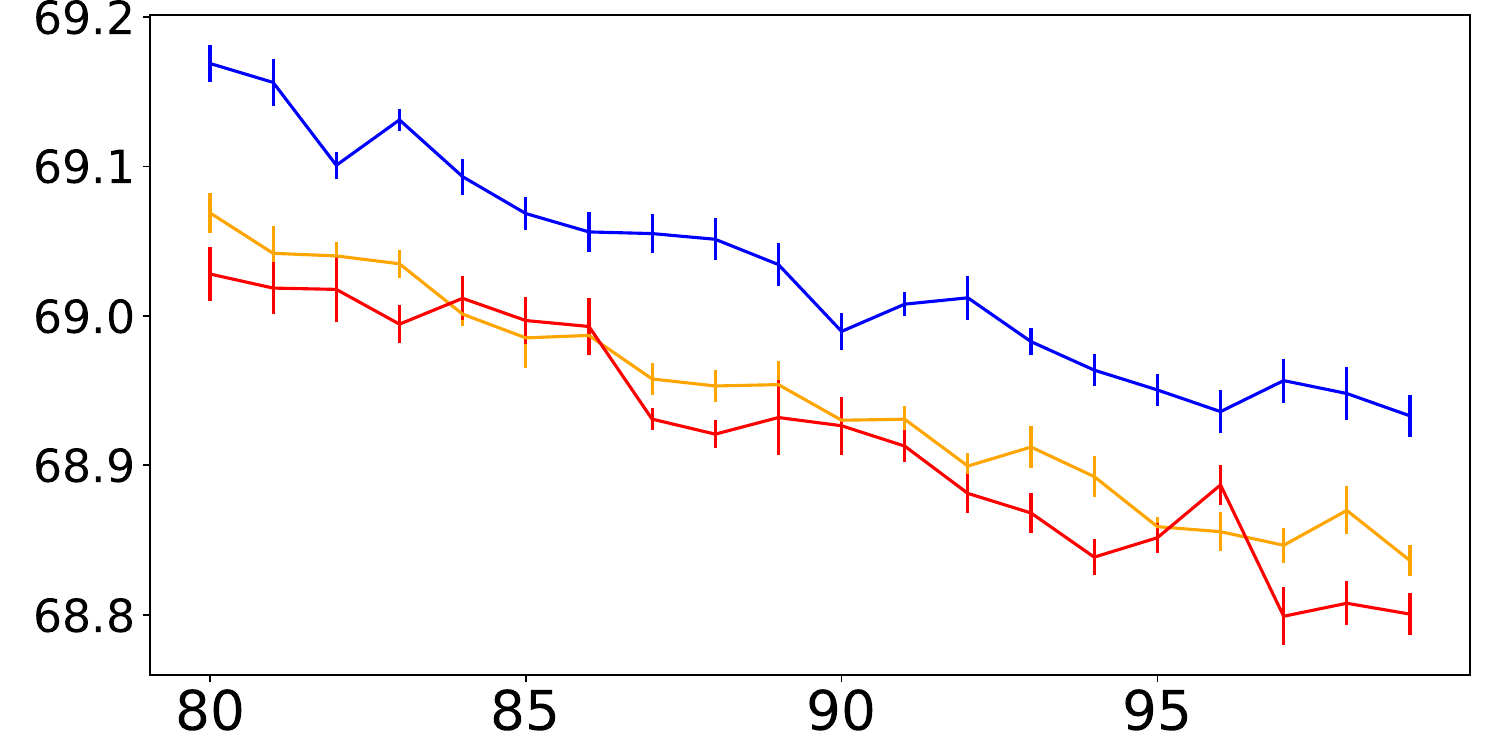}
        \caption{Dimension 20}
        \label{fig:training_loss_dim_20}
    \end{subfigure}
    \begin{subfigure}{0.32\textwidth}
        \includegraphics[width=\textwidth]{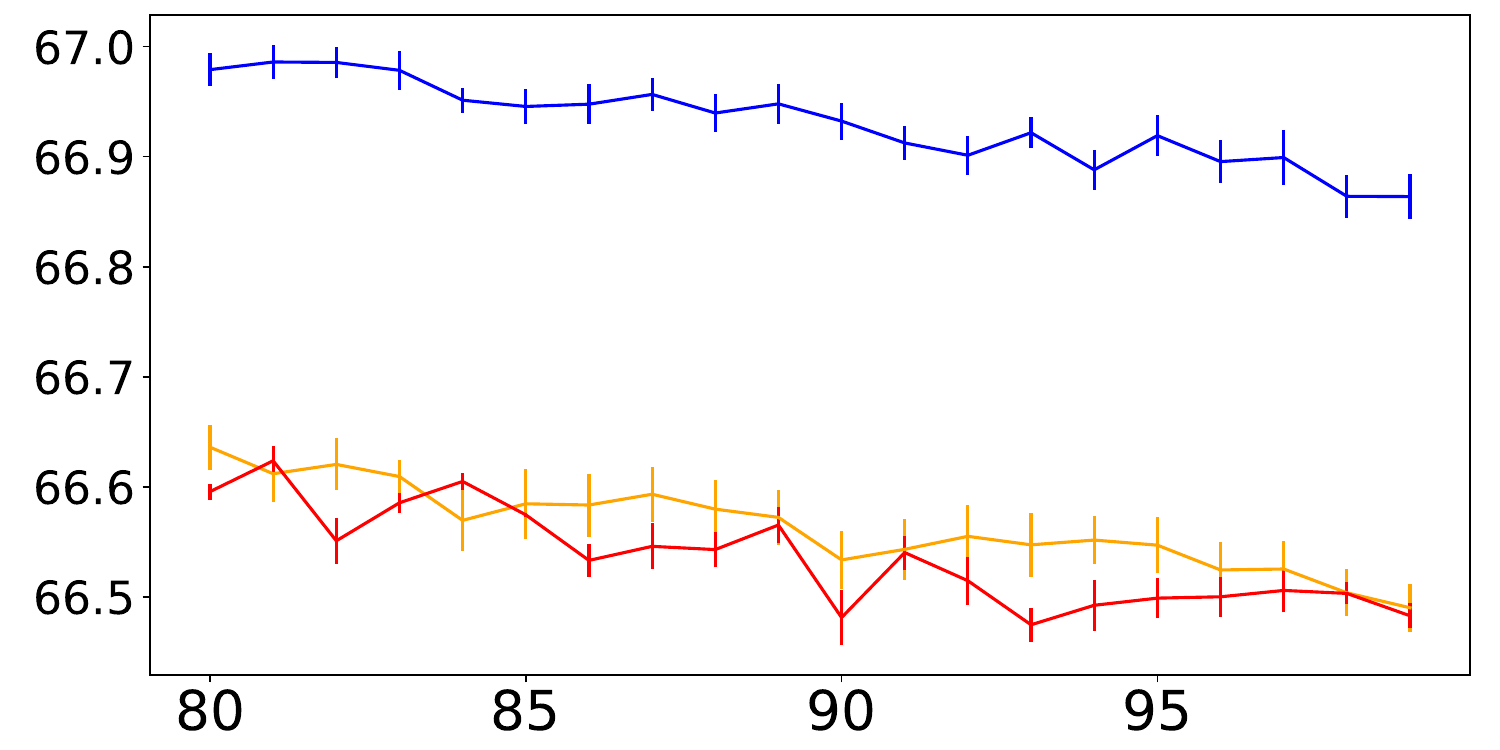}
        \caption{Dimension 50}
        \label{fig:training_loss_dim_50}
    \end{subfigure}
    \begin{subfigure}{0.32\textwidth}
        \includegraphics[width=\textwidth]{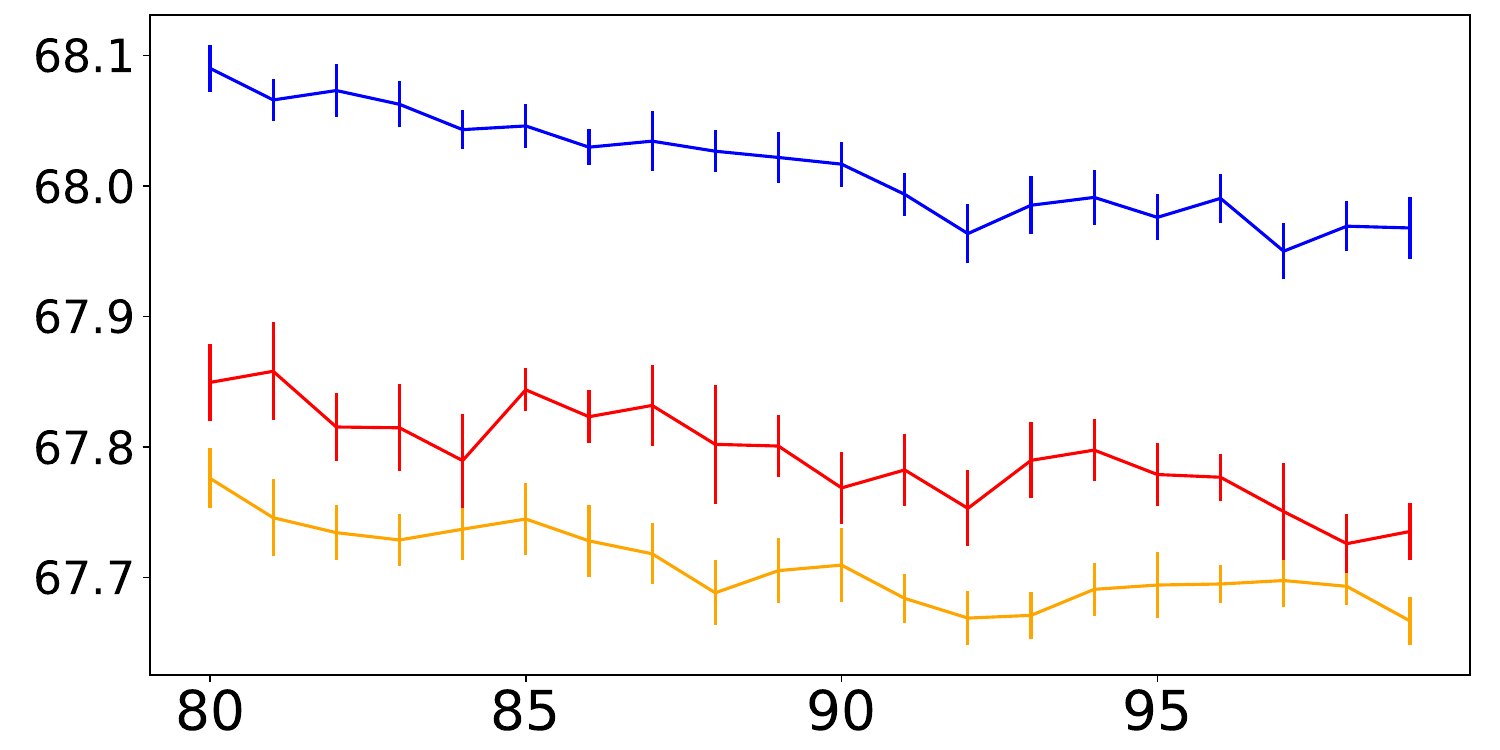}
        \caption{Dimension 100}
        \label{fig:training_loss_dim_100}
    \end{subfigure}
    \caption{Per epoch training loss (ELBO) for the last 25 epochs. Confidence intervals are calculated as $1.96 \sigma / \sqrt{n}$ over 4 ($n = 4$) different seeds.
    Blue corresponds to IWAE, orange to BR-IWAE with $k=4$ and red to BR-IWAE with $k=8$.}
    \label{fig:ELBO_train}
\end{figure}
\subsection{Resources}
\label{subsec:ressources}
All the simulations were done using a server with the following configuration:
\begin{itemize}
    \item GPUs: two Tesla V100-PCIE (32Gb RAM)
    \item CPU: 71 Intel(R) Xeon(R) Gold 6154 CPU @ 3.00GHz
    \item RAM: 377Gb
\end{itemize}
locally hosted. We estimate the total number of computing hours for the
results presented in this paper to be inferior to 200 hours of GPU usage (All the calculations were done in the GPU).

\typeout{get arXiv to do 4 passes: Label(s) may have changed. Rerun}
\end{document}